\documentclass[twoside,11pt]{article}

\usepackage{jmlr2e}
\usepackage{grffile}
\usepackage{dsfont}
\usepackage{amsmath}
\usepackage{tikz}
\usepackage{pgfplots}
\usetikzlibrary{arrows, automata}

\usetikzlibrary{shapes,arrows,positioning}
\tikzstyle{ellipsoid} = [draw, ellipse, minimum height=3em, minimum width=3em]
\tikzstyle{block} = [draw, rectangle, minimum height=3em, minimum width=3em]
\tikzstyle{round} = [draw, circle, minimum height=3em, minimum width=2em]
\tikzstyle{virtual} = [coordinate]

\newcommand\R{\mathbb{R}}

\newcommand\W{\mathbf{W}}

\newcommand\argmin{\operatorname*{\arg\min}}

\newcommand{\1}{\mathsf{(i)}}
\newcommand{\2}{\mathsf{(ii)}}
\newcommand{\3}{\mathsf{(iii)}}
\newcommand{\4}{\mathsf{(iv)}}

\newcommand{\real}[1]{\mathbb{R}^{#1}}
\newcommand{\sgn}{\mathsf{sgn}}

\newcommand{\tran}{\ensuremath{^\top}}

\newcommand{\hessian }[1]{\ensuremath{\nabla^2}}

\newcommand{\taninv}[1]{\ensuremath{\mathsf{tan}^{-1}\left(#1\right)}}

\newcommand{\alphahat}{\ensuremath{\widehat{\alpha}}}

\newcommand{\survivalb}{\ensuremath{\mathsf{SU}}}
\newcommand{\contaminationb}{\ensuremath{\mathsf{CN}}}

\newcommand{\normalconstant}[1]{\ensuremath{\frac{1}{\sqrt{2 \pi}}}}

\newcommand{\normal}[2]{\ensuremath{\mathcal{N}\left( #1, #2\right)}}
\newcommand{\snr}{\ensuremath{\mathsf{SNR}}}
\newcommand{\truek}{\ensuremath{t}}
\newtheorem{assumption}{Assumption}

\newcommand{\inprod}[2]{\ensuremath{\langle #1 , \, #2 \rangle}}

\newcommand{\Xvec}{\ensuremath{\mathbf{X}}}

\newcommand{\avec}{\ensuremath{\boldsymbol{\phi}}}

\newcommand{\z}{\ensuremath{\mathbf{z}}}

\newcommand{\y}{\ensuremath{\mathbf{y}}}

\newcommand{\alphastar}{\ensuremath{\boldsymbol{\alpha^{*}}}}
\newcommand{\betastar}{\ensuremath{\nu^{*}}}
\newcommand{\alphabold}{\ensuremath{\boldsymbol{\alpha}}}

\newcommand{\alphahatbold}{\ensuremath{\boldsymbol{\widehat{\alpha}}}}

\newcommand{\betabold}{\ensuremath{\boldsymbol{\beta}}}

\newcommand{\sumjdk}{\ensuremath{\sum\limits_{j=1, j \neq \truek}^d}}
\newcommand{\A}{\ensuremath{\mathbf{A}}}

\newcommand{\eig}{\mu}

\newcommand{\M}{\ensuremath{\mathbf{M}}}
\newcommand{\regloss}{\ensuremath{\mathcal{R}}}
\newcommand{\classloss}{\ensuremath{\mathcal{C}}}

\newcommand{\PP}{\ensuremath{\mathbb{P}}}
\newcommand{\EE}{\ensuremath{\mathbb{E}}}
\newcommand{\Sigmabold}{\ensuremath{\boldsymbol{\Sigma}}}

\newcommand{\lambdabold}{\ensuremath{\boldsymbol{\lambda}}}

\newcommand{\evec}{\ensuremath{\mathbf{e}}}
\newcommand{\Atrain}{\ensuremath{\boldsymbol{\Phi}_{\mathsf{train}}}}
\newcommand{\Ainv}{\ensuremath{\A^{-1}}}

\newcommand{\E}{\ensuremath{\mathbf{E}}}

\newcommand{\xtrain}{\mathbf{X}_{\mathsf{train}}}
\newcommand{\Ytrain}{\mathbf{Y}_{\mathsf{train}}}
\newcommand{\Ztrain}{\mathbf{Z}_{\mathsf{train}}}

\newcommand{\lambdatd}{\ensuremath{{\widetilde{\lambda}}}}
\newcommand{\Sigmaboldtd}{\ensuremath{{\Sigmabold}_{-\truek}}}

\newcommand{\Amk}{\ensuremath{\mathbf{A}_{-\truek}}}
\newcommand{\vvec}{\mathbf{v}}
\newcommand{\uvec}{\mathbf{u}}

\newcommand{\Amkinv}{\ensuremath{\Amk^{-1}}}
\newcommand{\Ctd}{\ensuremath{\mathbf{\widetilde{C}}}}
\newcommand{\B}{\ensuremath{\mathbf{B}}}
\newcommand{\C}{\ensuremath{\mathbf{C}}}
\newcommand{\lambdatk}{\ensuremath{\lambda_{\truek}}}
\newcommand{\tr}{\ensuremath{\mathsf{tr}}}
\newcommand{\sumjd}{\ensuremath{\sum\limits_{j=1}^d}}
\newcommand{\e}{\ensuremath{\mathbf{e}}}
\newcommand{\ztk}{\ensuremath{\z_{\truek}}}
\newcommand{\ytk}{\ensuremath{\y_{\truek}}}

\newcommand{\npm}[2]{\ensuremath{n^{-\frac{#1}{#2}}}}
\newcommand{\np}[2]{\ensuremath{n^{\frac{#1}{#2}}}}
\newcommand{\lambdasmall}{\ensuremath{\frac{(1-\gamma)d}{d-s}}}
\newcommand{\lambdabig}{\ensuremath{\frac{\gamma d}{s}}}
\newcommand{\survivalupperbound}{\ensuremath{\survival^U(n)}}

\newcommand{\survivallowerbound}{\ensuremath{\survival^L(n)}}

\newcommand{\survivalupperboundreal}{\ensuremath{\survivalr^U(n)}}
\newcommand{\oneminussurvivallowerboundreal}{\ensuremath{\widebar{\survivalr}^L(n)}}
\newcommand{\oneminussurvivalupperboundreal}{\ensuremath{\widebar{\survivalr}^U(n)}}
\newcommand{\survivallowerboundreal}{\ensuremath{\survivalr^L(n)}}

\newcommand{\contaminationupperbound}{\ensuremath{\contamination^U(n)}}

\newcommand{\SUClower}{\ensuremath{\snr^L(n)}}
\newcommand{\SUCupper}{\ensuremath{\snr^U(n)}}

\newcommand{\contaminationupperboundreal}{\ensuremath{\contaminationr^U(n)}}
\newcommand{\contaminationlowerboundreal}{\ensuremath{\contaminationr^L(n)}}

\newcommand{\phibold}{\ensuremath{\boldsymbol{\phi}}}
\newcommand{\Mbold}{\ensuremath{\mathbf{M}}}

\newcommand{\norm}[1]{\left\lVert#1\right\rVert}
\newcommand{\abs}[1]{\left| #1 \right|}

\newcommand{\reglossn}{\regloss(\alphahatbold_{2,\mathsf{real}};n)}
\newcommand{\classlossn}{\classloss(\alphahatbold_{2,\mathsf{binary}};n)}
\newcommand{\survivalr}{\survivalb_r}
\newcommand{\contaminationr}{\contaminationb_r}
\newcommand{\survival}{\survivalb_b}
\newcommand{\contamination}{\contaminationb_b}
\newcommand{\survivaln}{\survival(1;n)}
\newcommand{\contaminationn}{\contamination(1;n)}
\newcommand{\survivalnreal}{\survivalb_r(1;n)}
\newcommand{\contaminationnreal}{\contaminationb_r(1;n)}

\newcommand{\qed}{\ensuremath{\hfill\blacksquare}}

\makeatletter
\let\save@mathaccent\mathaccent
\newcommand*\if@single[3]{\setbox0\hbox{${\mathaccent"0362{#1}}^H$}\setbox2\hbox{${\mathaccent"0362{\kern0pt#1}}^H$}\ifdim\ht0=\ht2 #3\else #2\fi
  }
\newcommand*\rel@kern[1]{\kern#1\dimexpr\macc@kerna}
\newcommand*\widebar[1]{\@ifnextchar^{{\wide@bar{#1}{0}}}{\wide@bar{#1}{1}}}
\newcommand*\wide@bar[2]{\if@single{#1}{\wide@bar@{#1}{#2}{1}}{\wide@bar@{#1}{#2}{2}}}
\newcommand*\wide@bar@[3]{\begingroup
  \def\mathaccent##1##2{\let\mathaccent\save@mathaccent
\if#32 \let\macc@nucleus\first@char \fi
\setbox\z@\hbox{$\macc@style{\macc@nucleus}_{}$}\setbox\tw@\hbox{$\macc@style{\macc@nucleus}{}_{}$}\dimen@\wd\tw@
    \advance\dimen@-\wd\z@
\divide\dimen@ 3
    \@tempdima\wd\tw@
    \advance\@tempdima-\scriptspace
\divide\@tempdima 10
    \advance\dimen@-\@tempdima
\ifdim\dimen@>\z@ \dimen@0pt\fi
\rel@kern{0.6}\kern-\dimen@
    \if#31
      \overline{\rel@kern{-0.6}\kern\dimen@\macc@nucleus\rel@kern{0.4}\kern\dimen@}\advance\dimen@0.4\dimexpr\macc@kerna
\let\final@kern#2\ifdim\dimen@<\z@ \let\final@kern1\fi
      \if\final@kern1 \kern-\dimen@\fi
    \else
      \overline{\rel@kern{-0.6}\kern\dimen@#1}\fi
  }\macc@depth\@ne
  \let\math@bgroup\@empty \let\math@egroup\macc@set@skewchar
  \mathsurround\z@ \frozen@everymath{\mathgroup\macc@group\relax}\macc@set@skewchar\relax
  \let\mathaccentV\macc@nested@a
\if#31
    \macc@nested@a\relax111{#1}\else
\def\gobble@till@marker##1\endmarker{}\futurelet\first@char\gobble@till@marker#1\endmarker
    \ifcat\noexpand\first@char A\else
      \def\first@char{}\fi
    \macc@nested@a\relax111{\first@char}\fi
  \endgroup
}

\usepackage{lastpage}
\jmlrheading{22}{2021}{1-\pageref{LastPage}}{6/20; Revised
5/21}{7/21}{20-603}{Vidya Muthukumar, Adhyyan Narang, Vignesh Subramanian, Mikhail Belkin, Daniel Hsu and Anant Sahai}
\ShortHeadings{Does the loss function matter?}{Muthukumar, Narang, Subramanian, Belkin, Hsu and Sahai}
\firstpageno{1}

\begin{document}

\title{Classification vs regression in overparameterized regimes: Does the loss function matter?}

\author{\name Vidya Muthukumar$^*$ \email vmuthukumar8@gatech.edu \\
\addr Electrical and Computer Engineering and Industrial and Systems Engineering\\
Georgia Institute of Technology\\
Atlanta, GA-30332, USA
\AND
\name Adhyyan Narang$^*$  \email adhyyan@uw.edu  \\
\addr Department of Electrical and Computer Engineering\\
University of Washington\\
Seattle, WA-98115, USA
\AND
\name Vignesh Subramanian$^*$  \email vignesh.subramanian@eecs.berkeley.edu  \\
\addr Department of Electrical Engineering and Computer Sciences\\
University of California Berkeley\\
Berkeley, CA-94720, USA
\AND
\name Mikhail Belkin \email mbelkin@ucsd.edu   \\
\addr Halicio\u{g}lu Data Science Institute\\
UC San Diego\\
La Jolla, CA-92093, USA
\AND
\name Daniel Hsu \email djhsu@cs.columbia.edu \\
\addr Department of Computer Science and Data Science Institute\\
Columbia University\\
New York, NY-10027, USA
\AND
\name Anant Sahai \email sahai@eecs.berkeley.edu  \\
\addr Department of Electrical Engineering and Computer Sciences\\
University of California Berkeley\\
Berkeley, CA-94720, USA
}

\editor{Simon Lacoste-Julien}

\maketitle

\makeatletter
\def\blfootnote{\xdef\@thefnmark{}\@footnotetext}
\makeatother

\begin{abstract}%
We\blfootnote{$^*$indicates equal contribution among authors who were students at the time of submission. 
Faculty are listed alphabetically after the students.
The key results in this paper were unveiled at the ITA workshop in San Diego in February 2020.} compare classification and regression tasks in an overparameterized linear model with Gaussian features.
On the one hand, we show that with sufficient overparameterization {\it all training points are support vectors}: solutions obtained by least-squares minimum-norm interpolation, typically used for regression, are identical to those produced by the hard-margin support vector machine (SVM) that minimizes the hinge loss, typically used for training classifiers.
On the other hand, we show that there exist regimes where these interpolating solutions generalize well when evaluated by the 0-1 test loss function, but do not generalize if evaluated by the square loss function, i.e.~they approach the null risk.
Our results demonstrate the very different roles and properties of loss functions used at the training phase (optimization) and the testing phase (generalization).
\end{abstract}


\section{Introduction}

Paradigmatic problems in supervised machine learning (ML) involve predicting an output response from an input, based on patterns extracted from a (training) data set.
In classification, the output response is (finitely) discrete and we need to classify input data into one of these discrete categories.
In regression, the output is continuous, typically a real number or a vector.  Owing to this  important distinction in output response, the two tasks are typically treated differently.
The differences in treatment manifest in two phases of modern ML: optimization (training), which consists of an algorithmic procedure to
extract a predictor from the training data, typically by minimizing the training loss (also called \textit{empirical risk}); and generalization (testing), which consists of an evaluation of the obtained predictor on a separate test, or validation, data set.

Traditionally, the choice of loss functions for both phases is starkly different across classification and regression tasks.
The square loss function is typically used both for the training and the testing phases in regression.
In contrast, the hinge  or logistic (\textit{cross-entropy} for multi-class problems) loss functions are typically used in the training phase of classification, even though the 0-1 loss function is used for testing.
The use of the logistic and hinge losses can be motivated by their computational and statistical properties.
For example, the theory of \emph{surrogate losses}~\citep{zhang2004statistical,bartlett2006convexity,ben2012minimizing} gives theoretical arguments favoring the logistic and hinge losses over other convex surrogates, including the square loss.\footnote{Also see, e.g., Section 8.1.2 in~\citet{Goodfellow-et-al-2016} for a representative informal discussion.}
Yet, there have been indications that the reality is more complex, both in underparameterized and overparameterized regimes of ML.
For example, \citet{rifkin2002everything} extensively compared the hard-margin \textit{support vector machine} (SVM), which minimizes the hinge loss, and \textit{regularized least-squares classification} (RLSC), which minimizes the square loss --- ultimately concluding that ``the performance of the RLSC is essentially equivalent to that of the SVM across a wide range of problems, and the choice between the two should be based on computational tractability considerations.''
Quite similar results\footnote{In fact, while the results were generally close, in a majority of classification tasks models trained using the square loss outperformed models
trained with cross-entropy.} for a comparison between the square loss and cross-entropy loss have recently been obtained in~\citet{hui2021evaluation} for a range of modern neural
architectures and data sets across several application domains.
The latter is the current dominant standard for training neural networks.

In an important separate development, we have recently seen compelling evidence that overparameterized deep neural networks, as well as other models trained to
\textit{interpolate} the data (i.e.~achieve zero, or near zero,
training loss), are capable of good test performance,  questioning
conventional\footnote{For example, Hastie, Tibshirani and Friedman
  say, on page $221$ of their popular statistical learning
  textbook~\citep{friedman2001elements}, that ``a model with zero
  training error is overfit to the training data and will typically
  generalize poorly.''} statistical wisdom~\citep{neyshabur2014search,zhang2016understanding,belkin2018understand, geiger2019jamming,belkin2019reconciling}.
Since then, the theoretical ML community has identified regimes \textit{for regression tasks} under which overfitting is benign ~\citep[to borrow the language from][]{bartlett2020benign}, and interpolating \textit{noisy} data with solutions arising from empirical risk minimization is compatible with good generalization~\citep{bartlett2020benign,belkin2020two,hastie2019surprises,mei2019generalization,muthukumar2020harmless}.
It is worth noting that the ensuing test loss of interpolating solutions is unrelated to their training loss, which is identically zero for all such solutions.
This demonstrates, again, that the relationship between training and test losses --- both for regression and classification tasks --- is more complex than often assumed.

This paper introduces a direct comparison between the different loss functions used in classification and regression, in both the training and testing phases.
We analyze the modern overparameterized regime under the linear model with Gaussian features and uncover a remarkable phenomenon of overparameterized training: in sufficiently overparameterized settings, with high probability, {\it every training data point is a support vector}. 
Consequently, the outcome of optimization (with gradient descent) is the same whether we use the hinge loss, logistic loss, or the square loss.

On the other hand, we show that the choice of test loss function results in a significant asymptotic difference between classification and regression tasks.
In particular, we identify truly overparameterized regimes for which predictors will generalize poorly for regression tasks (measured by the square loss), but well for classification tasks (measured by the 0-1 loss).
The fact that regression and classification can be different has been understood for some time. 
For example, Devroye, Gyorgi and Lugosi point out in Chapter 6.7 of their classic textbook~\citep{devroye2013probabilistic} that \textit{``classification is easier than regression function estimation"}, in reference to sample complexity.
Artificial examples are also given where accurate regression estimates of
the density are impossible given the hypothesis class, but classification succeeds because of the separability of the two classes. In contrast, our paper demonstrates
that in overparameterized regimes, this phenomenon can be quite generic. Approximability is not the underlying issue; rather, it is a consequence of learning.

\subsection{Our contributions}
Our study investigates differences and commonalities  between classification and regression, using the
overparameterized linear model with Gaussian features.
On the side of commonality, we connect the hard-margin-maximizing SVM to the minimum-$\ell_2$-norm interpolator of classification training data (i.e.~binary labels), by showing that they are identical once the degree of ``effective overparameterization" is sufficiently large (Theorem~\ref{thm:everythingsupport}).
This shows that, contrary to the prevailing low-dimensional intuition, there is no difference between maximum margin and least-squares solutions in high-dimensional settings. In particular, using the appropriate optimization methods, minimizing the logistic, hinge, and square loss yield exactly identical predictors.
This phenomenon is not restricted to Gaussian feauturization: in fact, as Figure~\ref{fig:svmlsmaintext} illustrates, we first noticed that all training points tend to become support vectors in the case of Fourier features on regularly spaced training data.
For a detailed description of this ``ultra-toy" model and experiment, see Appendix~\ref{appendix:ultratoy}.

\begin{figure}[t]
  \centering
  \begin{tabular}{p{0.35\textwidth}p{0.35\textwidth}}
    \centering
    \includegraphics[width=0.35\textwidth]{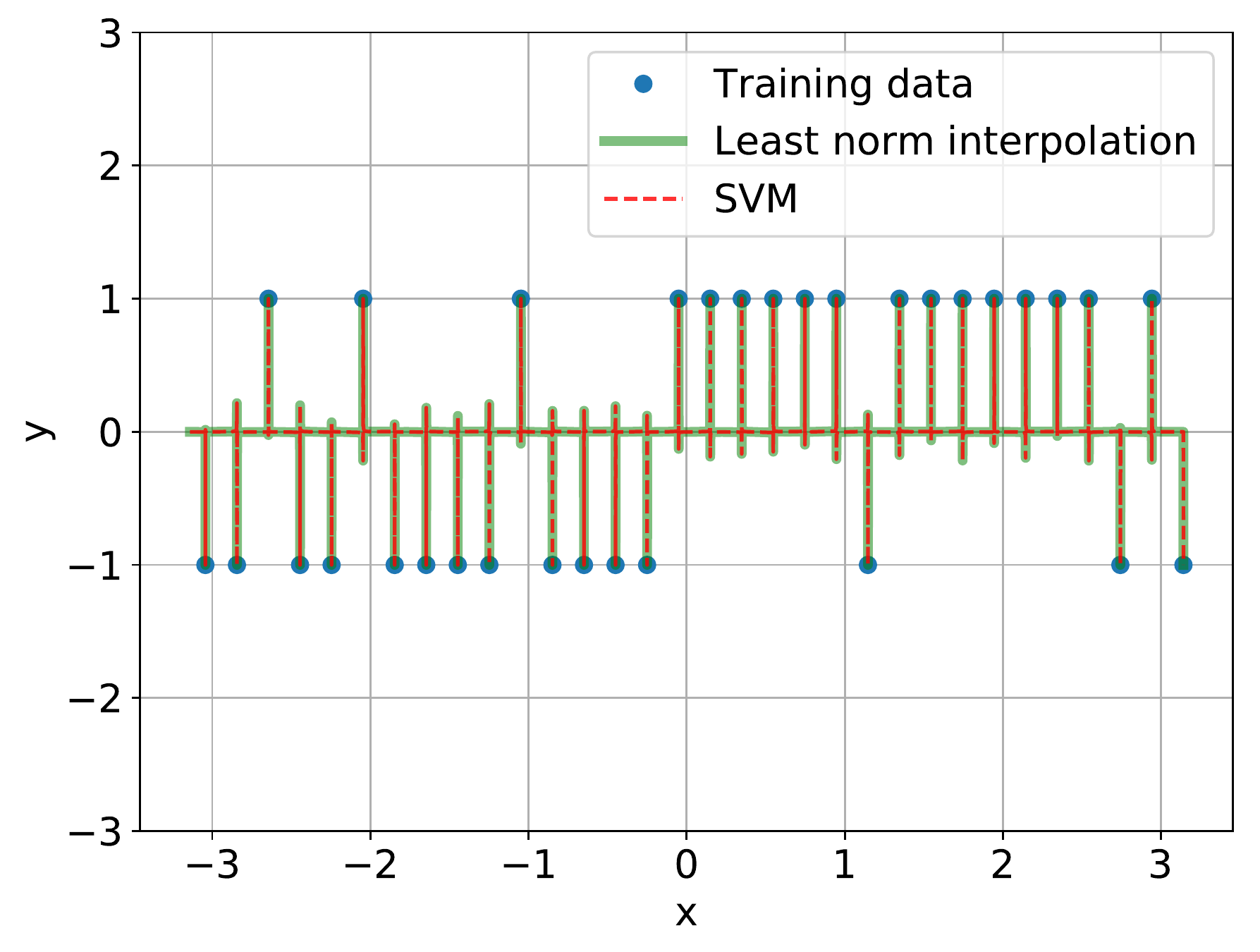}
    &
    \centering
    \includegraphics[width=0.35\textwidth]{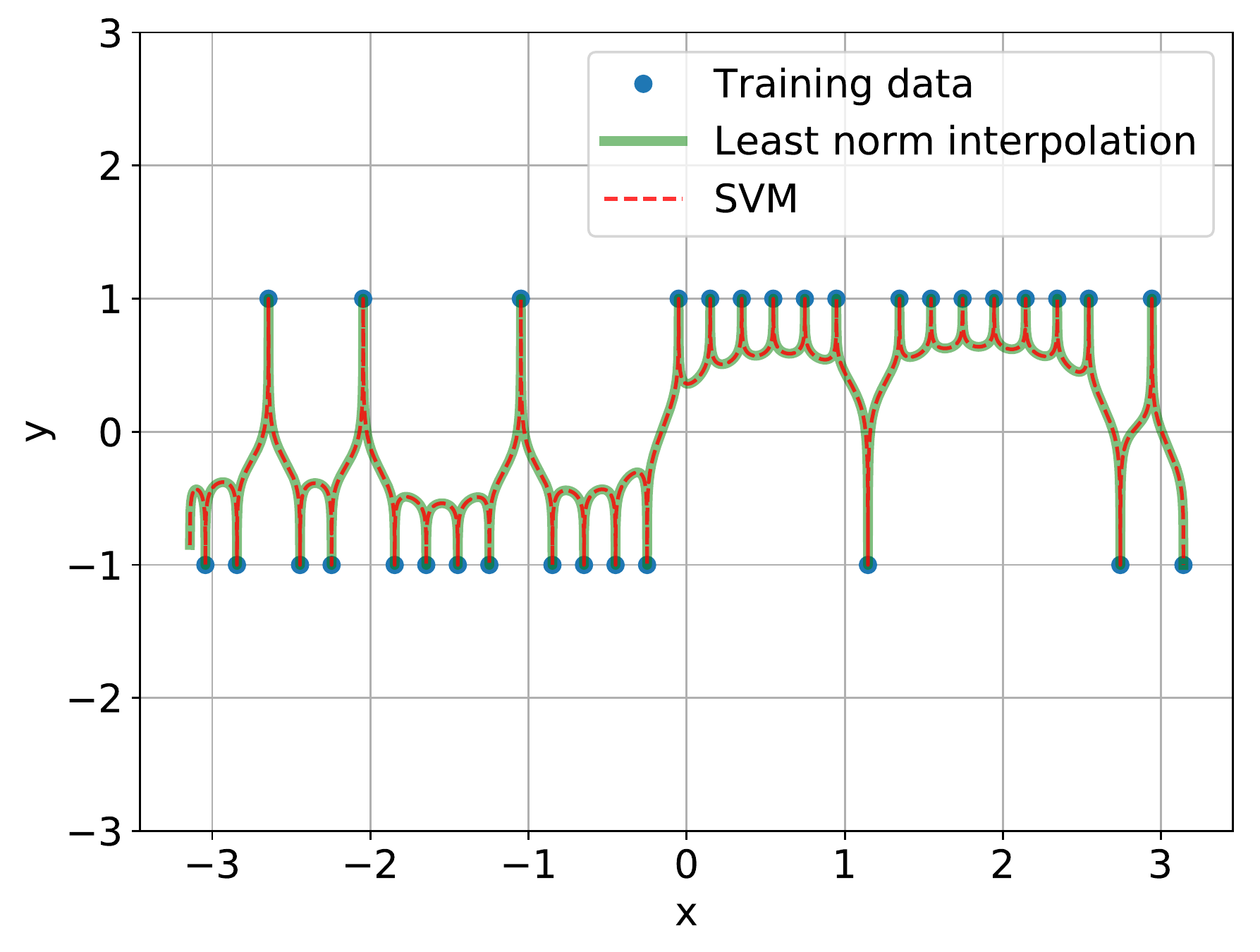}
    \tabularnewline
    \centering
    (a) $m=0$
    &
    \centering
    (b) $m=1$
    \tabularnewline
    \centering
    \includegraphics[width=0.35\textwidth]{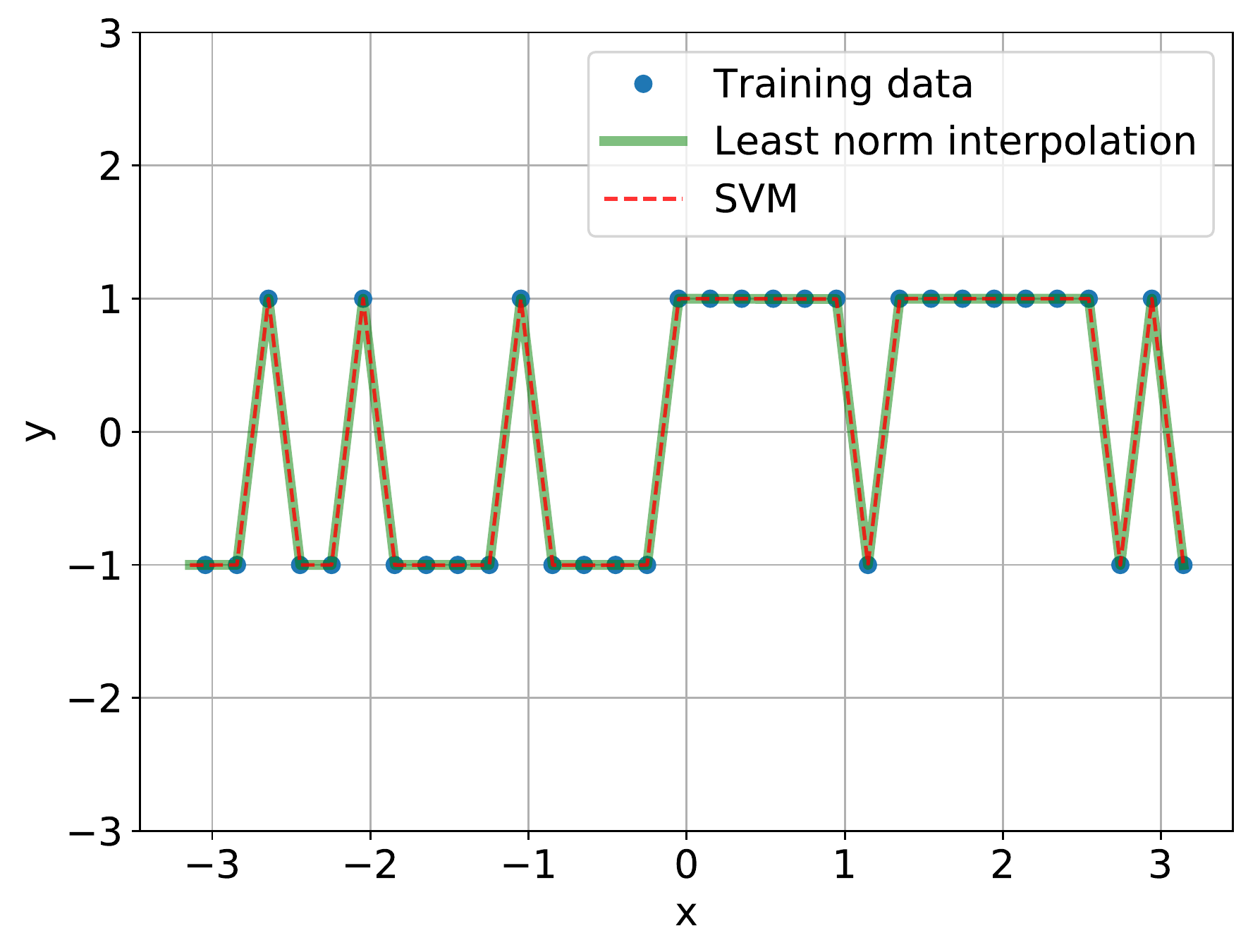}
    &
    \centering
    \includegraphics[width=0.35\textwidth]{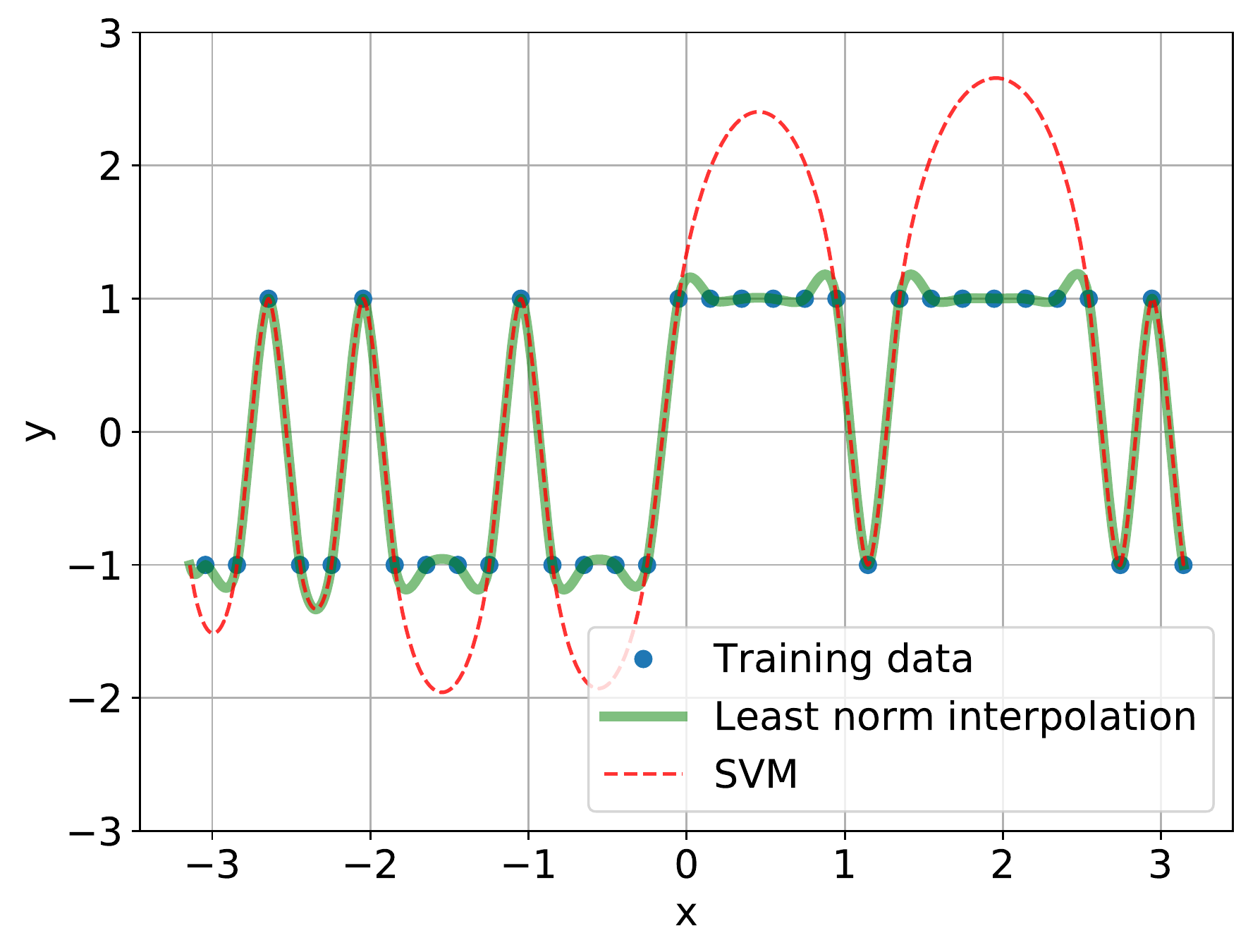}
    \tabularnewline
    \centering
    (c) $m=2$
    &
    \centering
    (d) $m=3$

  \end{tabular}
  \caption{Correspondence between the SVM and minimum-$\ell_2$-norm
    interpolation, illustrated by Fourier features on regularly spaced
    training data with $10\%$ label noise (for various rates of
    feature scaling corresponding to $\lambda_k = \frac{1}{k^{m}}$ in the optimization to adjust the preference for lower frequencies as given in Definition~\ref{eg:fourier} in Appendix~\ref{appendix:ultratoy}). For all the figures, the number of samples $n = 32$ and the number of features $d = 2^{14}$. Notice that below $m=2$, all points are support vectors and the SVM solution agrees with least-norm interpolation.}
  \label{fig:svmlsmaintext}
\end{figure}

In contrast, we show that the choice of loss function used on \textit{test data} yields significant differences between classification and regression.
Depending on the extent of ``effective overparameterization", the same minimum-norm solution can:
\begin{itemize}
  \item succeed at both regression and classification,
  \item succeed at classification and fail at regression, or
  \item fail at both,
\end{itemize}
as we show in Theorem~\ref{thm:main}.
The intermediate regime of special interest is the one for which minimum-$\ell_2$-norm interpolators generalize poorly in regression tasks, but well in classification tasks.
Underlying these results is a sharp non-asymptotic analysis of the minimum-$\ell_2$-norm interpolator \textit{for the classification task}.
We conceptually link the techniques introduced in recent analysis of this interpolator for the regression task~\citep{bartlett2020benign} to the classification task, using a signal-processing (Fourier-theoretic) interpretation of the overparameterized regime that was introduced in~\citet{muthukumar2020harmless}.
This constitutes a first analysis of this type for classification tasks, providing non-asymptotically matching upper and lower bounds.
In Section~\ref{sec:margin}, we demonstrate that the classic \textit{upper bounds}, based on training data margin, fail to produce meaningful results or useful intuition in our setting.

\section{Related work}

The phenomenon of overparameterization and interpolation yielding significantly improved empirical performance across a variety of models as well as tasks~\citep{neyshabur2014search,zhang2016understanding,geiger2019jamming,belkin2019reconciling} has received significant research attention over the last few years.
In this section, we contextualize our results in this research landscape.

\subsection{The role of the training loss function (and optimization algorithm)}\label{sec:relwork_lossfunction}

At a high level, any solution obtained in an overparameterized regime that generalizes well must have some sort of \textit{regularization}, i.e.~special structural constraints on the values it can take.
Thus, we need to understand the influence of the choice of training loss function on the resulting solution and its generalization guarantee.
In the overparameterized regime, there are infinitely many solutions that interpolate training data, and indeed even more that \textit{separate} discretely labeled data.
Thus, characterizing the \textit{implicit regularization}~\citep{ji2019implicit,soudry2018implicit,gunasekar2018characterizing,woodworth2019kernel,nacson2019convergence,azizan2020study} induced by the choice of optimization algorithm is important to understand properties of the obtained solutions.
For the linear model, we have a concrete understanding of the solutions obtained by the most common choices of training loss functions:
\begin{enumerate}
\item If we minimize the logistic loss using gradient descent on separable training data\footnote{The implicit bias has also been characterized for the more difficult non-separable case~\citep{ji2019implicit}, but we focus here on separable training data as this will always be the case for an overparameterized setting.}, we will converge to the hard-margin SVM~\citep{ji2019implicit,soudry2018implicit}.
    \item If we minimize the square loss on training data using gradient descent while also using an overparameterized model, we will converge to the minimum-$\ell_2$-norm interpolation~\citep[Theorem~6.1]{engl1996regularization} \textit{provided the initialization is equal to zero.}
\end{enumerate}

As mentioned in the introduction, conventional wisdom recommends the choice of the logistic loss, or the hinge loss, for classification tasks.
It is sometimes implied (without theoretical justification) that instead minimizing the square loss would be suboptimal for generalization.
However, our first main result (Theorem~\ref{thm:everythingsupport}) shows that with sufficient overparameterization, the SVM itself interpolates the binary labels --- as pictured in Figure~\ref{fig:equivalence}, this implies an equivalence in solutions corresponding to several choices of training loss function.
Moreover, our subsequent Theorem~\ref{theorem:threeregimes} shows that the \textit{interpolating} solution generalizes well in classification tasks, for a wide range of overparameterized regimes.
And when it does generalize poorly, so does the SVM!
These results add theoretical weight to the empirical evidence~\citep{rifkin2002everything,pmlr-v51-que16} that the hinge loss (and, by extension, the cross-entropy loss) is not necessarily the superior choice for classification tasks.

Our Theorem~\ref{thm:everythingsupport} establishes a link between the hard-margin SVM and the minimum-$\ell_2$-norm interpolation by exhibiting an overparameterized and separable setting where every training example is a support vector.
Previous works have related the number of support vectors in \textit{soft-margin} SVMs and the Bayes risk and conditional probability estimation~\citep{steinwart2003sparseness,bartlett2007sparseness}, but do not apply to the hard-margin SVM on separable data sets.

The SVM maximizes \textit{training data margin} in feature space.
Theoretical analyses of generalization error as a function of the margin have been proposed to explain the success of models such as boosting and neural networks~\citep{schapire1998boosting,bartlett1998sample,bartlett2017spectrally}.
Explanations based on margin bounds are sometimes credited, in a heuristic manner, for the empirical success of interpolated models in classification tasks.
This is, in fact, a misleading explanation~\citep[as also noted in][]{shah2018minimum}; in Section~\ref{sec:margin}, we provide experimental evidence for the tautology of generalization upper bounds as a function of the feature-space margin, when applied to sufficiently overparameterized models.
This evidence corroborates the recent perspectives on modern ML which argue against generalization bounds that tie training loss to the expected loss on test data~\citep{belkin2018understand,nagarajan2019uniform}.
We instead favor a first-principles approach to analyzing high-dimensional models for classification, inspired by recent progress in regression.

\begin{figure}
\centering
\begin{tikzpicture}[>=stealth,auto, node distance=2cm]
\node[block, align = left](logistic){Logistic loss};
    \node[block, right = 7cm of logistic] (square) {Square loss};
\node [ellipsoid, below = 2cm of logistic] (svm)     {Hard-margin SVM};
    \node [ellipsoid, below = 2cm of square] (regression)    {Min $\ell_2$-norm interpolation};
\draw [<->, thick] (svm) -- node {Theorem \ref{thm:everythingsupport}, this paper} (regression);
    \draw [->, thick] (square) -- node[left, align = left] {Gradient Descent, \\ initialized at 0\\\citep{engl1996regularization}} (regression);
    \draw [->, thick] (logistic) -- node[right, align = left] {Gradient Descent \\ \citep{ji2019implicit}\\\citep{soudry2018implicit}} (svm);
\end{tikzpicture}
\caption{Equivalence of training procedures in overparameterized settings.
Theorem~\ref{thm:everythingsupport} in this paper highlights exact equivalence \textit{with high probability} between the hard-margin SVM and minimum-$\ell_2$-norm interpolation under sufficient \textit{effective overparameterization}.}
\label{fig:equivalence}
\end{figure}
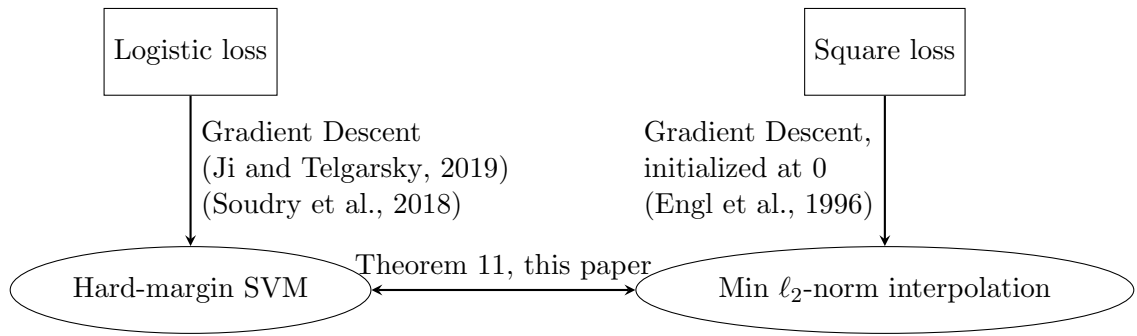

\subsection{Insights from least-squares regression}\label{sec:relwork_regression}

The recently observed phenomenon of \textit{double
  descent}~\citep{geiger2019jamming,belkin2019reconciling} made
concrete explicit empirical benefits of
overparameterization. Subsequent\\
work~\citep{bartlett2020benign,belkin2020two,hastie2019surprises,mei2019generalization,muthukumar2020harmless,mitra2019understanding,
nakkiran2019more} has identified theoretical conditions under which overparameterization and interpolation can be helpful, or at the very least, harmless in linear regression with different feature families.
The main insight can be crystallized as follows: for overparameterized solutions to interpolate ``benignly", the feature family needs to satisfy a delicate balance between having a few important directions that favor the true signal (unknown function), and a large number of unimportant directions that \textit{absorb} the noise in a harmless manner.
This trade-off was explored both for minimum-$\ell_2$-norm~\citep{bartlett2020benign,belkin2020two,hastie2019surprises,mei2019generalization,muthukumar2020harmless,mitra2019understanding} and minimum-$\ell_1$-norm interpolations \citep{muthukumar2020harmless,mitra2019understanding}.

In this paper, we build on these insights from regression tasks, to show that overparameterized models can similarly generalize well for classification tasks.
In fact, it turns out that the conditions for classification are milder, and there is an intermediate regime of overparameterization where the regression problem is ``hard'' --- but the classification  problem is ``easy''.
Focusing on the well-specified case, we show that \textit{the balance between preserving signal and absorbing noise does not have to be as delicate for classification tasks as it is for regression tasks}.
This conclusion cannot be made directly from the regression analyses, as 0-1 classification error is quite different from the mean-square regression error.
To bridge the gaps, we use a signal processing perspective on the overparameterized regime that was first developed in \citet{muthukumar2020harmless}, where the conditions for low test error are linked to notions of \textit{survival} of the true features and \textit{contamination} by falsely discovered features.
We will see (in Section~\ref{sec:pathtoanalysis}) that these same quantities show up explicitly in the analysis of classification test error.

\subsection{Recent work on high-dimensional classification/logistic regression}\label{sec:concurrentwork}

High-dimensional \textit{logistic regression} and classification are naturally closely connected, and statisticians have studied the former in a number of contexts.
Properties of penalized maximum likelihood estimators in overparameterized logistic regression have received substantial attention (an incomplete list is~\citet{bunea2008honest,van2008high,kakade2010learning,fan2011nonconcave}).
Here, the penalty, or regularizer, is typically in $\ell_1$-norm and its relatives, and the studied solutions do not interpolate training data.
In contrast, our focus is on classification problems and thus the
properties of the $\ell_2$-margin-maximizing support-vector-machines,
and moreover we make explicit connections to solutions that
interpolate binary labels.

Most pertinent to our setting, we acknowledge a recent line of work~\citep{deng2019model,montanari2019generalization,kini2020analytic} that identifies precise asymptotics for the generalization error of the SVM as a function of the overparameterization factor.
The main technical tool common to these works is the convex Gaussian min-max theorem~\citep{thrampoulidis2015regularized}, a generalization of Gordon's min-max theorem~\citep{gordon1985some} that has seen substantial application to obtain precise asymptotics in high-dimensional regression.
It is worth noting that these elegant analyses specifically assume isotropic featurization, and do not study the ramifications of anisotropy, which is known to be critical for good generalization of $\ell_2$-regularized solutions.
While~\citet{kini2020analytic} do compare the outcomes of logistic and square loss in binary classification, and show that the test error is almost identical under extreme overparameterization, they do not show an explicit link between the actual solutions.

Concurrent to our work, ~\citet{chatterji2020finite} directly
\textit{upper bound} the generalization error of the max-margin SVM
under overparameterized linear discriminant models that are not
isotropic and, like us, explore how anisotropic the situation needs to be for good generalization.
Both their results and techniques are quite different from ours in that they study the iterates of gradient descent leveraging the implicit regularization perspective of optimization algorithms.
 \section{Setup}\label{sec:setup}

We begin with some basic notation.
Thereafter, we describe the setup for training and test data, evaluation of classification and regression tasks, and choices of featurization (in that order).

\subsection{Basic notation}

We describe basic notation for vectors, matrices, and functions.

\subsubsection{Vector and matrix notation}
Let $\evec_i$ represent the $i^{th}$ standard basis vector (with the dimension implicit).
For a given vector $\mathbf{v}$, the functional $\sgn(\mathbf{v})$ denotes the sign operator applied element-wise.
Let $\eig_i(\Mbold)$ denote the $i^{th}$ largest eigenvalue of positive semidefinite matrix $\Mbold$, and $\eig_{\mathsf{max}}(\Mbold)$ and $\eig_{\mathsf{min}}(\Mbold)$ denote in particular the maximal and minimal eigenvalue respectively.
Further, we use $||\Mbold||_{\mathsf{op}}, \tr(\Mbold)$ and $||\Mbold||_{\mathsf{F}}$ to denote the operator norm, trace norm, and Frobenius norm respectively.

\subsubsection{Function-specific notation}
For two functions $f(n)$ and $g(n)$, we write $f \asymp g$ \textit{iff} there exist universal positive constants $(c, C, n_0)$ such that
$$c\abs{g(n)} \leq \abs{f(n)} \leq C\abs{g(n)} \ \forall n \geq n_0.$$
(In most places where we apply the above inequality, the functions $f$ and $g$ are positive valued and so we automatically drop the absolute value signs.)

\subsection{Data}

Let $\mathcal{X}$ denote the space of input data.
For classification, our training data are \textit{input data-binary label} pairs $(X_1, Y_1),\dotsc,(X_n,Y_n)$ taking values in $\mathcal{X} \times \{-1,+1\}$; for regression, the training data are \textit{input data}-real output pairs $(X_1, Z_1), \dotsc, (X_n,Z_n)$ taking values in $\mathcal{X} \times \R$.
We assume that there is a feature map $\avec \colon \mathcal{X} \to \R^d$, target linear function parameterized by $\alphastar \in \R^d$, and label noise parameter $0 \leq \betastar < 1/2$ such that for every $i \in \{1,2,\ldots,n\}$, we have

\begin{align}\label{eq:generativemodel}
    Z_i &= \inprod{\avec(X_i)}{\alphastar} \text{ and } \\
    Y_i &= \begin{cases}
        \sgn(Z_i) \text{ with probability } (1 - \betastar) \\
        - \sgn(Z_i) \text{ with probability } \betastar .
    \end{cases}
\end{align}
Here, the feature map $\avec$ is known, but the target parameter
$\alphastar$ (which we refer to as the signal) is unknown. The label
noise in $Y_i$ is assumed to be independent of everything else.

Let $\avec(x) = \begin{bmatrix} \phi_1(x) & \ldots & \phi_d(x)\end{bmatrix}^T$ for $x \in \mathcal{X}$, i.e. $\phi_j(x)$ is the value of the $j^{\text{th}}$ feature in $\avec(x)$.
We will consider the training data $\{X_i\}_{i=1}^n$ to be mutually independent and identically distributed (iid).
Let $\Sigmabold = \EE[\avec(X) \avec(X)^\top]$ denote the covariance matrix of the feature vector $\avec(X)$ for $X$ following the same distribution as $X_i$.
We assume $\Sigmabold$ is invertible, so its square-root-inverse $\Sigmabold^{-1/2}$ exists.

We define shorthand notation for the training data: let
\begin{align*}
    \Atrain := \begin{bmatrix}
    \avec(X_1) & \avec(X_2) & \dotsb & \avec(X_n)
  \end{bmatrix}^\top \in \R^{n \times d}
\end{align*}
denote the data (feature) matrix; $\Ztrain := \begin{bmatrix} Z_1 & \ldots & Z_n\end{bmatrix}^\top \in \real n$ denote the regression output vector; and $\Ytrain := \begin{bmatrix} Y_1 & \ldots & Y_n\end{bmatrix}^\top$ denote the classification output vector.
Note that if there is no label noise (i.e.~$\betastar = 0$), then we have $\Ytrain = \sgn(\Ztrain)$.

\subsection{Classification, regression, and interpolation}

The overparameterized regime constitutes the case in which the dimension (or number) of features is greater than the number of samples, i.e. $d \geq n$.
We define the two types of solutions that we will primarily consider in this regime, starting with interpolating solutions.
\begin{definition}\label{def:interpolators}
We consider solutions $\alphabold$ that satisfy one of the following \textit{feasibility conditions} for interpolation:
\begin{subequations}
\begin{align}
    \Atrain \alphabold &= \Ytrain \text{ or } \label{eq:interpolatingsoln_binary}\\
    \Atrain \alphabold &= \Ztrain \label{eq:interpolatingsoln_real}
\end{align}
\end{subequations}
In particular, we denote the \textit{minimum-$\ell_2$-norm interpolation on binary labels} as

\begin{align*}
  \alphahatbold_{2,\mathsf{binary}} := \argmin_{\alphabold \in \R^d} \|\alphabold\|_2 \text{ s.t. Equation~\eqref{eq:interpolatingsoln_binary} holds} .
\end{align*}
Similarly, we denote the \textit{minimum-$\ell_2$-norm interpolation on real labels} as
\begin{align*}
  \alphahatbold_{2,\mathsf{real}} := \argmin_{\alphabold \in \R^d} \|\alphabold\|_2 \text{ s.t. Equation~\eqref{eq:interpolatingsoln_real} holds}.
\end{align*}
\end{definition}
Recall from our discussion in Section~\ref{sec:relwork_lossfunction} that these interpolations arise from minimizing the square loss on training data.
If we instead minimized the logistic or hinge loss, we would obtain the hard-margin \textit{support vector machine} (SVM), defined below.

\begin{definition}\label{def:l2SVM}
For \textit{linearly separable data}, the hard-margin Support Vector Machine (SVM) is $\alphahatbold_{\mathsf{SVM}} \in \R^d$, defined by
\begin{align}
  \alphahatbold_{\mathsf{SVM}} &:= \argmin_{\alphabold \in \R^d}  \quad \|\alphabold\|_2 \nonumber \\
  \text{s.t.} & \quad Y_i \avec(X_i)^\top \alphabold \geq 1 \quad \text{for all $i=1,\dotsc,n$} . \label{eq:l2SVM}
\end{align}
Note that data is defined to be linearly separable \textit{iff} the constraints in Equation~\eqref{eq:l2SVM} can be feasibly satisfied by some parameter vector $\alphabold$.
\end{definition}

As long as $d \geq n$, any solution that interpolates the binary labels $\{Y_i\}_{i=1}^n$ satisfies Equation~\eqref{eq:l2SVM} with equality \textit{almost surely} for any continuous distribution on the features.
Thus, in the overparameterized regime, the training data is trivially linearly separable.
Note, however, that the feasibility constraints do not require the SVM solution to interpolate the binary labels.

The standard metrics for test error in regression and classification tasks are, respectively, the mean-square-error (MSE) and classification error, defined as follows.
In these definitions, we have ignored the irreducible error terms arising from possible additive noise in real outputs and label noise in binary outputs respectively.
This reflects the practical goal of all prediction to get the underlying true output right, as opposed to matching noisy measurements of that underlying true output.

\begin{definition}
  The excess mean-square-error (MSE) of $\alphahatbold \in \R^d$ is
  \begin{align}
    \regloss(\alphahatbold) & := \EE[\inprod{\avec(X)}{\alphastar - \alphahatbold}^2] \label{eq:regloss} .
  \end{align}
 The excess classification error of $\alphahatbold \in \R^d$ is given by
  \begin{align}
    \classloss(\alphahatbold) & := \EE\left[\mathbb{I}\left[\sgn(\inprod{\avec(X)}{\alphastar}) \neq \sgn(\inprod{\avec(X)}{\alphahatbold}) \right]\right] \nonumber \\
      & \hphantom:=
    \Pr \left[\sgn(\inprod{\avec(X)}{\alphastar}) \neq \sgn(\inprod{\avec(X)}{\alphahatbold}) \right] .
      \label{eq:classloss}
  \end{align}
Here, all expectations (and ensuing probabilities) are only over the random sample $X$ of test data.
As is standard, we will characterize the regression and classification test errors with high probability over the randomness in the training data $\{X_i,Y_i\}_{i=1}^n$.
\end{definition}
As a final comment, we will typically construct an empirical estimate of both test error metrics from $n_{test}$ \text{test samples} of data drawn without any label noise.
This is for ease of empirical evaluation.

\subsection{Featurization}

We consider zero-mean Gaussian featurization, i.e. for every $i \in \{1,\ldots,n\}$, we have
\begin{align}\label{eq:gaussiandef}
    \phi(X_i) \sim \mathcal{N}(\boldsymbol{0},\Sigmabold) .
\end{align}
We denote the spectrum of the (positive definite) covariance matrix $\Sigmabold$ by the vector $\lambdabold := \begin{bmatrix} \lambda_1 & \ldots & \lambda_d \end{bmatrix}$, where the eigenvalues are sorted in descending order, i.e. we have $\lambda_1 \geq \lambda_2 \geq \ldots \geq \lambda_d > 0$.

Throughout, we will consider various \textit{overparameterized ensembles} obtained by scaling the covariance parameter $\Sigmabold$ as a function of both the number of training data points, $n$, and the number of features, $d$.
We theoretically characterize the performance of solutions for classification and regression tasks using two representative ensembles, defined below.

\begin{definition}[Isotropic ensemble$(n,d)$]\label{def:isotropic_covariance}
    The \textit{isotropic ensemble}, parameterized by $(n,d)$, considers isotropic Gaussian features, $\Sigmabold = \mathbf{I}_d$.
    For this ensemble, we will fix $n$ and study the evolution of various quantities as a function of $d \geq n$.
\end{definition}
Note that the isotropic ensemble constitutes the ``maximal" level of effective overparameterization (as defined in the second effective rank in~\citet{bartlett2020benign}) for a given choice of $(n,d)$.

\begin{definition}[Bi-level ensemble$(n,p,q,r)$]
\label{def:bilevel_covariance}
The \textit{bi-level ensemble} is parameterized by $(n,p,q,r)$, where\footnote{We restrict $(p,q,r)$ to this range to ensure that a) the regime is truly overparameterized (choice of $p$), b) the eigenvalues of the ensuing covariance matrix are always positive and ordered correctly (choice of $q$), c) the number of ``high-energy" directions is sub-linear in $n$ (choice of $r$).} $p > 1, 0 \leq r < 1$ and $0 < q < (p-r)$. Here, parameter $p$ controls the extent of artificial overparameterization), $r$ sets the number of preferred features, and $q$ controls the weights on preferred features and thus effective overparameterization. In particular, this ensemble sets parameters
\begin{align*}
    d &:= n^p \\
    s &= n^r \text{ and } \\
    a &= n^{-q} .
\end{align*}
The covariance matrix of the Gaussian features $\Sigmabold(p,q,r)$ is set to be a diagonal matrix, whose entries are given by:
\begin{align*}
    \lambda_j=  \begin{cases}
       \frac{a d}{s}, & 1 \leq j \leq s\\
            {\frac{(1 - a)d}{d-s}}, & \text{otherwise}.
     \end{cases}
\end{align*}
For this ensemble, we will fix $(p,q,r)$ and study the evolution of various quantities as a function of $n$.
\end{definition}
The bi-level covariance matrix is %
parameterized by the choice for the top $s$ eigenvalues and the bottom $(d-s)$ eigenvalues, with the sum of eigenvalues being invariant(equal to $d$).
The parameters of critical importance are $p$, which determines the extent of overparameterization (i.e. number of features), $r$, which determines the number of larger eigenvalues, and $q$, which determines the relative values of larger and smaller eigenvalues (all as a function of the number of training points $n$).
We make a few remarks below on this ensemble.

\begin{remark}
This bi-level ensemble is inspired by the study of estimation of high-dimensional \textit{spiked covariance matrices} ~\citep[e.g.][]{wang2017asymptotics,mahdaviyeh2019risk} when the number of samples is much smaller than the dimension.
In these spiked matrices, the parameter $s$ is typically set to a constant (that does not grow with $n$), and the top $s$ eigenvalues are highly spiked with respect to the other $(d-s)$ eigenvalues.
In fact, it is assumed that there exists a universal positive constant $C$, such that the smaller eigenvalues are bounded and the top (larger) eigenvalues grow with $(d,n)$ in the following way:
\begin{subequations}
\begin{align}\label{eq:spikedensemble}
    \lambda_j &\geq \frac{d}{Cn} \quad \text{ for all } j \in \{1,\ldots, s\} \\
    \lambda_j &\leq C \quad \text{ for all } j \in \{s + 1,\ldots, d\} .
\end{align}
\end{subequations}
Under these conditions, the ratio of the top to the bottom eigenvalues grows as $\Omega\left(\frac{d}{n}\right)$, and Wang and Fan~\citep{wang2017asymptotics} show that the top $s$ estimated eigenvalues of the high-dimensional covariance matrix can be estimated reliably from samples, even when less than the dimension (i.e. $n < d$).
This condition, which is also critical for good generalization\footnote{In particular, avoiding signal shrinkage, as also shown in~\citet{bartlett2020benign}.} in \textit{regression problems}, can be verified to be equivalent to the condition $q \leq (1 - r)$ in our bi-level ensemble (see Theorem~\ref{theorem:threeregimes} for a full statement).
Our definition of the bi-level ensemble allows further flexibility in the choice of these parameters, and we will later show that classification tasks can generalize well even in the absence of this condition.
\end{remark}

\begin{remark}
The bi-level ensemble can be verified to match the isotropic ensemble (Definition~\ref{def:isotropic_covariance}) as a special case when the parameters are set as $q + r = p$.
This case represents the maximal level of effective overparameterization, and in general we take $q \leq (p-r)$ to ensure correct ordering of the eigenvalues.
The smaller the value of $q$, the less the effective overparameterization. The models of~\cite{chatterji2020finite} are spiritually related in how they also use an exponent like $q$ to control the effective overparameterization.
\end{remark}

\begin{remark}
We know that for ``benign overfitting"~\citep{bartlett2020benign} of additive noise to occur in regression problems, we need to have sufficiently many (growing super-linearly in $n$) ``unimportant" directions, corresponding to the lower level of eigenvalues.
The choice of parameters $p > 1$ and $r < 1$ ensures that the number of such ``unimportant" directions is equal to $(d - s) = (n^p - n^r) \gg n$, and so the bi-level ensemble as defined does not admit the regime of harmful overfitting of noise for any choice of parameters $(p,q,r)$.
This allows us to isolate signal shrinkage as the principal reason for large regression error, and also study the ramifications of such shrinkage for classification error.
\end{remark}
In addition to the above, we empirically study (in Section~\ref{sec:margin}) the behavior of various quantities for two other ensembles defined below, both of which have been previously studied in regression tasks.

\begin{definition}[``Weak features" ensemble]
\label{def:weakfeatures}
This ensemble is a
simplification of feature families introduced by~\citet{belkin2020two,hastie2019surprises,mei2019generalization}, all of which demonstrate an explicit benefit of overparameterization in generalization for regression tasks.
The features consist of raw, $1$-dimensional random variable $X_i \sim \mathcal{N}(0,\sigma^2)$, and independent $d$-dimensional random variables $\mathbf{W}_i \text{ i.i.d.} \sim \mathcal{N}(\mathbf{0}, \mathbf{I}_d)$ for $i \in \{1,\ldots,n\}$.
Then, for a given value of $d > n$, each lifted feature is given by:

\begin{align}\label{eq:weakfeaturevector}
\avec(X_i) = X_i\mathbf{1_d} + \mathbf{W}_i.
\end{align}
We define the real output $Z_i$, corresponding to input $X_i$, as $Z_i = X_i$, and similarly the binary output $Y_i$ is defined as $Y_i = \mathsf{sgn}(X_i)$.
For this ensemble, we will fix $n$ and study the evolution of various quantities as a function of $d$. This model allows us to study model mis-specification because the real output is not exactly representable in the feature space.
\end{definition}

Define the column vector of raw features as $\xtrain = \begin{bmatrix} X_1 & X_2 & \ldots & X_n\end{bmatrix}^\top$ and the matrix $\mathbf{W} \in \R^{n \times d}$ whose  $i_{th}$ row is $\mathbf{W}_i$. For the above featurization, the training matrix admits the simple form $\begin{bmatrix} \xtrain & \xtrain & \ldots & \xtrain \end{bmatrix} + \mathbf{W}$. 

\begin{definition}[``Polynomial decay of eigenvalues" ensemble]\label{def:kpweighting}
This ensemble is inspired by commonly chosen reproducing kernel Hilbert spaces and is parameterized by $m \geq 0$.
We set the spectrum of the covariance matrix $\Sigmabold$ to be

\begin{align}
\label{eq:kpweighting}
    \lambda_k = \frac{1}{k^m} \text{ for all } k  \in \{1,2,\ldots,d\} .
\end{align}
For this ensemble, we will fix $(n,d)$ and study the evolution of various quantities as a function of the parameter $m$.
\end{definition}

\section{Approximating the SVM by exact interpolation}\label{sec:svm}

\begin{figure}
  \centering
  \begin{tabular}{p{0.48\textwidth}p{0.48\textwidth}}
    \centering
    \includegraphics[width=0.5\textwidth]{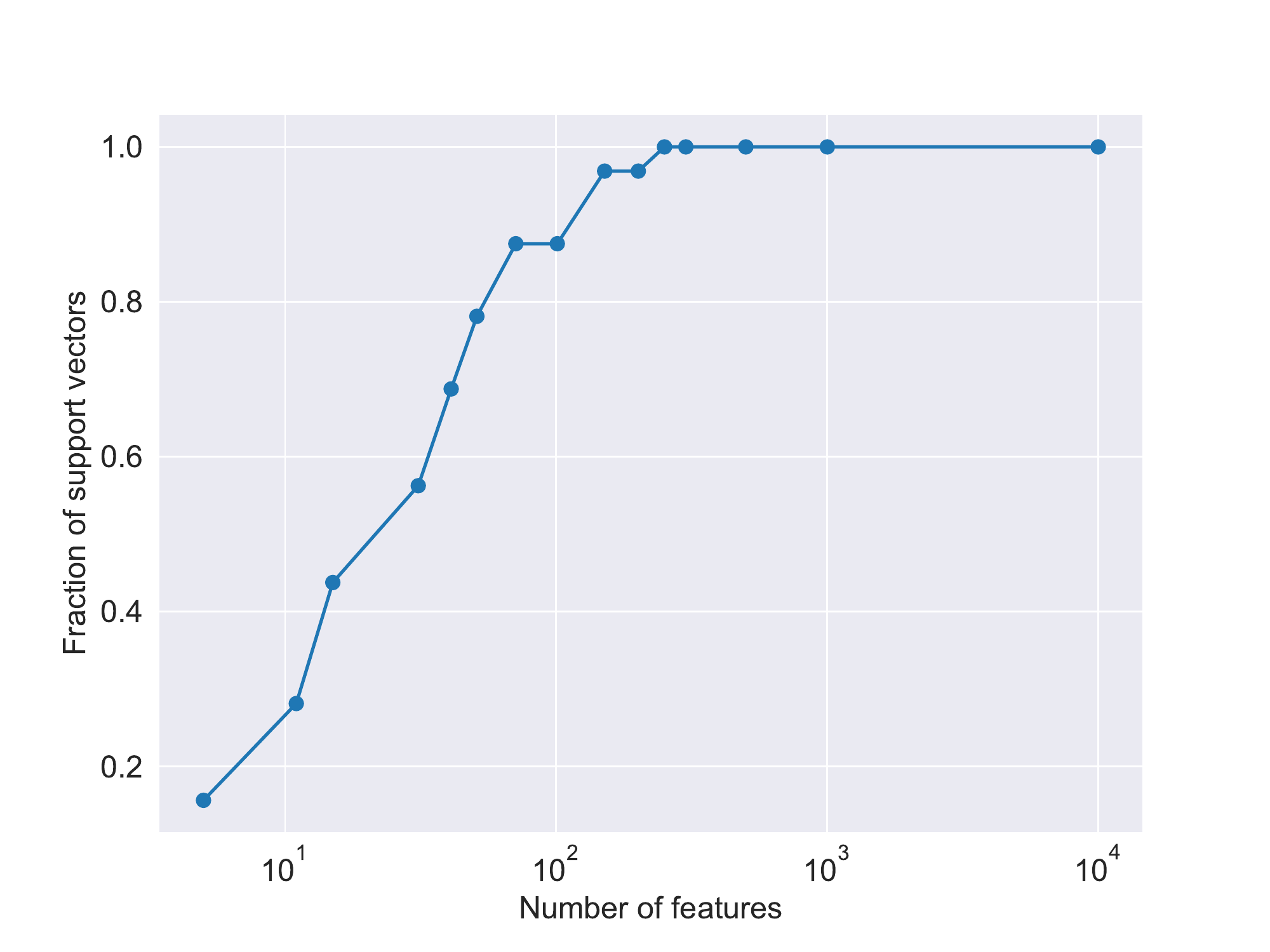}
    &
    \centering
    \includegraphics[width=0.5\textwidth]{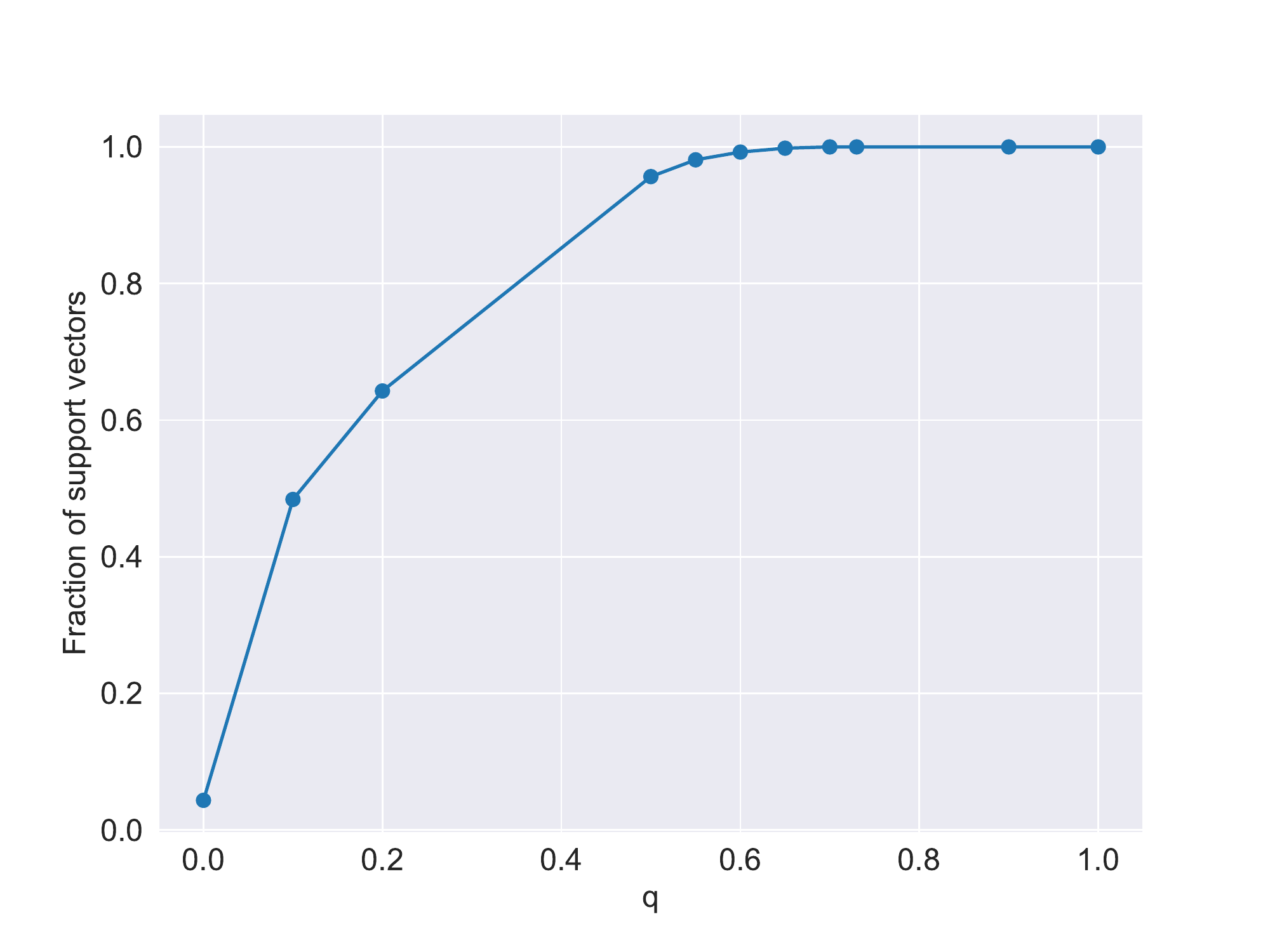}
    \tabularnewline
    \centering
    (a) Isotropic features with $n = 32$ fixed.
    Equation~\eqref{eq:everythingsupportconditiongaussian} would require $d \geq 1141$.
    &
    \centering
    (b) Bi-level ensemble with~$(n = 529, d = 12167)$ fixed and parameters $(p = 3/2, r = 1/2)$ fixed.
    As the parameter $q$ increases, the fraction of support vectors  increases.
  \end{tabular}
  \caption{Experimental illustration of Theorem~\ref{thm:everythingsupport} for Gaussian features: fraction of the training points that are support vectors increases as effective overparameterization increases.}
  \label{fig:frac_support}
\end{figure}

From the optimization objective and constraints defined in Equation~\eqref{eq:l2SVM}, we can see that there is a continuum of margins, defined by $Y_i \cdot \phi(X_i)^\top \alpha$, that is possible for each training point. 
Thus, unlike in least-squares regression, even obtaining an exact expression for the margin-maximizing SVM solution, $\alphahatbold_{\mathsf{SVM}}$, appears difficult in the overparameterized regime.

The heart of our approach to tractably analyzing the SVM involves making an explicit link to minimum-$\ell_2$-norm interpolation, by showing that \textit{all the training data points usually become support vectors} in a sufficiently overparameterized regime.
We actually first identified this phenomenon in visualizations of the
SVM and the minimum-$\ell_2$-norm interpolation using \textit{Fourier
  features} on one-dimensional data; details of this auxiliary
experiment are contained in Appendix~\ref{appendix:fourier}. The
number of support vectors was also recently empirically
observed\footnote{In \citep{snyder2020sample}, the interest is primarily in showing that the number of support vectors constrains the complexity of the learned model and can be related to generalization performance. Our results here show that good generalization is possible even when everything becomes a support vector.} to increase with the number of model parameters~\citep{snyder2020sample}.
In fact, we believe that these ideas are spiritually connected to the well-known folk wisdom that the number of support vectors tends to  proliferate when increasing the ``bandwidth" of kernels like the radial basis function (RBF) kernel.

Such a phenomenon, when true, explicitly links the concept of support-vector-machines with a positive margin constraint to exact interpolation of the training data labels, and suggests a roadmap to analyzing the generalization error of the SVM by analyzing the latter solution (which we do subsequently in Section~\ref{sec:generalization}).
We show in the theorem below that this phenomenon manifests with high probability provided there is sufficient \textit{effective overparameterization}.

\begin{theorem}\label{thm:everythingsupport}
Let $\Atrain$ follow the Gaussian featurization from Equation~\eqref{eq:gaussiandef} with covariance matrix $\Sigmabold$, and let $\alphahatbold_{\mathsf{SVM}}$ be the solution to the optimization problem in Equation~\eqref{eq:l2SVM}.
  \begin{enumerate}
    \item
      If $\Sigmabold$ satisfies
      \begin{align}\label{eq:everythingsupportcondition}
          ||\lambdabold||_1 \geq 72 \left(||\lambdabold||_2 \cdot n\sqrt{\ln n} + ||\lambdabold||_{\infty} \cdot n\sqrt{n} \ln n + 1\right) ,
      \end{align}
      the vector $\alphahatbold_{\mathsf{SVM}}$ satisfies the binary label interpolation constraint (Equation~\eqref{eq:interpolatingsoln_binary}) simultaneously for every $\Ytrain \in \{\pm1\}^n$ with probability at least $\left(1 - \frac{2}{n}\right)$.
\item
      If $\Sigmabold = \mathbf{I}_d$ (i.e., $\Atrain$ follows the isotropic ensemble), and
      \begin{align}\label{eq:everythingsupportconditiongaussian}
          d > 10n \ln n + n - 1 ,
      \end{align}
      then the vector $\alphahatbold_{\mathsf{SVM}}$ satisfies the binary label interpolation constraint (Equation~\eqref{eq:interpolatingsoln_binary}) for any fixed $\Ytrain \in \{\pm1\}^n$ with probability at least $\left(1 - \frac{2}{n}\right)$.
  \end{enumerate}
\end{theorem}
Theorem~\ref{thm:everythingsupport} is proved in Appendix~\ref{appendix:svmproofs}.
Both conditions are proved by showing that a complementary slackness condition on the dual of the SVM optimization problem holds with high probability --- where the conditions differ is in the application of concentration bounds.
The condition in Equation~\eqref{eq:everythingsupportcondition} is proved using a broadly applicable ``epsilon-net" argument to bound the operator norm of a random matrix, while the sharper condition in Equation~\eqref{eq:everythingsupportconditiongaussian} leverages Gaussian isotropy and precise properties of the inverse Wishart distribution.
Note that the first result holds for all label vectors $\Ytrain \in \{\pm1\}^n$ simultaneously, while the second result holds for any fixed $\Ytrain$ (but independent of the features).

We now remark on the result for the isotropic and the bi-level ensemble.
\begin{remark}
Plugging in the condition in Equation~\eqref{eq:everythingsupportcondition} into the bi-level ensemble (Definition~\ref{def:bilevel_covariance}), the following conditions on $(p,q,r)$ are \textit{sufficient} for all training points to become support vectors with high probability (see Appendix~\ref{appendix:svm_bilevel} for a full calculation):
\begin{subequations}\label{eq:everythingsupportcondition_bilevel}
\begin{align}
p &> 2 \text{ and } \label{eq:everythingsupportcondition_bilevel_p} \\
q &> \left(\frac{3}{2} -r \right) \label{eq:everythingsupportcondition_bilevel_q} .
\end{align}
\end{subequations}
There is an intuitive interpretation for each of these conditions in light of the second ``effective rank" condition that is sufficient for benign overfitting~\citep{bartlett2020benign} of noise (although our proof technique is quite different).
First, the condition $p > 2$ mandates an excessively large number of unimportant directions, i.e.~corresponding to lower-level (smaller) eigenvalues ($(n^p - n^r)$ of them).
Second, the condition $q > \left(\frac{3}{2} - r\right)$ mandates that the ratio between the important directions, i.e. higher-level eigenvalues, and the unimportant directions, is sufficiently small --- thus, the unimportant directions are sufficiently weighted.
This second condition appears to be strictly stronger than what is required for benign overfitting of noise.

Equation~\eqref{eq:everythingsupportcondition_bilevel} is quite strong as a sufficient condition, but nevertheless admits non-trivial regimes for which classification can generalize well or poorly (see the text accompanying Theorem~\ref{theorem:threeregimes} for a full discussion).
However, there is ample evidence to suggest that this condition is not necessary.
Notably, Figure~\ref{fig:frac_support}(b) shows that with a choice of parameterization $(p = 3/2, r = 1/2)$, the fraction of support vectors becomes equal to $1$ around when $q \geq 0.7$.
This choice of parameters for the bi-level ensemble clearly violates both conditions in Equation~\eqref{eq:everythingsupportcondition_bilevel}.
Thus, all training points become support vectors more often than our theory predicts.
Subsequent work to ours~\citep{hsu2021proliferation} tightened the condition in Equation~\eqref{eq:everythingsupportcondition} by providing a new deterministic equivalent to the phenomenon of all training points becoming support vectors.
\end{remark}

From this section, we have identified an interesting phenomenon by which all training points become support vectors with sufficient effective overparameterization.
Moreover, this phenomenon is even more prevalent empirically than our current theory predicts.

\section{Generalization analysis for interpolating solution with Gaussian features}\label{sec:generalization}

In Section~\ref{sec:svm}, we showed that the SVM solution often exactly corresponds to the \textit{minimum-$\ell_2$-norm} interpolation on binary labels, denoted by $\alphahatbold_{2,\mathsf{binary}}$.
In this section, we attempt an approximate characterization of the ensuing classification error of this interpolation.
Our hope is that we can leverage comprehensive analyses of minimum-$\ell_2$-norm interpolation for least-squares regression~\citep{bartlett2020benign,muthukumar2020harmless}.
However, it turns out that direct plug-ins of these analyses do not work for a number of reasons:
\begin{enumerate}
    \item Even with clean data (i.e. zero label noise), the classification setup admits misspecification noise of the form $Y_i - \phibold(X_i)^\top \alphastar$.
    The misspecification noise is clearly non-zero mean, and is non-trivially correlated with the features.
    This resists a clean decomposition of generalization error into the error arising from signal identifiability (or lack thereof) + error arising from overfitting of noise, as in~\citet{bartlett2020benign}.
    \item For a given interpolation $\alphahatbold$, the expression for classification error is distinctly different from mean-square-error (we will see this explicitly in Theorem~\ref{theorem:classloss}).
    In particular, we will see that characterizing this expression sharply requires novel analysis of the individual recovered coefficients as a result of interpolation.
\end{enumerate}
Our analysis is subsequently non-trivial to engage with both of these difficulties, and directly addresses both of them by analyzing the minimum-$\ell_2$-norm interpolator of binary labels from first principles.
This is, roughly speaking, in two steps: first, by characterizing the expected generalization error in terms of 0-1 classification loss \textit{for any solution} (regardless of whether it interpolates or not) as a function of \textit{survival} and \textit{contamination} factors; second, by obtaining sharp characterizations of these factors for the minimum-$\ell_2$-norm interpolator of binary labels.

\subsection{Setup and result}

We state our main result for this section in the context of the bi-level ensemble (Definition~\ref{def:bilevel_covariance}).
We fix parameters $p > 1$ (which represents the extent of artificial overparameterization), and~$r \in [0,1)$ (which sets the number of preferred features), and $q \in [0,p-r]$ (which controls the weights on preferred features, thus effective overparameterization); and study the evolution of regression and classification risk as a function of $n$.
For the purpose of this section, we denote the regression and classification test losses under the bi-level ensemble as $\reglossn$ and $\classlossn$, to emphasize that these losses vary with $n$.

In addition to this and the broad setup as described in Section~\ref{sec:setup} we make a \textit{$1$-sparse assumption} on the unknown parameter vector $\mathbf{\alphastar}$, as described below.

\begin{assumption}[$1$-sparse linear model]
Recall that the bi-level ensemble sets $s := n^r$.
For some unknown\footnote{The intuition for this condition, also motivated in prior analyses of minimum-$\ell_2$-norm interpolation~\citep{muthukumar2020harmless}, is that for any reasonable preservation of signal, the true feature needs to be sufficiently preferred, therefore weighted highly.} $t \in \{1,\ldots,s\}$, we assume that $\mathbf{\alphastar} = \frac{1}{\sqrt{\lambda_t}} \cdot \evec_t$, i.e. the parameter vector $\mathbf{\alphastar}$ is $1$-sparse.
\label{assumptionsimple}
\end{assumption}
Assumption~\ref{assumptionsimple} is most useful to for us to derive clean expressions for regression and classification error in terms of natural notions of ``survival" and ``contamination", as detailed subsequently in Section~\ref{sec:pathtoanalysis}.
While this assumption appears rather strong, it is actually without loss of generality within the bi-level ensemble \textit{for analyzing the performance of minimum-$\ell_2$-norm interpolation specifically}.
If the true parameter vector $\alphastar$ has support only within the $s$ favored directions, then we can choose another orthonormal coordinate system in which this $\alphastar$ is only along the first direction.
Because minimum-$\ell_2$-norm interpolation does not care about orthonormal coordinate changes and such a change will not change the underlying covariance matrix, we just assume $1$-sparsity to capture the representability of the true model by the favored features.

Under Assumption~\ref{assumptionsimple}, we now show the existence of a regime, corresponding to choice of $(p,q,r)$ above, for which the regression test loss stays prohibitively high, but the classification test loss goes to $0$ as $n \to \infty$.
(We also derive non-asymptotic versions of these results in Appendix~\ref{appendix:bilevel}, but only state the asymptotic results here for brevity.)

\begin{theorem} \label{thm:main}
Assume that the true data generating process is $1$-sparse (Assumption~\ref{assumptionsimple}). For the bi-level covariance matrix model, the limiting classification and regression error of the minimum-$\ell_2$-norm interpolation (of binary labels and real labels respectively) converge in probability, over the randomness in the training data, as a function of the parameters $(p,q,r)$ in the following way:
\begin{enumerate}
    \item For $0 \leq q < (1 - r)$, we have
\begin{align*}
    \lim_{n \rightarrow \infty} \reglossn =0 ,\\
       \lim_{n \rightarrow \infty} \classlossn = 0 .
\end{align*}
In this regime, both regression and classification generalize well.

\item For $(1 - r) < q < (1 - r) + \frac{(p-1)}{2}$, we have
\begin{align*}
    \lim_{n \rightarrow \infty} \reglossn =1 ,\\
       \lim_{n \rightarrow \infty} \classlossn =0 .
\end{align*}

In this regime, classification generalizes well but regression does not.
\item For $(1 - r) + \frac{(p-1)}{2} < q \leq (p - r)$, we have
\begin{align*}
    \lim_{n \rightarrow \infty} \reglossn =1 ,\\
       \lim_{n \rightarrow \infty} \classlossn = \frac{1}{2} .
\end{align*}
In this regime, the generalization is poor for both classification and regression.
\end{enumerate}

Note that the presence of label noise $\nu^*$ does not affect these asymptotic scalings (since $\nu^* < 0.5$).

\label{theorem:threeregimes}
\end{theorem}

\begin{figure}[t]
  \centering
  \begin{tabular}{p{0.49\textwidth}p{0.49\textwidth}}
    \centering
    \includegraphics[width=0.5\textwidth]{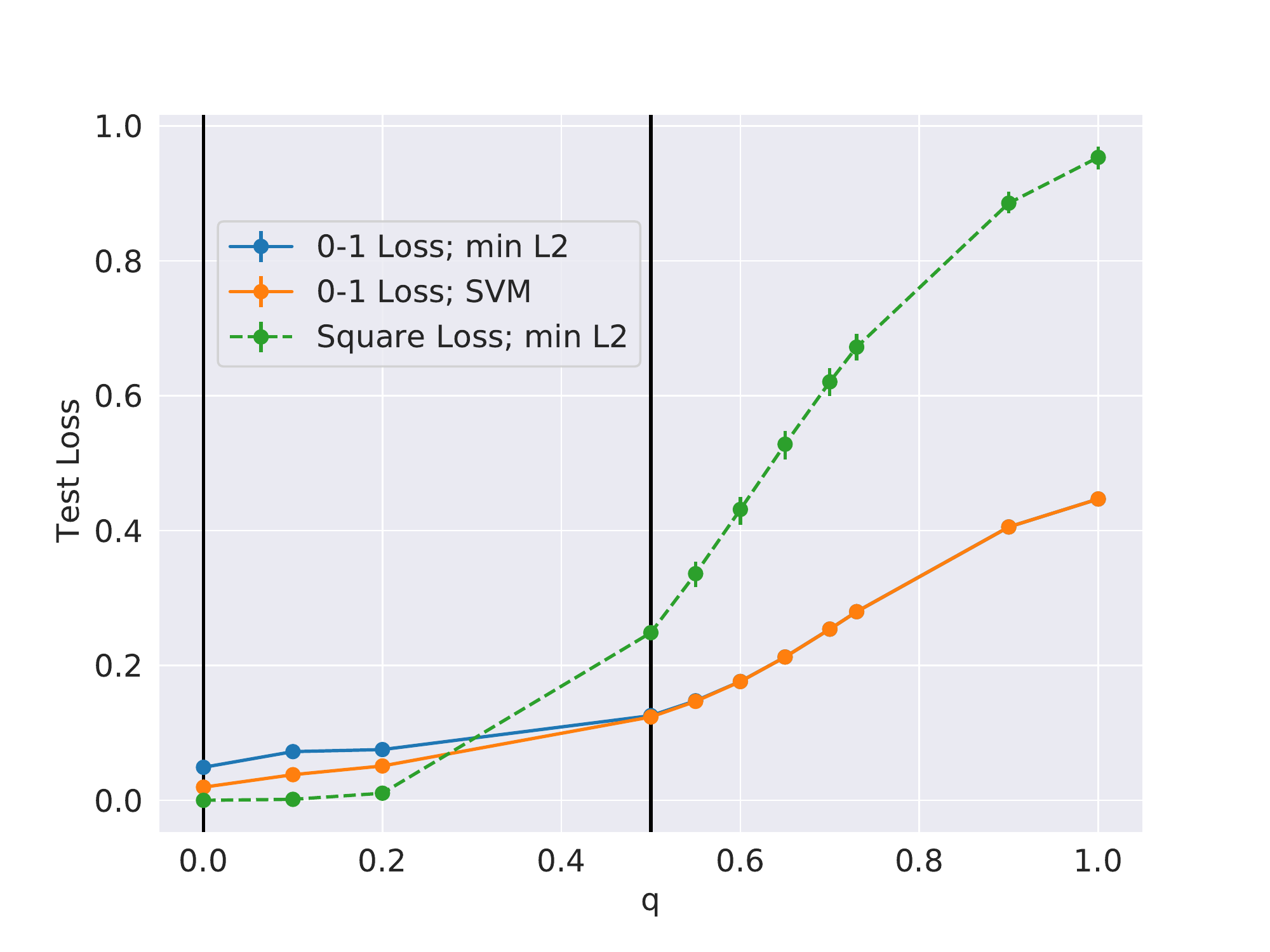}
    &
    \centering
    \includegraphics[width=0.5\textwidth]{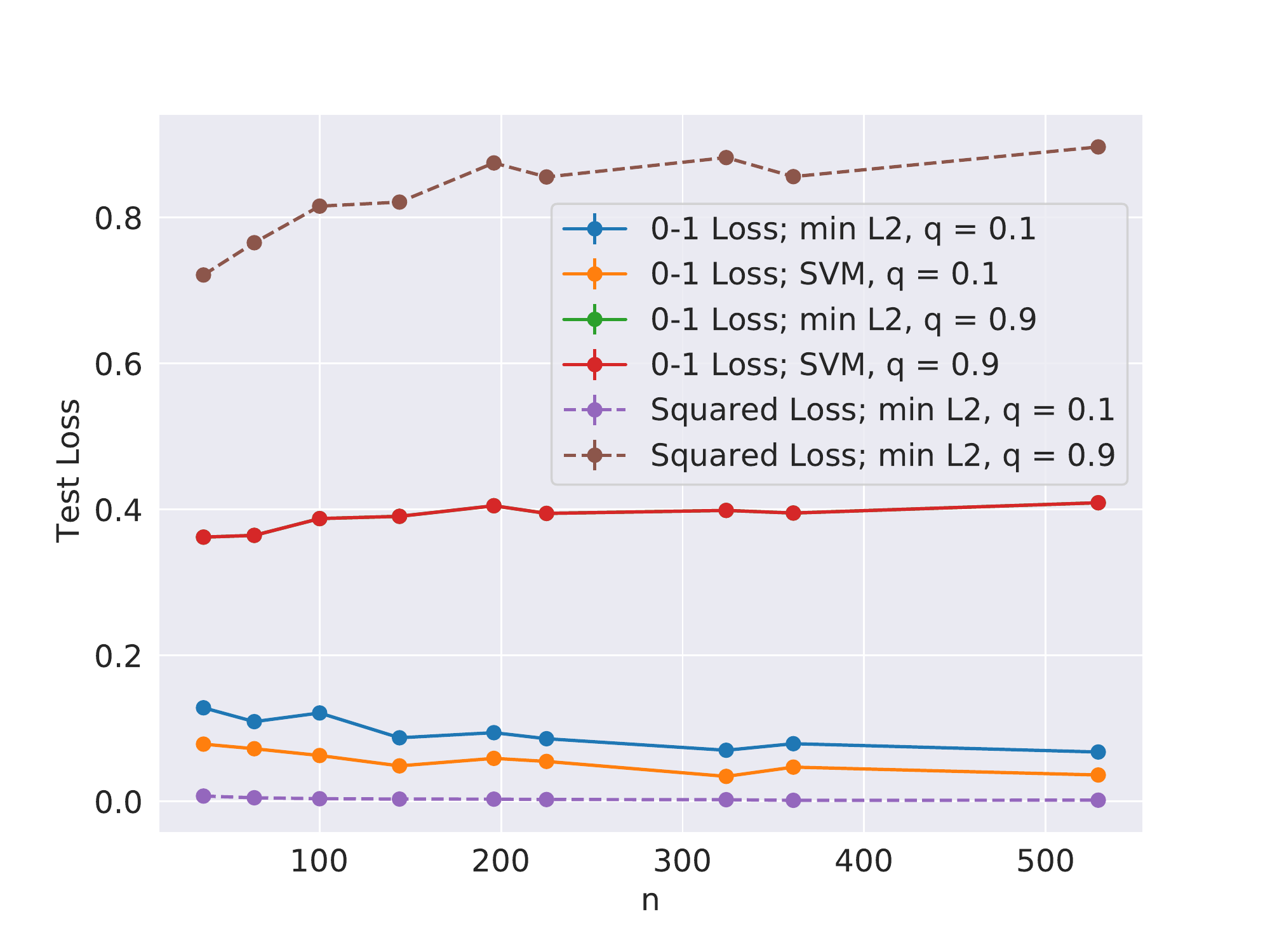}
    \tabularnewline
    \centering
    (a) Test loss vs $q$ 
    &
    \centering
    (b) Test loss vs number of training points
  \end{tabular}
  \caption{Comparison of test classification and regression error on solutions obtained by minimizing different choices of training loss on the bi-level ensemble. For both figures, parameters $(p = 3/2, r = 1/2)$ are fixed. On the left, $n=529, d=12167$ are fixed. Here, the dashed green curve corresponds to $\alphahatbold_{2,\mathsf{real}}$ (Equation \ref{eq:interpolatingsoln_real}), the orange curve corresponds to $\alphahatbold_{2,\mathsf{binary}}$ (Equation~\ref{eq:interpolatingsoln_binary}), the solid blue curve corresponds to $\alphahatbold_{\mathsf{SVM}}$ (Equation \ref{eq:l2SVM}), and the black lines demarcate the regimes from Theorem~\ref{theorem:threeregimes}. On the right, $d$ varies as $n^{\frac{3}{2}}$.}
  \label{fig:classification_vs_regression}
\end{figure}

Figure~\ref{fig:classification_vs_regression}(a) shows the evolution of classification and regression error as a function of the parameter $q$, fixing $p = 3/2$ and $r = 1/2$.
The classification error is plotted for both the SVM and the minimum-$\ell_2$-norm interpolation --- as we expect from Theorem~\ref{thm:everythingsupport}, these are remarkably similar.
Figure~\ref{fig:classification_vs_regression}(b) shows that the empirical quantities converge to the limiting quantities from Theorem~\ref{theorem:threeregimes}.

The new regime of principal interest that we have identified is values of $q\in(1-r,1-r+\frac{p-1}{2})$ for which classification generalizes, but regression does not.
The entire proof of Theorem~\ref{theorem:threeregimes} is deferred to
Appendices~\ref{appendix:gaussianproofs} and~\ref{appendix:bilevel},
but we briefly illustrate the intuition for this discrepancy between classification and regression tasks in Section~\ref{sec:pathtoanalysis}.
In particular, we will see that good generalization for classification requires a far less stringent condition on coefficient recovery than regression.

We now provide some intuition for the scalings described in Theorem~\ref{theorem:threeregimes} for the bi-level ensemble.

\begin{remark}\label{remark:signalregression}
Observe that in this ensemble, regression tasks generalize well \textit{iff} we have $q < (1 - r)$, which is a condition directly related to signal preservation.
Recall that for fixed values of $(p,r)$, the parameter $q$ controls the relative ratio of the larger eigenvalues to the smaller eigenvalues (corresponding to unimportant directions).
The higher the value of $q$, the smaller this ratio, and the harder it is to preserve signal.
The results on ``benign overfitting"~\citep{bartlett2020benign} \textit{upper bound} the contribution of (bounded $\ell_2$-norm) pure signal to regression error.
This upper bound can also be verified to decay with $n$ iff we have $q \leq (1 - r)$.
Furthermore, as we already remarked on Definition~\ref{def:bilevel_covariance}, the bi-level ensemble is designed to always avoid harmful \textit{noise overfitting}.
(We will, however, see in the next remark that the rate of effective noise absorption is important.)
\end{remark}

\begin{remark}
The regime that we have identified that is of principal interest is \textit{intermediate values} of $q$, i.e. $(1 - r) < q < (1 - r) + \frac{(p-1)}{2}$.
This highlights a fascinating role that overparameterization, in the form of the parameter $p$, plays in allowing the good generalization of interpolating solutions in classification tasks.
Recall that the larger the value of $p$, the larger the total number of features $d = n^p$.
Thus, there are several ``unimportant directions" in the bi-level ensemble all corresponding to the smaller eigenvalue --- which helps in harmless absorption of effective noise.
In the proof of Theorem~\ref{theorem:threeregimes}, we will identify an explicit mechanism by which having many unimportant directions helps in good generalization for classification, even though the signal is not preserved.
At a high level, this mechanism constitutes the spreading out of attenuated signal across several features in a relatively ``harmless" way, to exhibit minimal influence on classification performance.
In fact, this influence is quantified by a notion of ``contamination" by falsely discovered features (defined in Section~\ref{sec:pathtoanalysis}) that can be directly linked to the contribution of noise overfitting to regression error.
\end{remark}
Finally, we remark that Theorem~\ref{theorem:threeregimes} provides a connection between classification and regression test error when both tasks are solved using the minimum-$\ell_2$-norm interpolation, i.e. minimizing the square loss on training data.
Since we explicitly linked the minimum-$\ell_2$-norm interpolation and the SVM in the preceding Section~\ref{sec:svm}, it is natural to ask whether the generalization results in Theorem~\ref{theorem:threeregimes} help us directly compare the SVM for classification tasks and the minimum-$\ell_2$-norm interpolation for regression tasks.
We can indeed do this in a slightly more restricted regime of the bi-level ensemble, described below.

\begin{corollary}\label{cor:svmgeneralization}
Assume that the true data generating process is $1$-sparse (Assumption~\ref{assumptionsimple}).
Consider the bi-level ensemble with $p > 2$.
Then, the classification error of the SVM (on binary labels), and the regression error of the minimum-$\ell_2$-norm interpolation (on real labels), converge in probability as follows:
\begin{enumerate}
    \item For $\left(\frac{3}{2} - r\right) < q < (1 - r) + \frac{(p-1)}{2}$, we have
    \begin{align*}
    \lim_{n \rightarrow \infty} \regloss(\alphahatbold_{2,\mathsf{real}}) =1 ,\\
       \lim_{n \rightarrow \infty} \classloss(\alphahatbold_{\mathsf{SVM}}) =0 .
    \end{align*}
    \item For $(1 - r) + \frac{(p-1)}{2} < q \leq (p - r)$, we have
     \begin{align*}
    \lim_{n \rightarrow \infty} \regloss(\alphahatbold_{2,\mathsf{real}}) =1 ,\\
       \lim_{n \rightarrow \infty} \classloss(\alphahatbold_{\mathsf{SVM}}) = \frac{1}{2}.
    \end{align*}
\end{enumerate}
\end{corollary}
Observe that Corollary~\ref{cor:svmgeneralization} directly follows from plugging in the condition required in the bi-level ensemble for all training points usually becoming support vectors (Equation~\eqref{eq:everythingsupportcondition_bilevel}), and noting that for $p > 2$, we have
\begin{align*}
    (1 - r) + \frac{(p-1)}{2} > (1 - r) + \frac{1}{2} = \left(\frac{3}{2} - r\right) .
\end{align*}
Importantly, we have identified that even highly overparameterized regimes, in which all training points become support vectors, can yield good generalization for classification tasks when the hard-margin SVM is used.

\subsection{Path to analysis: Classification vs regression test error}\label{sec:pathtoanalysis}

The first step to proving Theorem~\ref{theorem:threeregimes} is obtaining clean expressions for both classification and regression test error.
The $1$-sparsity assumption that we have made on the unknown signal enables us to do this as a function of natural quantities corresponding to the preservation of the true feature (\textit{survival}) and the pollution due to  false features (\textit{contamination}).
If we assume that the real labels are generated by the ${\truek}^{\text{th}}$ feature, $\alpha^*_{\truek}$, then we can define these quantities for any solution $\alphahatbold$.
First, as classically observed in statistical signal processing, the estimated coefficient corresponding to the true feature $\alpha^*_{\truek}$ will experience \textit{shrinkage} and be attenuated by a factor that we denote as \textit{survival}.
From Assumption~\ref{assumptionsimple}, we defined $\boldsymbol{\alphastar} := \frac{1}{\sqrt{\lambda_t}} \cdot \evec_t$, and so we have
 \begin{align}
     \survivalb(\alphahatbold, \truek) = {\frac{\alphahat_{\truek}}{\alpha^*_{\truek}}} = \sqrt{\lambda_\truek}{\alphahat_{\truek}}  \label{eq:survivalformula}
 \end{align}
Second, we have the \textit{false discovery of features}.
We measure the effect of this false discovery for prediction on a test point $X$ by a \textit{contamination} term:
\begin{align}\label{eq:falsediscovery}
      B = \sumjdk \alphahat_j \phibold_j(\Xvec).
 \end{align}
Recall that $X$ is random, and the features $\avec(X)$ are zero-mean.
Therefore, $B$ is a zero-mean random variable.
Accordingly, we can define the standard deviation of the contamination term on a test point as below:
\begin{align}
     \contaminationb(\alphahatbold, \truek) &= \sqrt{\EE\left[B^2\right]} \nonumber \\
    &= \sqrt{\sumjdk \lambda_j \alphahat_j^2}. \label{eq:contaminationformula}
\end{align}
where the last step follows from the orthogonality of the $d$ features.
The ideas of survival and contamination can be related to the classical signal processing concept of \textit{aliasing}; Figure~\ref{fig:survivalcontamination} in Appendix~\ref{appendix:ultratoy} provides an illustration.

We state and prove the following proposition, which directly expresses regression and classification test loss in terms of these terms.

\begin{proposition}
Under the $1$-sparse noiseless linear model, the regression test loss (excess MSE) is given by:
\begin{align}
    \regloss(\alphahatbold) = (1 - \survivalb(\alphahatbold,\truek))^2 + \contaminationb^2(\alphahatbold,\truek). \label{eq:reglosssuc}
\end{align}
and the classification test loss (excess classification error) is given by:
\begin{align}
   \classloss(\alphahatbold) =  \frac{1}{2} - \frac{1}{\pi}\taninv{\frac{\survivalb(\alphahatbold, \truek)}{\contaminationb(\alphahatbold, \truek)}}. \label{eq:classloss_formula}
\end{align}
We can think of the quantity $\survivalb(\alphahatbold, \truek)/\contaminationb(\alphahatbold, \truek)$ as the effective ``signal-to-noise ratio" for classification problems.
\label{theorem:classloss}
\end{proposition}

\begin{proof}
We first prove Equation~\eqref{eq:reglosssuc}.
Recall that for any estimator $\alphahatbold$, the excess MSE is given by
\begin{align}
     \regloss(\alphahatbold) :&= \EE[(\inprod{\avec(X)}{\alphastar-\alphahatbold})^2] \nonumber \\
     &= \sum\limits_{j=1}^d \lambda_j (\alpha^*_j - \alphahat_j)^2, \nonumber
\end{align}
and then substituting in the $1$-sparse Assumption~\ref{assumptionsimple} gives us Equation~\eqref{eq:reglosssuc}.

Next, we prove Equation~\eqref{eq:classloss_formula}.
Since $\avec(X) = \Sigmabold^{1/2} \W$ for $\W = (W_1,\dotsc,W_d) \sim \mathcal{N}(\mathbf0,\mathbf I_d)$, we can write
  $\avec(X)^\top \alphastar = W_{\truek}$
  and
  $\avec(X)^\top \alphahatbold = \sum_{j=1}^d \sqrt{\lambda_j} W_j \alphahat_j$.
Thus, the excess classification error of $\alphahatbold$ is given by
  \[
    \classloss(\alphahatbold)
    = \PP\left( \avec(X)^\top \alphahatbold \avec(X)^\top \alphastar \leq 0 \right)
    = \PP\left( \sqrt{\lambda_{\truek}} \alphahat_{\truek} W_{\truek}^2 + W_{\truek} \cdot \sum_{j \neq \truek} \sqrt{\lambda_j} \alphahat_j W_j \leq 0 \right)
    .
  \]
Now, the random sum $\sum_{j \neq \truek} \sqrt{\lambda_j} \alphahat_j W_j$ has a Gaussian distribution with mean zero and variance $\contaminationb(\alphahatbold, \truek)^2$.
  Since the $\{ W_j \}_{j=1}^d$ are independent, the classification test error of $\alphahatbold$ is the probability of the following event:
  \[
    \survivalb(\alphahatbold, \truek) U^2 + U \cdot \contaminationb(\alphahatbold, \truek) V \leq 0 ,
  \]
  where $U$ and $V$ are independent standard Gaussian random variables.
  This event is equivalently written as
  \[
    \frac{V}{U} \leq - \frac{\survivalb(\alphahatbold, \truek)}{\contaminationb(\alphahatbold, \truek)} .
  \]
  Since $V/U$ follows the standard Cauchy distribution with cumulative distribution function $F(t) = \tfrac12 + \tfrac1\pi \taninv{t}$, the claim follows.
\end{proof}
Equations~\eqref{eq:reglosssuc} and~\eqref{eq:classloss_formula} give us an initial clue as to why classification test error can be easier to minimize than regression test error.
For the right hand side of Equation~\eqref{eq:reglosssuc} to be small, we need $\survivalb \to 1$ to avoid shrinkage, as well as $\contaminationb \to 0$ to avoid contamination.
However, for the right hand side of Equation~\eqref{eq:classloss_formula} to be small, we only require the ratio of contamination to survival to be small (i.e.~$\contaminationb/\survivalb \to 0$).
Clearly, the former condition directly implies the latter, showing that classification is ``easier" than regression\footnote{Our decomposition of classification error is reminiscent of the decomposition by~\citet{friedman1997bias} into the ratio of terms depending on the variance (like contamination) and bias (like survival) respectively.
Because our data is Gaussian, Proposition~\ref{theorem:classloss} allows an \textit{exact} decomposition.}.
Theorem~\ref{theorem:threeregimes} is proved fully in Appendices~\ref{appendix:gaussianproofs} and~\ref{appendix:bilevel} in the following series of steps:
\begin{enumerate}
    \item Matching (non-asymptotic) upper and lower bounds are proved on both survival and contamination for interpolation of both real and binary labels.
    The full statements for these bounds are contained in Theorems~\ref{theorem:survivalbounds} and~\ref{theorem:contaminationbounds} in Appendix~\ref{appendix:survivalcontaminationtheoremstatements}.
    \item These bounds are substituted into the bi-level ensemble to get asymptotic scalings for classification and regression test error (Appendix~\ref{appendix:bilevel}).
\end{enumerate}
The bulk of the technical work is involved in proving the matching bounds on survival and contamination, i.e.~Theorems~\ref{theorem:survivalbounds} and~\ref{theorem:contaminationbounds}.
Although these results are inspired by the calculations provided in Appendix~\ref{appendix:fourierproof} for the Fourier case, we build on the techniques provided in \citet{bartlett2020benign} for Gaussian features, particularly making use of fundamental concentration bounds that were proved on ``leave-one-out" matrices in that work.
We build on these techniques to sharply bound both the ``survival" and ``contamination" terms, and thus obtain matching upper and lower bounds for the classification test error.
Crucially, our analysis needs to circumvent issues that stem from effective misspecification in the linear model that arise from the sign operator.
While we do not provide a generic analysis of ``misspecification noise,'' we exploit the special misspecification induced by the sign operator in a number of technical equivalents of the aforementioned random matrix concentration results.

We essentially show that this induced misspecification makes no difference, asymptotically, to classification error arising from interpolation from binary labels, and the behavior is essentially the same as though we had instead interpolated the real output.
This is another interesting consequence of requiring only the ratio $\frac{\contaminationb}{\survivalb} \to 0$, as opposed to the stronger requirements for regression, $\contaminationb \to 0$ and $\survivalb \to 1$.
We will see in Appendix~\ref{appendix:bilevel} that in the asymptotic limit $n \to \infty$, interpolation of binary \textit{noiseless} labels attenuates the signal by a factor exactly equal to $\sqrt{\frac{2}{\pi}}$.
This also corresponds to the attenuation factor of signal that has been traditionally been observed as a result of $1$-bit quantization applied before a matched filter\footnote{Recall that~\citet{muthukumar2020harmless} naturally connected matched filtering to minimum-$\ell_2$-norm interpolation.
}~\citep{1454543,chang1982presampling}.
Since this factor is strictly positive, it does not affect the asymptotic classification error.

In fact, the non-asymptotic scalings of survival and contamination terms are unaffected even by non-zero label noise on classification training data, provided that the label noise still preserves non-trivial information about the signal.
The survival is further attenuated by a non-zero factor of $(1 - 2\betastar)$, which is strictly positive as long as $\betastar < 1/2$.
Observe that this is equivalent to a hypothetical scenario where the binary labels take on ``shrunk" values $\{- (1 - 2\betastar), (1 - 2\betastar)\}$ instead of the usual $\{-1,1\}$.
As long as $\betastar < 1/2$, the magnitude of the labels is strictly non-zero and so the labels still provide useful information for classification.

Finally, it is natural to ask how fundamental our assumptions of Gaussianity on data and bi-level covariance structure are to our main generalization result (Theorem~\ref{theorem:threeregimes}). 
We chose the bi-level ensemble to illustrate the separation between classification and regression in the cleanest possible way.
However, Theorems~\ref{theorem:survivalbounds} and~\ref{theorem:contaminationbounds} do provide non-asymptotic expressions for survival and contamination for \textit{arbitrary} covariance matrices.
In principle, these expressions can be plugged into Proposition~\ref{theorem:classloss} to get upper and lower bounds on classification error for arbitrary covariance matrices.
Further, the analysis of benign overfitting in linear
regression~\citep{bartlett2020benign, muthukumar2020harmless} extends to sub-Gaussian features.
In the same spirit, we can show that the results --- including the existence of the intermediate regime, in
which classification works but regression does not --- extend to a
weaker assumption of \textit{independence} and sub-Gaussianity on the underlying features.
This extension uses an argument similar to the Fourier-case argument
given in Appendix~\ref{appendix:fourierproof} but requires a more
direct treatment of the approximation error arising from
misspecification.
We provide this argument in a forthcoming note.
Our results \textit{do not} extend to kernel settings, where there can be complex dependencies among the (infinite-dimensional) features.
This is an important direction for future work.

\section{Examining margin-based explanations for generalization}\label{sec:margin}

In this section, we explore the potential for generalization bounds as a function of training data margin to explain the behavior we have observed for classification tasks in the overparameterized regime.
Through simple experiments, we demonstrate that margin-based generalization bounds are  uninformative in sufficiently overparameterized settings.

\subsection{The historical role of margin}

For a particular function class $\mathcal{F}$, \textit{uniform convergence bounds} conservatively approximate the generalization error of $f \in \mathcal{F}$ by that of the least generalizable function in $\mathcal{F}$.
The ensuing generalization bounds typically depend on measures of complexity, such as the Vapnik-Chervonenkis dimension, which increase with the number of parameters in the model.
Thus, the uniform convergence approximation is not as good when $\mathcal{F}$ is large, e.g. the model has several parameters.
This shortcoming of uniform convergence-based bounds was first brought into focus by the remarkable success of boosting with a very large number of primitive classifiers \citep{schapire1998boosting}.
The main observation was that even after the training 0-1 loss became zero, increasing the number of primitive classifiers in the boosted model still reduced the test error.

An analysis in terms of the \textit{training data margin} was proposed as a possible explanation for this behavior for classifiers $f(\cdot)$ that make their predictions by discretizing the outputs of a real-valued function $g \in \mathcal{G}$, i.e.~$f(X) = \sgn(g(X))$.
The training margin, $\gamma := \min_i Y_i g(X_i)$ can be intuitively interpreted as a measure of prediction confidence; for linear classifiers, it is precisely the minimum (over training points) distance to the decision boundary.
The worst-case margin is not the only quantity that has been considered: generalization bounds based on a weighted combination of margin on all training data points have also been considered and demonstrated to be sharper in certain settings~\citep{gao2013doubt}.
In the settings we investigate, all training data points become support vectors -- therefore {\it the margins at each training point are equal}, and all such notions of margin become equivalent.

Under certain conditions, margin-based generalization bounds can scale far slower with the number of parameters in the model than uniform convergence bounds; for example, in boosting, the dependence is reduced to $\ln(\text{\# of primitive classifiers})$.
Since  the margin $\gamma$ could be artificially increased (without changing any of the predictions) simply by rescaling  the  real-valued function $g(\cdot)$ , the quantity of interest
is an appropriately \textit{normalized} margin, e.g.~the margin normalized by the Lipschitz constant of the learned function $g(\cdot)$ or its approximation.

The decrease of generalization error despite increasing complexity in ``modern'' overparameterized regimes is strongly reminiscent of the observations from boosting with a large number of primitive classifiers.
It is of particular interest to examine the ensuing generalization bounds for the hard-margin SVM, which maximizes margin on linearly separable data.
For the case of linear classifiers, the normalized margin is defined as $\gamma_N = \frac{\gamma}{\norm{\alphahatbold}_2}$.
We can now state the ensuing classification test error (i.e.~0-1 test loss) as a function of the normalized margin.
Notation in the statement is adapted to be consistent with the notation in this paper --- for an elementary verification, see Appendix~\ref{appendix:spectral_margin}.

\begin{theorem}[Theorem 21,~\citep{bartlett2002rademacher}]
For a random test \newline point $(X,Y)$ drawn from the same distribution as the training data, the following holds with probability $(1 - \delta)$ over the training data $\{X_i,Y_i\}_{i=1}^n$:

\begin{equation}
\label{eq:spectral_gen_bound}
   \Pr[\mathsf{sgn}(\alphahatbold^T \avec(X)) \neq Y] \leq \frac{1}{n} \sum_{i=1}^n l_{\gamma}(\alphahatbold^\top \avec(X_i) \cdot Y_i) + \frac{4}{\gamma_{N}} \cdot  \frac{\norm{\Atrain}_{\mathsf{F}}}{n} +\left(\frac{8}{\gamma} + 1\right) \cdot  \sqrt{\frac{\ln(4/\delta)}{2n}}
\end{equation}
where the ramp loss function $l_{\gamma}$ is defined by
\begin{align*}
    l_{\gamma}(z) := \begin{cases}
    1 \text{ if } z \leq 0 \\
    1 - \frac{z}{\gamma} \text{ if } 0 < z \leq \gamma \\
    0 \text{ if } z > \gamma .
    \end{cases}
\end{align*}
\end{theorem}
When the training data are separable, we apply the above bound setting the first (average training loss) term to $0$, and only consider the the second term in the bound, i.e.~we ignore the high-probability term.
Equation~\eqref{eq:spectral_gen_bound} reminds us that there is a critical dependence on the intrinsic data dimension, captured by the term $\|\Atrain\|_{\mathsf{F}}$.
We will shortly see that this dependence is critical to track in the overparameterized regime.

\subsection{Can margin track performance of overparameterized models?}

We now investigate whether this generalization bound is effective in tracking the true test classification error for the hard-margin SVM in our setting for a number of choices of featurization.
Importantly, we consider the solution $\alphahatbold_{\mathsf{SVM}}$ only in sufficiently overparameterized settings under which all training points become support vectors with high probability; therefore, the un-normalized margin $\gamma = 1$ and the normalized margin of the SVM solution is exactly equal to $\gamma_{N} = \frac{1}{\norm{\alphahatbold}_2}$.

We study the evolution of margin, the ensuing upper bound in Equation~\eqref{eq:spectral_gen_bound}, and the true test classification error as we increase the level of overparameterization for two choices of featurizations:  isotropic Gaussian features (Definition~\ref{def:isotropic_covariance}) which generalize poorly according to Theorem~\ref{theorem:threeregimes} and weak features (Definition~\ref{def:weakfeatures}), which are known to exhibit the double-descent behavior.
For the case of isotropic features, we retain our 1-sparse assumption from Section \ref{sec:generalization}.
For the case of weak features, we consider $Y_i = \sgn(U_i)$ for $i \in \{1,\ldots,n\}$.

Figure~\ref{fig:margin_gaussian} plots the isotropic case, and Figure~\ref{fig:margin_wiggly} plots the weak features case. For both figures, we hold the number of training points, $n$ constant and vary the number of features, $d$, to tune the extent of overparameterization.
In both Figures~\ref{fig:margin_gaussian}(a) and~\ref{fig:margin_wiggly}(a), the normalized margin increases with increasing $d$, since the optimizer can use more features to meet the constraint in Equation~\eqref{eq:interpolatingsoln_binary}. The generalization bounds in Figure~\ref{fig:margin_gaussian}(b) and Figure~\ref{fig:margin_wiggly}(b) are consequently very similar as well.
However, while the test classification loss increases with $d$ for isotropic features, it decreases with $d$ for weak features.

\begin{figure}
  \centering
  \begin{tabular}{p{0.48\textwidth}p{0.48\textwidth}}
    \centering
    \includegraphics[width=0.5\textwidth]{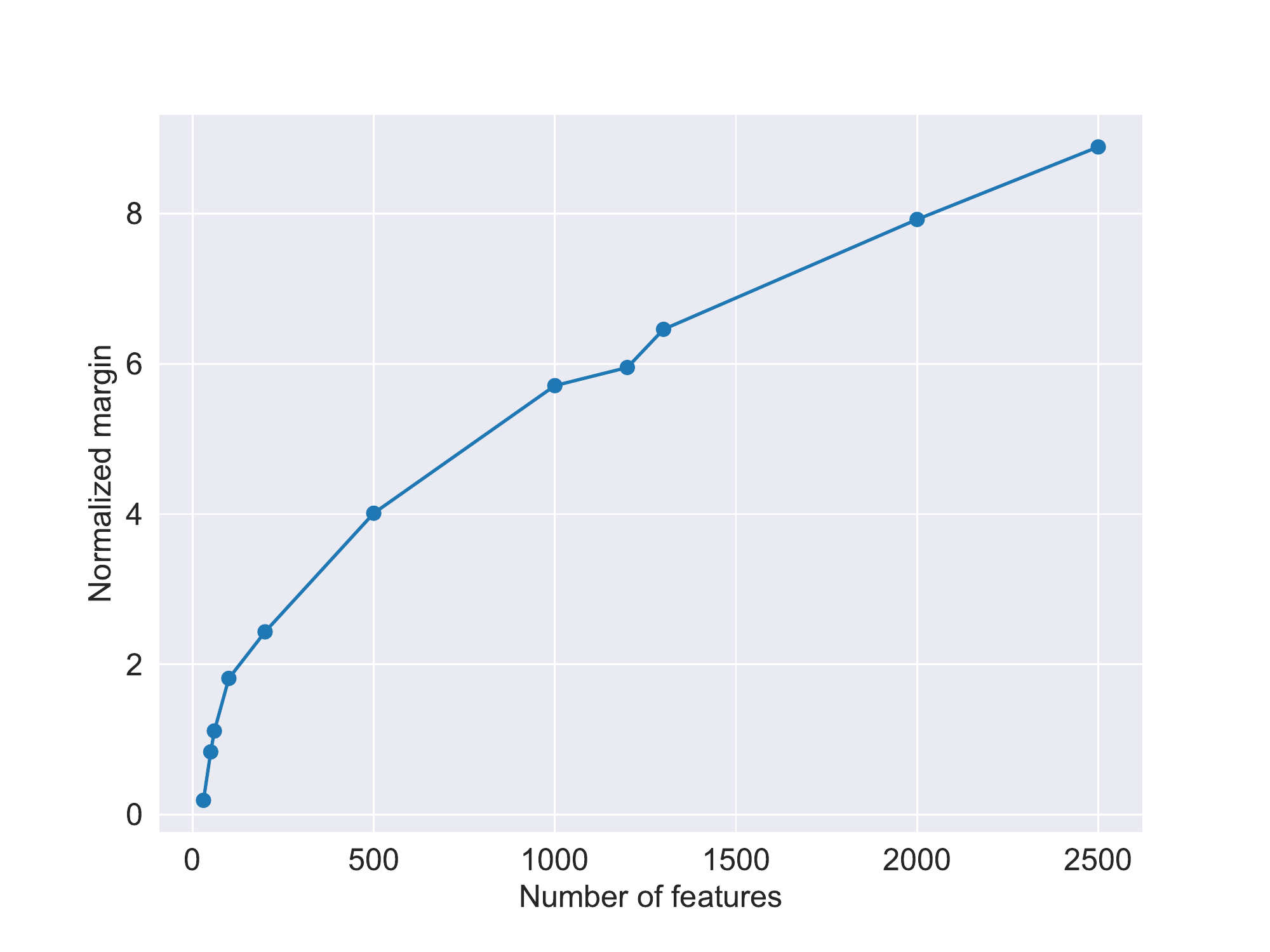}
    &
    \centering
    \includegraphics[width=0.5\textwidth]{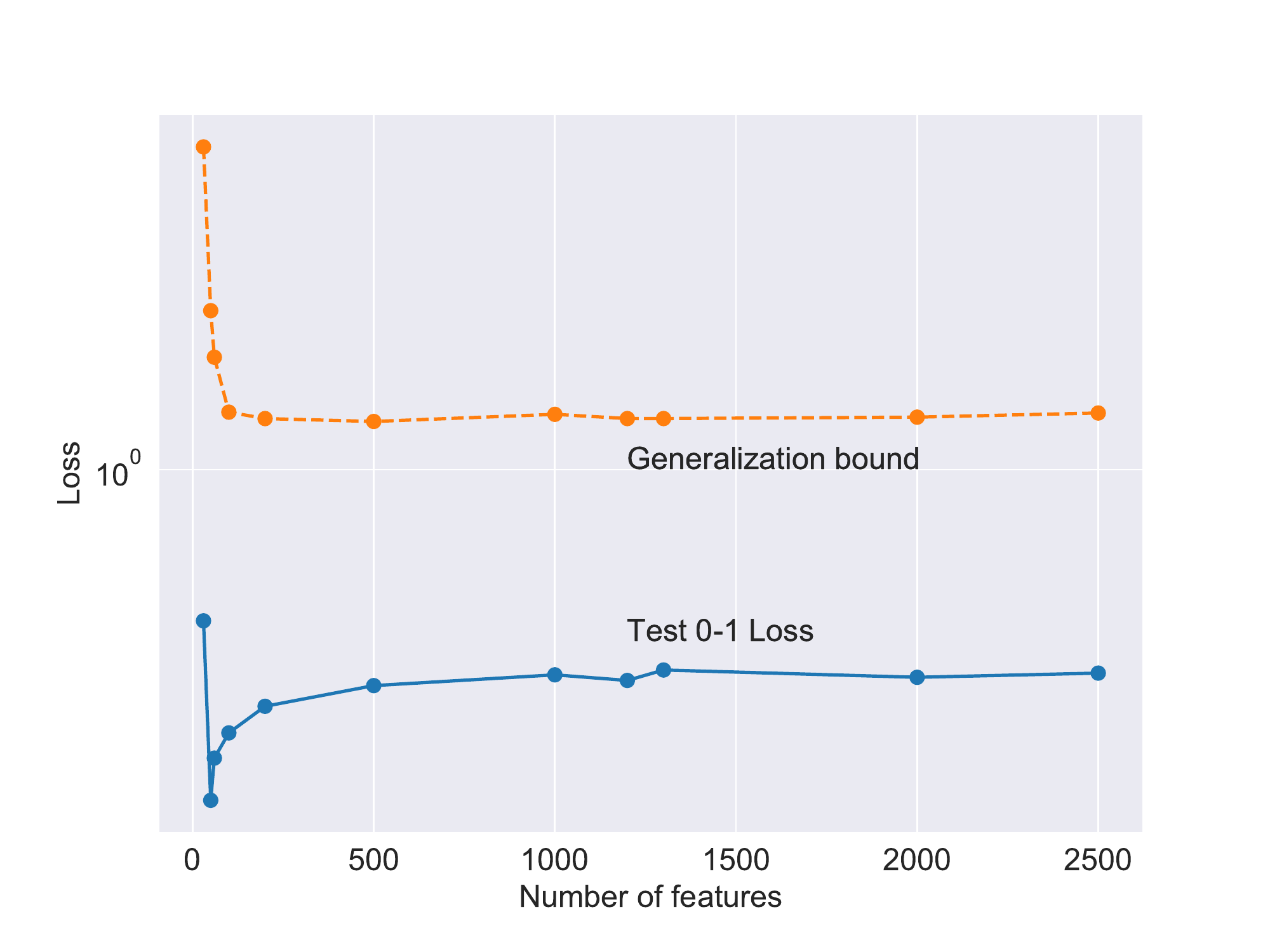}
    \tabularnewline
    \centering
    (a) Normalized margin.
    &
    \centering
    (b) Comparison of the generalization bound \eqref{eq:spectral_gen_bound} with true test classification loss.
  \end{tabular}
  \caption{Evolution of normalized margin, ensuing generalization bound and true classification test loss as a function of number of features $d$ for isotropic Gaussian features ($n = 32$ fixed). Observe that the terms $\norm{\Atrain}_{\mathsf{F}}$ and $\norm{\alphahatbold}_2$ cancel each other's effect on the bound, leading to a roughly constant bound. The true test error increases as $d$ is increased.}
  \label{fig:margin_gaussian}
\end{figure}

\begin{figure}
  \centering
  \begin{tabular}{p{0.48\textwidth}p{0.48\textwidth}}
    \centering
    \includegraphics[width=0.5\textwidth]{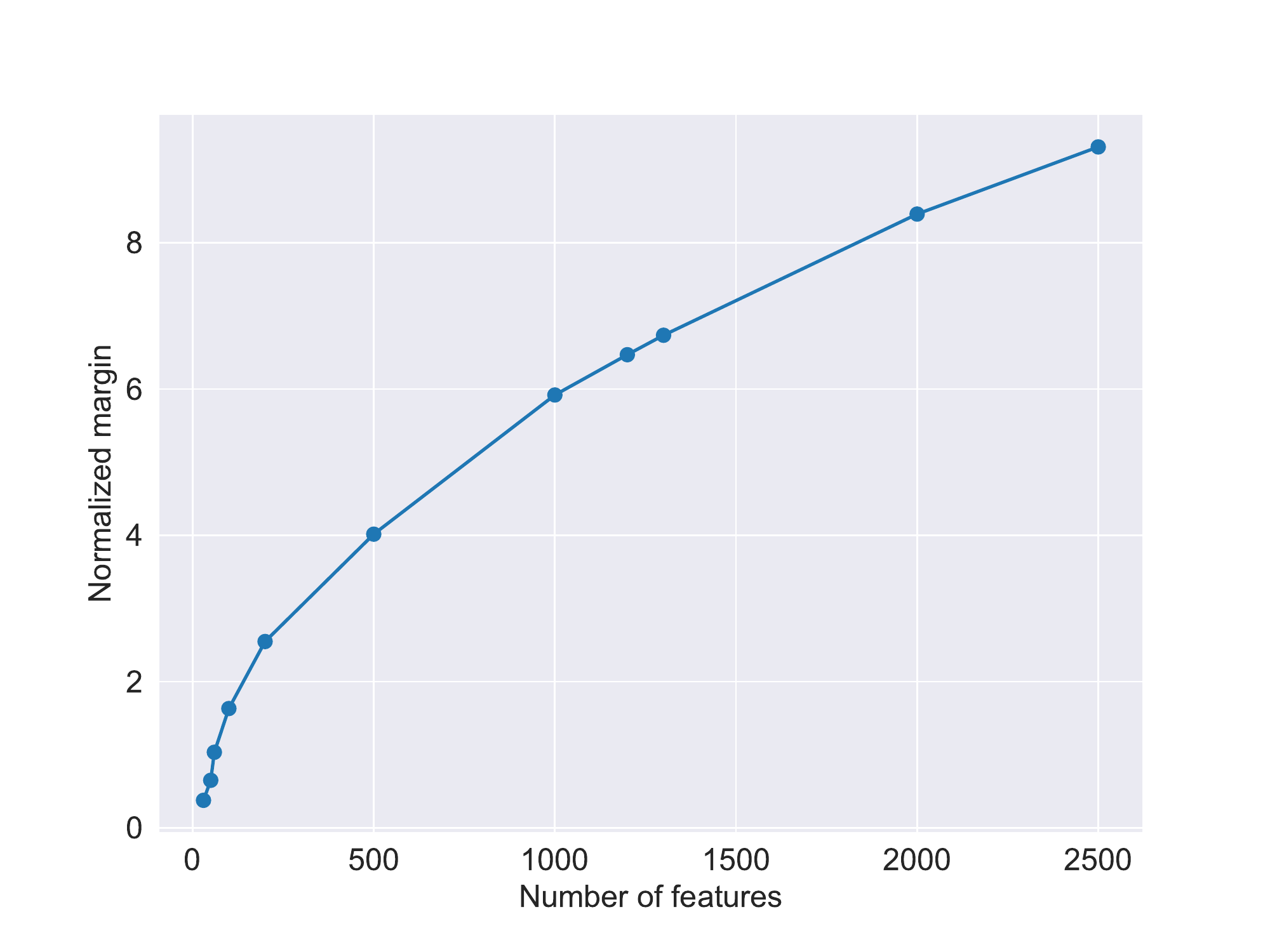}
    &
    \centering
    \includegraphics[width=0.5\textwidth]{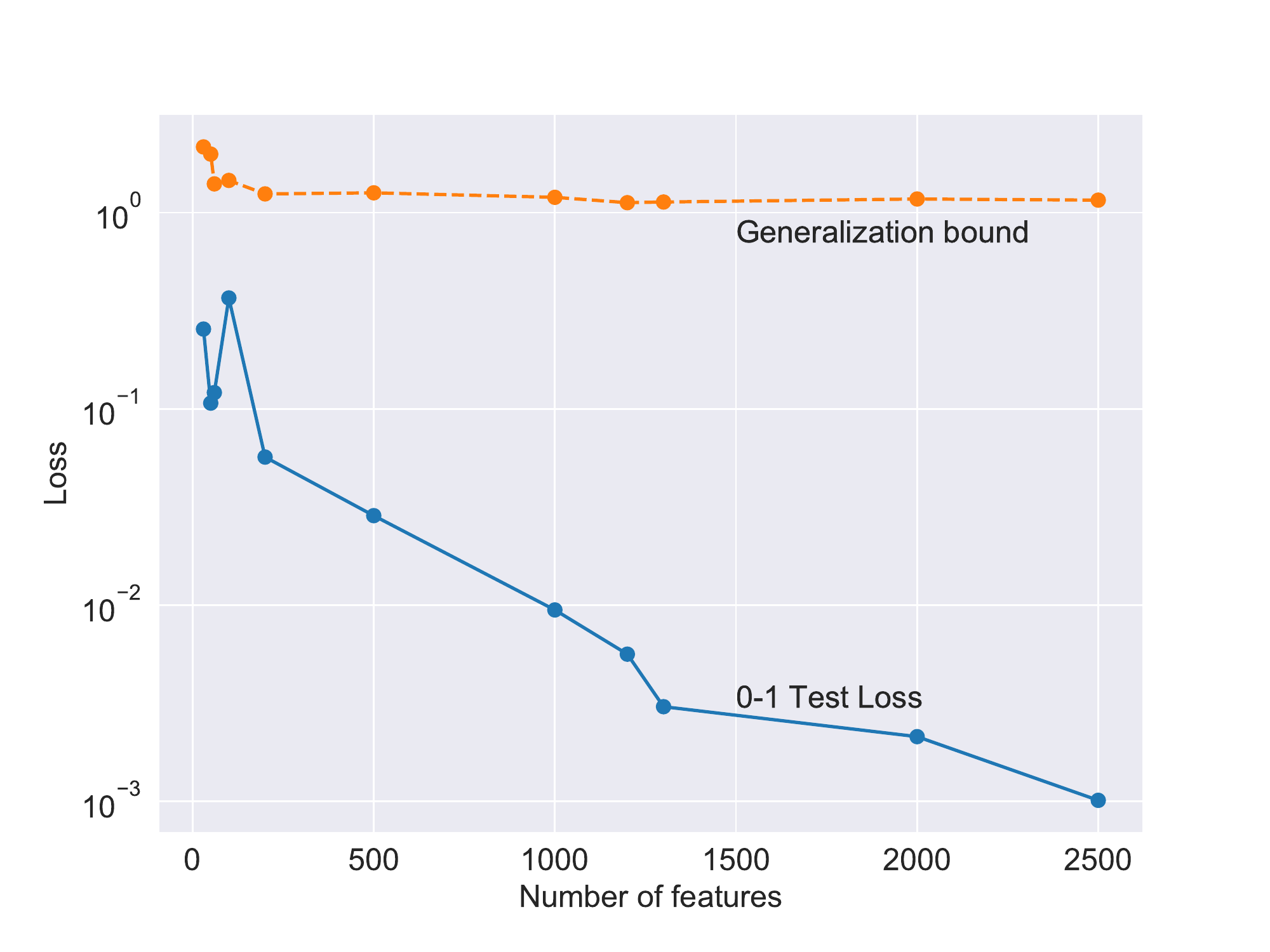}
    \tabularnewline
    \centering
    (a) Normalized margin.
    &
    (b) Comparison of the generalization bound \eqref{eq:spectral_gen_bound} with true test classification loss.
  \end{tabular}
  \caption{Evolution of normalized margin, ensuing generalization bound and true classification test loss as a function of number of features $d$ for weak features in Definition~\ref{def:weakfeatures} ($n = 32$ fixed, $\sigma = 0.1$). Observe that the terms $\norm{\Atrain}_{\mathsf{F}}$ and $\norm{\alphahatbold}_2$ cancel each other's effect on the bound, leading to a roughly constant bound. The true test error decreases as $d$ is increased.}
  \label{fig:margin_wiggly}
\end{figure}

Figures~\ref{fig:margin_gaussian} and~\ref{fig:margin_wiggly} together show that the relationship between margin and generalization is more complex than typically assumed in highly overparameterized regimes.
We highlight a few observations:
\begin{enumerate}
  \item In both featurizations, the generalization bounds are always greater than $1$, and hence, tautological. A similar empirical observation was made by \citet{dziugaite2017computing} for a two-layer neural network.
    The $\norm{\Atrain}_{\mathsf{F}}$ term, which represents the scale of the data, effectively cancels out any beneficial effect of increasing normalized margin\footnote{In fact, for the case of minimum-$\ell_2$-norm interpolation and isotropic features, this can be verified quantitatively, as we know that $||\alphahatbold||_2 \sim \sqrt{\frac{n}{d}}$ and $||\Atrain||_{\mathsf{F}} \sim \sqrt{nd}$ with high probability.
    }.
Intuitively, it is clear that \textit{feature-space} margin-based bounds will have to scale with the intrinsic input dimension, which itself is overparameterized for Gaussian featurization.
    \item Whether margin is \textit{qualitatively} predictive of generalization is also unclear, as evidenced by the contrasting examples of weak features and isotropy.
    Under both featurizations, the normalized margin increases with increased overparameterization; but the actual test error behaves very differently (decreasing for weak features, but increasing for isotropy).
\end{enumerate}
Thus, we see that margin-based bounds are not predictive of the behavior of overparameterized models in our setting.
It is still possible that an appropriate sense of large margin implies good generalization in certain cases.
In particular, for linear models, maximizing the margin is equivalent to minimizing the norm --- which, as we have seen, has important generalization properties.
However, evidence of this needs to come from first-principles analysis, not from the existing bounds.

The recent work of ~\citet{negrea2019defense} suggests a way forward in the
interpolating regime via the introduction of appropriate surrogates,
that implicitly capture the good generalization properties vis-a-vis
the underlying patterns and the learning algorithm used.
It would be interesting to see how these ideas could be unified with the
survival/contamination perspective developed here across all three
regimes identified in Theorem~\ref{theorem:threeregimes}.

\acks{We would like to acknowledge the Simons Institute Summer 2019 program on ``Foundations of Deep Learning", which facilitated the initial collaboration that led to this work.
More broadly, we thank the participants of this program for many stimulating research discussions that inspired this collaboration --- especially Suriya Gunasekar and Matus Telgarsky.

MB acknowledges federal support from NSF and a Google Research Award.
DH acknowledges support from NSF grant CCF-1740833 and a Sloan Research Fellowship.
AS acknowledges the support of the ML4Wireless center member companies and NSF grants AST-144078 and ECCS-1343398.}

\newpage

\appendix
\appendix
\section{Fourier features on regularly spaced training data: An ``ultra-toy" model}
\label{appendix:ultratoy}
In~\citet{muthukumar2020harmless}, the case of Fourier features on regularly spaced training data was introduced and studied as an ``ultra-toy", or caricature model to highlight the consequences of overparameterization in linear regression on noisy data.
The ramifications of $\ell_2$-minimization are clearly illustrated through this model, as an explicit connection can be made to the classical phenomenon of \textit{aliasing} that is involved to understand the under-sampling of continuous time signals.
Using this signal processing perspective, survival and contamination are natural quantities of interest, as illustrated in Figure~\ref{fig:survivalcontamination}(a) for the $1$-sparse case.
In Figure~\ref{fig:survivalcontamination}(b), we see how these concepts would \textit{qualitatively} manifest more generally when the underlying signal is hard-sparse.

\begin{figure}
  \centering
  \begin{tabular}{p{0.48\textwidth}p{0.48\textwidth}}
    \centering
    \includegraphics[width=0.4\textwidth]{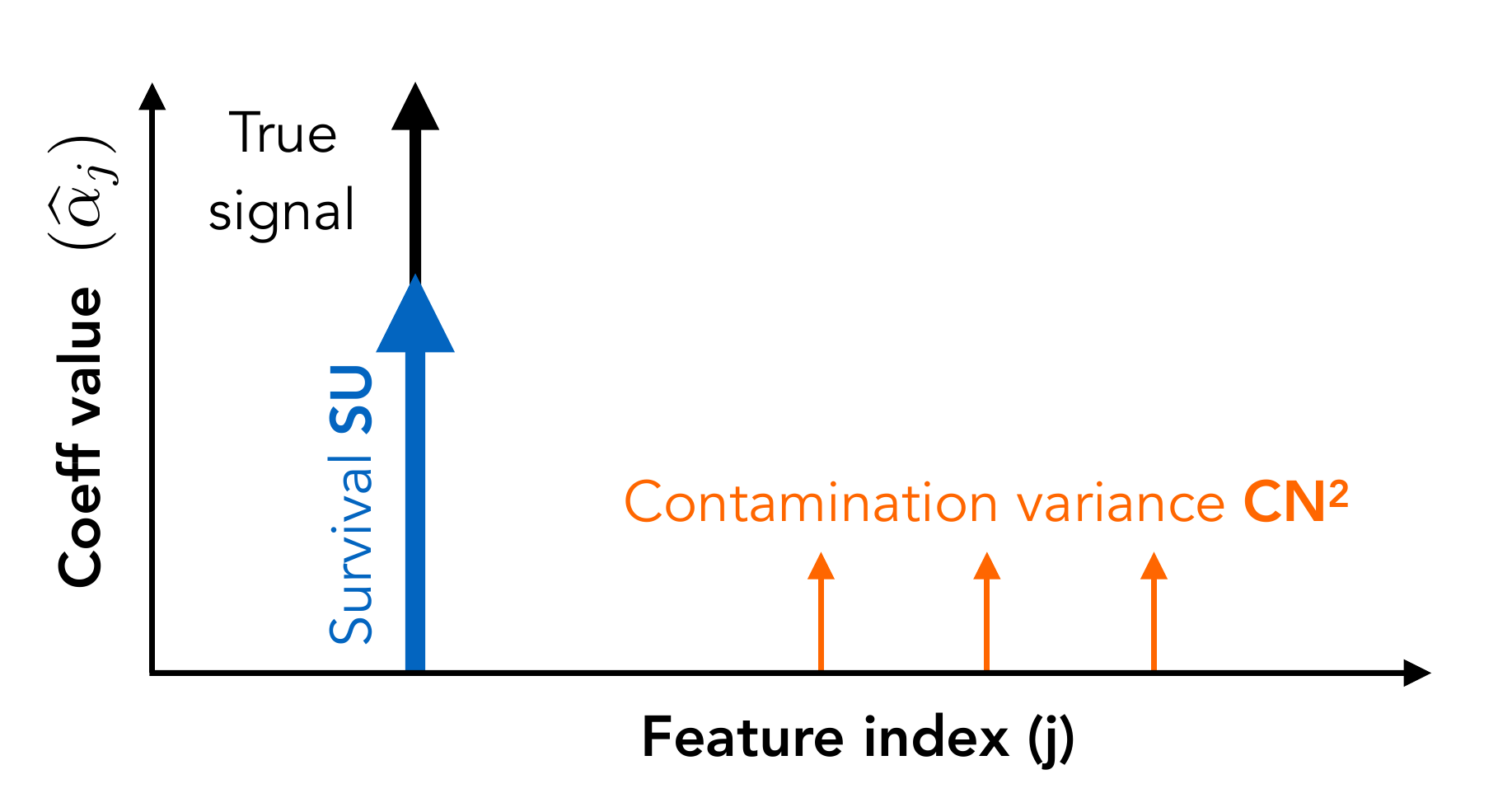}
    &
    \centering
    \includegraphics[width=0.4\textwidth]{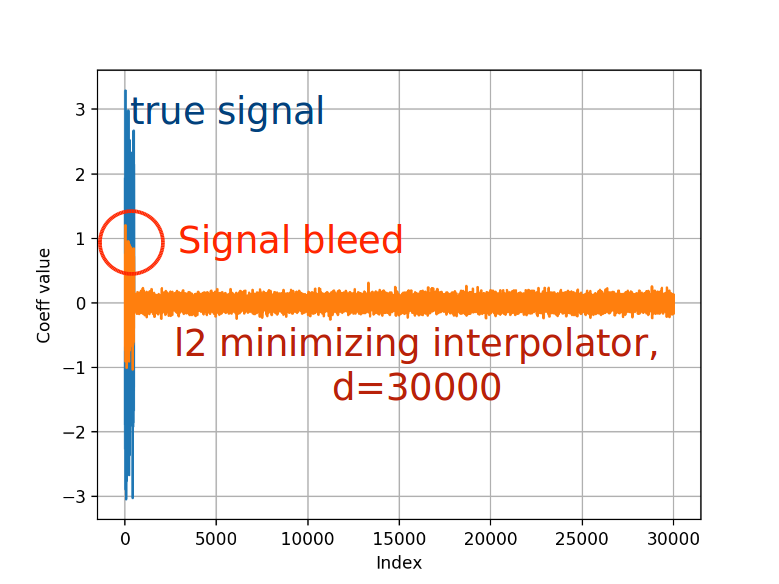}
    \tabularnewline
    \centering
    (a) Illustration for $1$-sparse signal (Assumption~\ref{assumptionsimple}).
    &
    \centering
    (b) Actual signal ``bleed" and ``contamination" for $30000$ isotropic Gaussian features, and $1000$ samples of $500$-sparse signal.
  \end{tabular}
  \caption{Illustrations of survival and contamination factors that affect both classification and regression test error. Here, signal ``bleed'' refers to the qualitative phenomenon where magnitude of recovered components of the true signal is much smaller than the magnitude of true signal components and is an extension of the survival concept to the hard-sparse setting.}
  \label{fig:survivalcontamination}
\end{figure}

As we illustrate in this section, appropriate weightings of these features under this ``ultra-toy" model also helped us conjecture all of the main results of this paper.

The Fourier ensemble is defined below.

\begin{definition}[Weighted Fourier features on regularly spaced data]\label{eg:fourier}
We consider $n$ (odd) regularly spaced training points from $(-\pi,+\pi)$ --- specifically the sequence $(-\pi+\frac{\pi}{n}, -\pi+\frac{3\pi}{n}, \ldots, -\frac{2\pi}{n}, 0, +\frac{2 \pi}{n}, \ldots, +\pi - \frac{ \pi}{n})$, a test distribution of $X$ drawn uniformly at random from $(-\pi, +\pi)$, and the $d$ (odd multiple of $n$) features chosen to be the standard real orthonormal Fourier features: $$\frac{1}{\sqrt{2\pi}}, \frac{1}{\sqrt{\pi}}\sin(x), \frac{1}{\sqrt{\pi}}\cos(x),  \ldots, \frac{1}{\sqrt{\pi}}\sin\left(\frac{d-1}{2} x\right), \frac{1}{\sqrt{\pi}}\cos\left(\frac{d-1}{2} x\right).$$
For doing interpolative inference using a weighted norm minimization,
we define the weights\footnote{If $\alpha_j$ represents the learned
  coefficient on the cosine at frequency $f$, and $\beta_j$ the
  learned coefficient on the sine at frequency $f$, the minimization
  is of $\sum_{j} \frac{\alpha_j^2 + \beta_j^2}{\lambda_j}$. A higher
  $\lambda_j$ means that frequency is favored.} corresponding to sines and cosines of frequency $j$ by
$\{\lambda_j\}_{j=0}^{\frac{(d-1)}{2}}$.
Following the convention of the rest of the paper, we take the weights $\{\lambda_j\}$ to be a decreasing, strictly positive sequence.
\end{definition}
\textit{Exact aliases} are defined as distinct features that agree with each other (possibly up to a constant multiple) on all the sampled points.
The Fourier featurization allows exact aliases to exist. There are three different groups of these exact aliases:
\begin{itemize}
    \item The initial constant feature is essentially aliased by the cosines at every multiple\footnote{For ease of exposition, the minor issue of the constant feature having a slightly different scaling vis-a-vis its aliases is going to be ignored in this treatment, but this is simply a matter of keeping track of notation. Alternatively, we could eliminate this by using complex Fourier features. We will finesse this issue here by simply not allowing the true signal to have a constant term in it.} of $n$.
    \item Each  cosine feature in the first $n$ features (namely corresponding to a frequency $j \in \{1,2,\ldots, \frac{n-1}{2}\}$) picks up $(\frac{d}{n} -1)$ cosine aliases with frequencies $(n-j), (n+j), (2n-j), (2n+j), \ldots$. This is because cosine is an even function and the training samples are symmetrically distributed about 0.

    \item Similarly, each sine feature in the first $n$ features (corresponding to a frequency $j \in \{1,2,\ldots, \frac{n-1}{2}\}$) picks up $(\frac{d}{n} -1)$ sine aliases with frequencies $(n-j), (n+j), (2n-j), (2n+j), \ldots$. However, because sine is an odd function, these aliases have their signs alternating with the $(kn-j)$ ones being multiplied by $(-1)$ and the $(kn+j)$ ones being exact aliases.
\end{itemize}

\subsection{Empirical evidence that all training points become support vectors}\label{appendix:fourier}

Consider, first, the case of $d = n$ and $\lambda_j = 1$ for all $j$.
The orthonormality of the Fourier features above essentially (this argument would be exact if we used the complex Fourier features) makes the first $n$ columns of the training data feature matrix look like a rotation of a scaled version of the identity.
In such a situation, minimizing the $\ell_2$-norm of the learned parameters in a hard-margin SVM essentially forces\footnote{Why? Suppose the scalar prediction on one of the training points was larger than $+1$. If so, we could reduce the norm of the learned parameters without impacting any other training point's prediction by making this point have a scalar prediction of $+1$. Norm-minimization and a full complement of orthonormal vectors in the training matrix forces every point to be a support vector. This can alternatively be viewed as a consequence of extreme symmetry --- under a Fourier featurization and the 2-norm, no training point is special if they are regularly spaced.}  every point to be a support vector. Adding more aliases (in a balanced way) for all features is not going to change this.
This leads to every point becoming a support vector in the isotropic case for any $d$ that is a multiple of $n$.

The case of non-isotropic overparameterized models is more complex.
Here, we describe the experiment underlying Figure~\ref{fig:svmlsmaintext} in the main text that first showed that all training points became support vectors in high dimensions (Theorem~\ref{thm:everythingsupport}).
We conducted this experiment for regularly spaced training data and the Fourier featurization as defined above,
with polynomial decay in the weights used, i.e.~$\lambda_k = \frac{1}{k^m}$ for $m \geq 0$, and (since the training data $\{X_i\}_{i=1}^n$ is $1$-dimensional), visualized both the minimum-$\ell_2$-norm interpolation and the SVM.
Figure~\ref{fig:svmlsmaintext} in the main text shows a remarkable equivalence between the two solutions for various values of $m$, and provides initial empirical evidence for this phenomenon (that we theoretically established for Gaussian featurization in Theorem~\ref{thm:everythingsupport}).

\subsection{Regression vs Classification} \label{appendix:fourierproof}

To see the counterpart of Theorem~\ref{theorem:threeregimes}, which compares classification and regression test error of interpolating solutions, we consider the underlying true function to be $\cos(x)$.
At training time, we get actual real-valued outputs $z_j = \cos(x_j)$ corresponding to the $n$ regularly spaced points $\{x_j\}$.
The minimum-$\ell_2$-norm interpolation of real-valued output leads to the following coefficients on the $d$ Fourier features:
\begin{align}
  \alphahatbold = \argmin_{\boldsymbol{\alpha} \big{|} \Atrain \boldsymbol{\alpha} = \mathbf{z}_{\text{train}}} \left(\frac{1}{\lambda_H} \sum_{j=0}^{s-1} \alpha_j^2\right) + \sum_{j=s}^{d-1} \alpha_j^2 \label{eq:optimizationfourier}
\end{align}
Because of the presence of \textit{exact aliases}, the training data matrix $\Atrain$ consists of $n$ distinct columns that repeat again and again.
In keeping with the  bi-level covariance model in Definition~\ref{def:bilevel_covariance}, we scale the parameters $(s, \lambda_H, d)$ with $n$ in a coordinated way.
Recall that the number of prioritized features is given by $s := n^r$ for $r \in [0,1)$, and the number of features $d = n + n^p$ for $p > 1$. (We added an extra term of $n$ to make it easier to count the aliases. This has no asymptotic effect when $p>1$ and $n \rightarrow \infty$.) The $\lambda_H$ represents how much we favor the special features and in keeping with the scaling in Definition~\ref{def:bilevel_covariance}, we set $\lambda_H = n^{p-r-q}$ for some $q \in [0,p-r]$.

Because of the known orthogonality of the sine and cosine features on $n$ regularly spaced points, the first $n$ columns of $\Atrain$ are orthogonal.
This means that the solution $\alphahatbold$ will only have nonzero entries in the positions that correspond to the $\frac{d}{n} = 1+ n^{p-1}$ different columns of $\Atrain$ that are copies of the column corresponding to the feature $\cos(x)$. Since $s < n$, exactly one of these will be favored and so the optimization problem in Equation~\eqref{eq:optimizationfourier} turns into the much simpler problem:
\begin{align}
  \min_{a,b \big{|} a + n^{p-1} b = 1} \frac{a^2}{n^{p-r-q}} + n^{p-1} b^2 \label{eq:simplifiedoptimization}
\end{align}
where $a$ represents the recovered coefficient corresponding to the true underlying feature $\cos(x)$ and $b$ represents the coefficients on all of its exact aliases.

\begin{figure}
    \centering
    \includegraphics[width=0.7\textwidth]{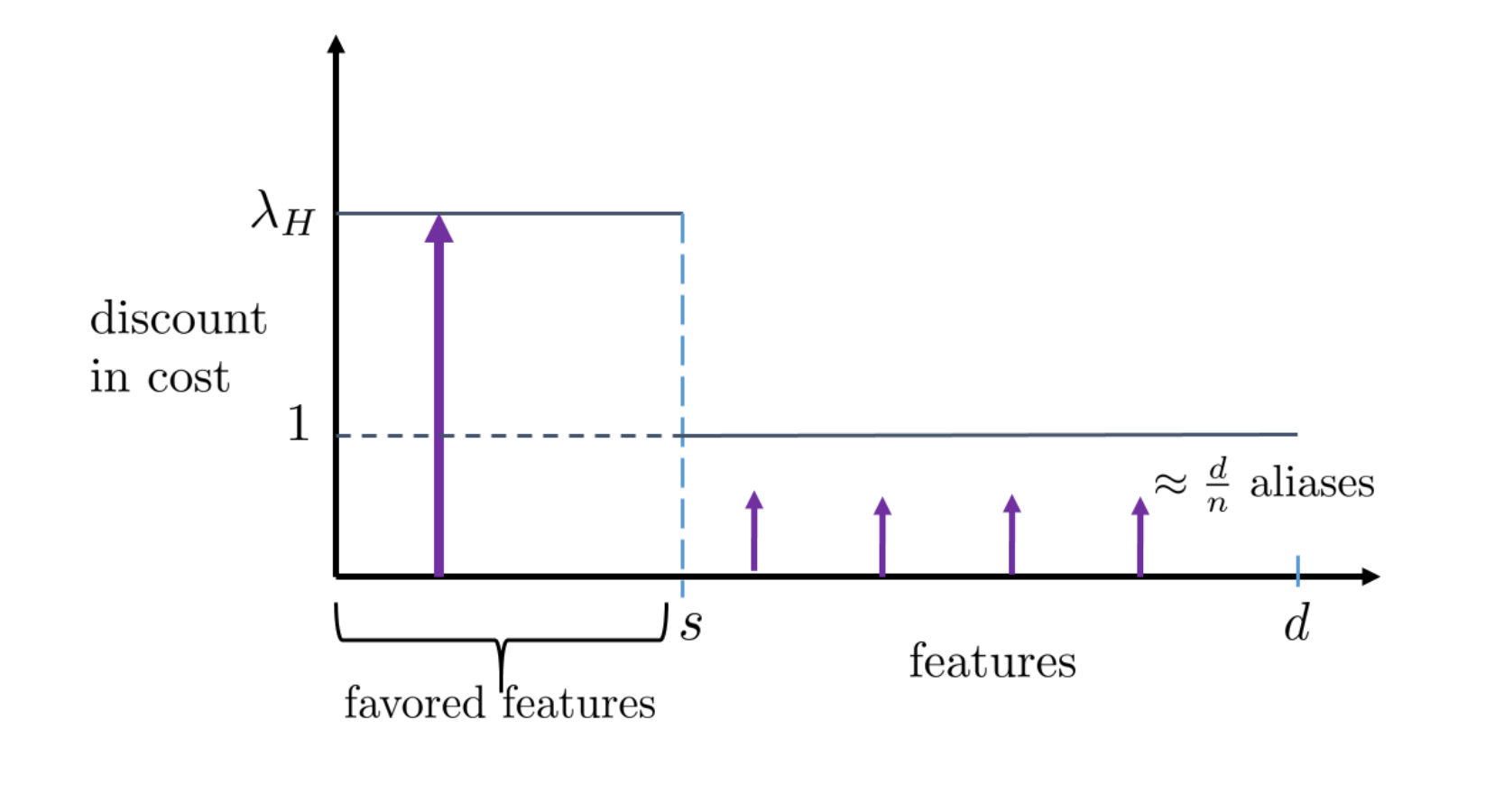}
    \caption{An illustration of the bi-level model for the Fourier features.}
    \label{fig:weight_app}
\end{figure}
An elementary calculus calculation\footnote{See the appendix of \citet{muthukumar2020harmless} for more discussion of this calculation and its connection to matched filtering in signal processing.} shows that Equation~\eqref{eq:simplifiedoptimization} is solved by:
\begin{subequations}
\begin{align}
  a &= \frac{\lambda_H}{\lambda_H + \left(\frac{d}{n} - 1\right)} = \frac{1}{1 + n^{q - (1-r)}}
  \label{eq:survivalaexpression} \text{ and } \\
  b &= \frac{1}{\lambda_H + \left(\frac{d}{n} - 1\right)} = \frac{1}{n^{p-r-q} + n^{p-1}} .
  \label{eq:individualaliasamount}
\end{align}
\end{subequations}
The reader can verify that $a$ represents the survival of the true signal.
For large enough $n$, this is approximated\footnote{In the style of the Bode Plot of a one-pole low pass filter.} by
\begin{align}
  a \approx \begin{cases}
    1 & \text{if}\ q < 1-r \\
    n^{-(q - (1-r))} & \text{if}\ q > 1-r
  \end{cases} . \label{eq:survivalapproximation}
\end{align}
This approximation is guaranteed to be good to within a factor of 2 everywhere and is usually much better.
Notice that Equation~\eqref{eq:survivalapproximation} is the Fourier-feature counterpart of the upper and lower bounds on survival in Lemmas~\ref{lemma:survivalspikedl2} (binary labels) and \ref{lemma:survivalspikedl2real} (real-valued output).
Now, taking $n\rightarrow \infty$, we get
\begin{align}
  a_\infty = \begin{cases}
    1 & \text{if}\ q < (1-r) \\
    0 & \text{if}\ q > (1-r)
  \end{cases} \label{eq:survivallimit}
\end{align}
which shows that the signal only fully survives if $q < (1-r)$.

Let us now measure the \textit{contaminating} effect of falsely discovered features.
Following Equation~\eqref{eq:falsediscovery}, we denote $B(\mathbf{X})$ as the random variable that represents the contribution of all of the aliases to the predictions. Each of the Fourier features of non-zero frequency is zero-mean and has variance $1$.
From the orthonormality (in expectation over test data) of the aliases, we get
\begin{align}
  \mbox{Var}[B(\mathbf{X})]
  &= n^{p-1} b^2 \nonumber \\
  &= \left(\frac{1 }{n^{\frac{p}{2} + \frac{1}{2} -r-q} + n^{\frac{(p-1)}{2}}}\right)^2 , \label{eq:contamvariance}
\end{align}
where in the last step, we substituted Equation~\eqref{eq:individualaliasamount}.
Notice that $\frac{(p-1)}{2} > \frac{p}{2} + \frac{1}{2} -r-q $ whenever $q > (1-r)$, and so asymptotically we get
\begin{align}
\contaminationb = \sigma_{CN} \approx \begin{cases}
 n^{-\left(\frac{p+1}{2} -(q+r)\right)  } & \text{if}\ q < (1-r) \\
 n^{-\frac{(p-1)}{2}}  & \text{if}\ q > (1-r)
\end{cases}
\label{eq:contaminationapproximation}
\end{align}
This approximation is guaranteed to be good to within a factor of 2 everywhere and is usually much better.
Notice that this expression is the Fourier-feature counterpart of the lower bound on contamination established for Gaussian features in Lemma~\ref{lemma:contaminationlowerboundspikedl2}.

Thus, provided that $q < (1-r)$, the expression in Equation~\eqref{eq:contaminationapproximation} always decays to zero as $n \rightarrow \infty$, regardless of which case we are in.
The combination of Equations~\eqref{eq:contaminationapproximation} and~\eqref{eq:survivallimit} tells us that regression in this problem can work\footnote{In fact, the argument also works for \textit{noisy training} data, i.e.~$\mathbf{Z}_i = \sin (x_i) + \mathbf{W}_i$ where the noise $\mathbf{W}_i$ is iid and has zero mean, bounded support, and finite variance $\sigma^2$. The formal argument is in~\citet{muthukumar2020harmless}, but is summarized here for the ease of the reader. From the central limit theorem and the theory of wide-sense-stationary random variables, in the limit, the representation of the noise part will look marginally Gaussian in the basis of the first $n$ columns of $\Atrain$, where each of them will be $~ \mathcal{N}\left(0,\frac{\sigma^2}{n}\right)$. The first $s$ of these will survive and thereby contribute a variance of approximately $\sigma^2 n^{r-1} a^2$ to test points, while the other $(n-s)$ of these will be absorbed across all the aliases and thereby contribute a contamination variance of $\sigma^2 n^{-p}$. The total contribution will be dominated by the $\sigma^2 n^{-(1-r)}a^2$ term. If $q < (1-r)$ and $a \approx 1$, this term vanishes as $n \rightarrow \infty$ and so additive noise does not contribute to regression error unless the noise variance $\sigma^2$ itself grows with $n$ (at a rate faster than $n^{1-r}$).}
to get mean-square-error approaching zero as long as $q <(1-r)$. On the other hand, when $q>(1-r)$, signal does not asymptotically survive and regression MSE approaches the null risk.

\begin{figure}
    \centering
    \includegraphics[width = 0.7\textwidth]{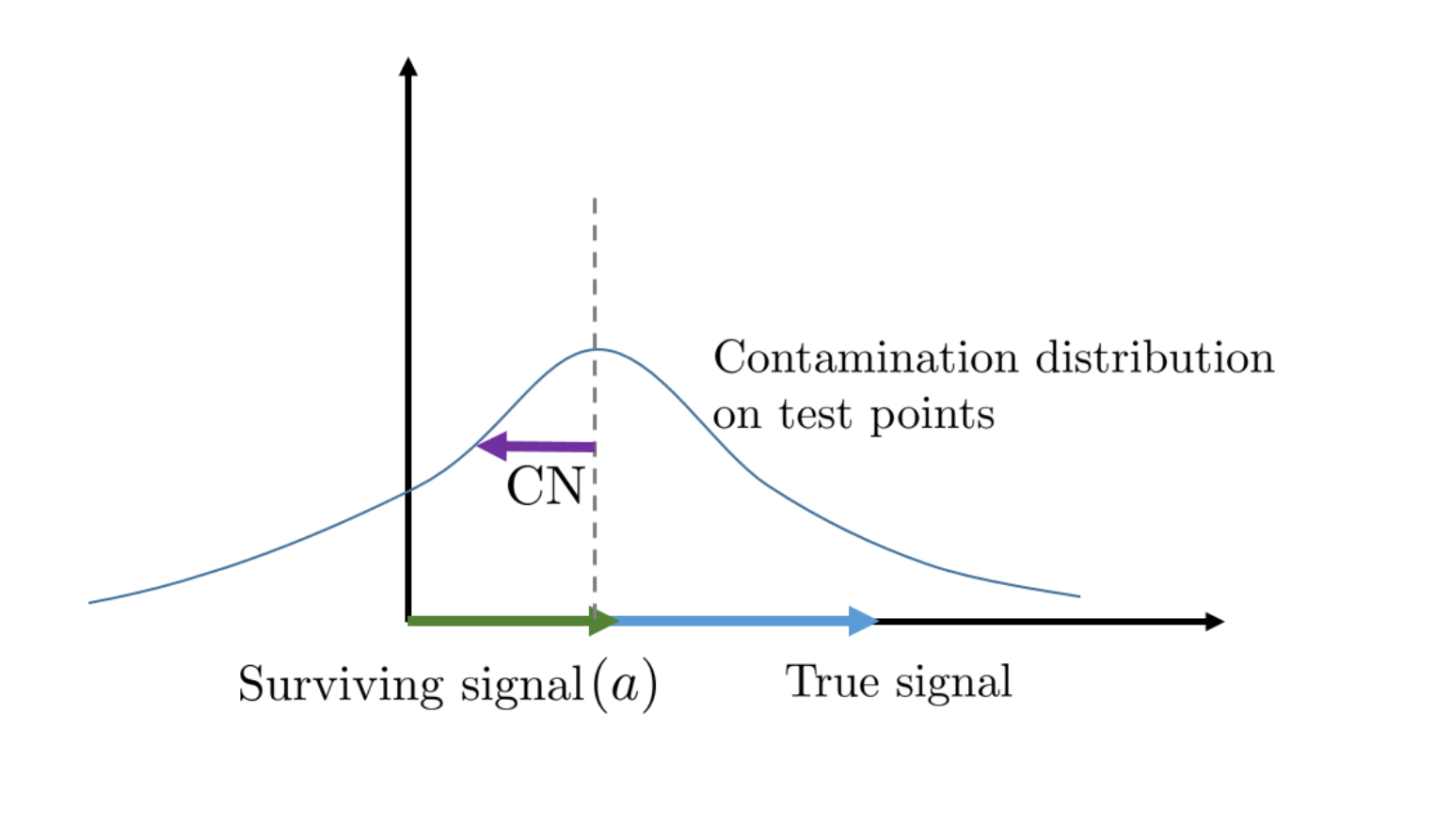}
    \caption{Illustration of how contamination can flip the sign of the prediction at a test point.
The survival \textit{relative} to the standard deviation of the contamination, $\contaminationb$, is what matters --- if the latter is much smaller than the former, then the probability of classification error is low.}
    \label{fig:su_vs_c_app}
\end{figure}

\subsubsection{Implications for classification: existence of the separating regime}

For classification, we only care about predicting $\sgn(\cos(\mathbf{X}))$ correctly with high probability when $\mathbf{X} \sim \text{Unif}\left[-\pi,\pi\right]$.
Clearly, classification also works under the conditions for which regression works (i.e.~$q < (1 - r)$), but, as we showed in Theorem~\ref{theorem:threeregimes}, can work even in the absence of these conditions.
Recall that when $q > (1-r)$, the survival factor $a \rightarrow 0$ as $n \rightarrow \infty$.
However, if the contamination is small enough, i.e.~$\sigma_{CN} \ll a$, the probability of classification error is extremely low, as illustrated in Figure~\ref{fig:su_vs_c_app}. We observe from Equations~\eqref{eq:contaminationapproximation} and~\eqref{eq:survivalapproximation} that $\sigma_{CN} \ll a$ if $q < (1-r) + \frac{(p-1)}{2}$. When that happens, classification will asymptotically work.

\begin{figure}[!htb]
  \centering
  \begin{tabular}{p{0.48\textwidth}p{0.48\textwidth}}
    \centering
    \includegraphics[width=0.64\textwidth]{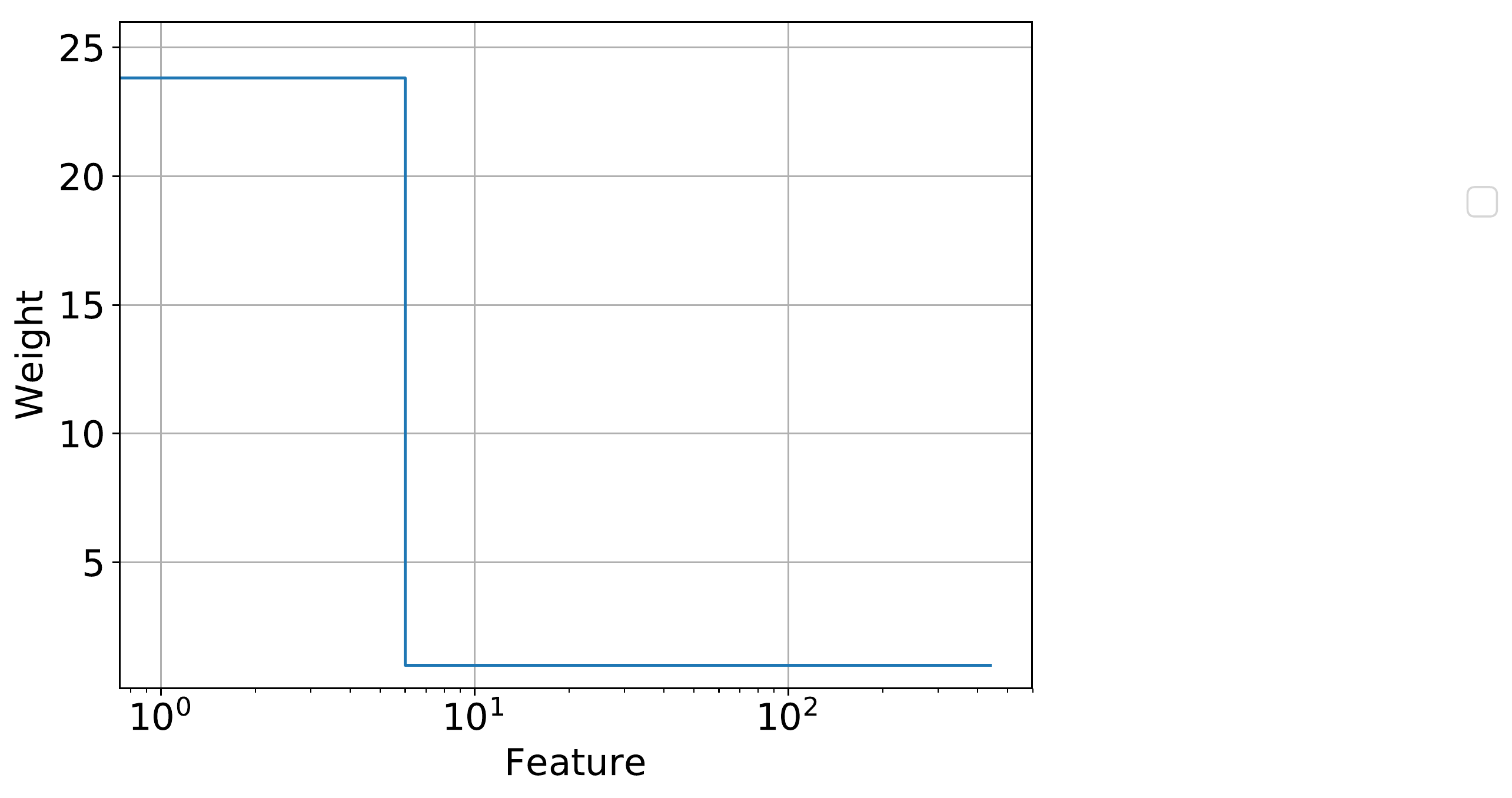}
    &
    \centering
    \includegraphics[width=0.48\textwidth]{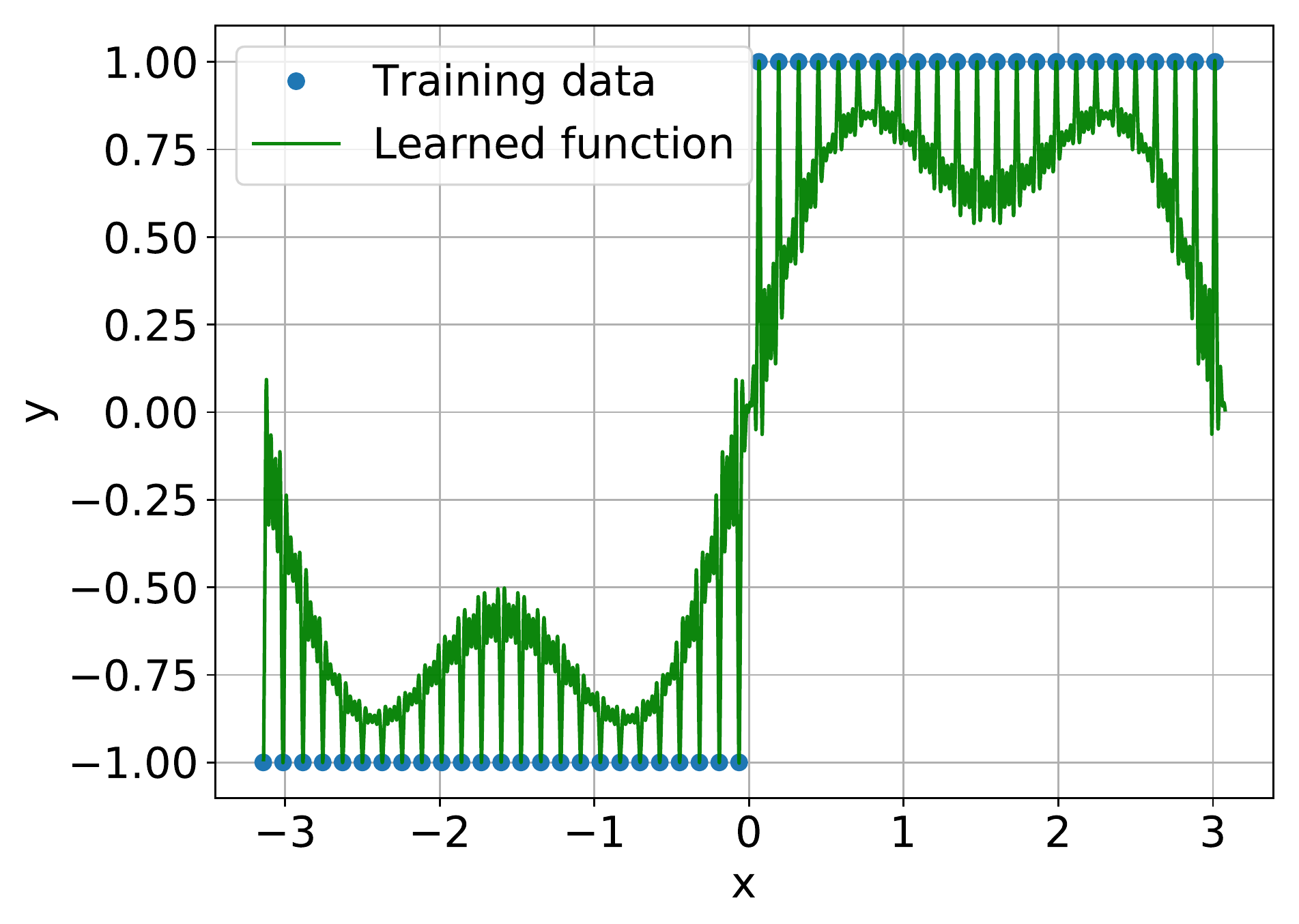}
    \tabularnewline
    \centering
    (a) $\lambda_H = 23.81$
    &
    \centering
    (b) $\lambda_H = 23.81$
    \tabularnewline
    \centering
    \includegraphics[width=0.64\textwidth]{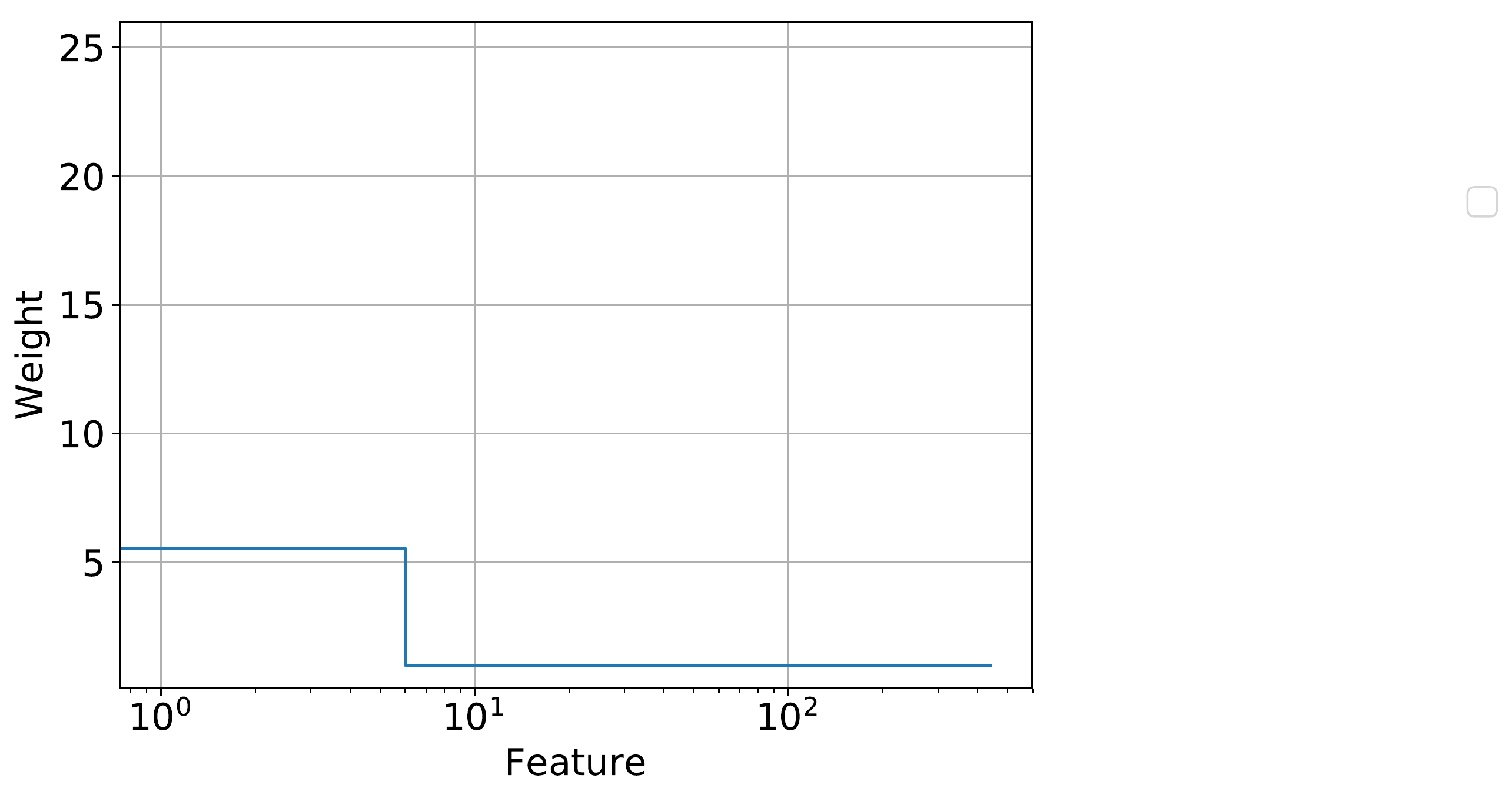}
    &
    \centering
    \includegraphics[width=0.48\textwidth]{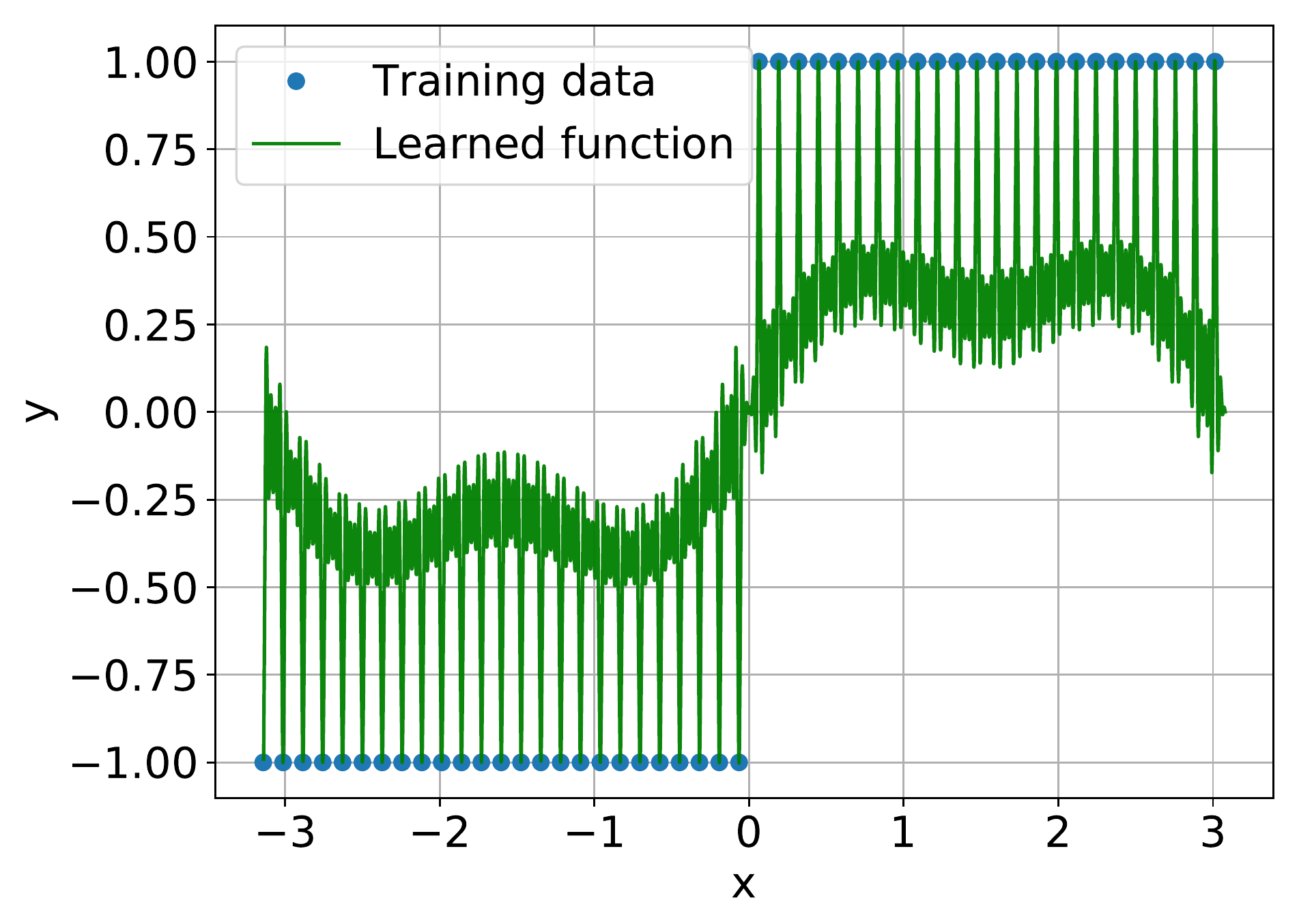}
    \tabularnewline
    \centering
    (c) $\lambda_H = 5.53$
    &
    \centering
    (d) $\lambda_H = 5.53$
    \tabularnewline
    \centering
    \includegraphics[width=0.64\textwidth]{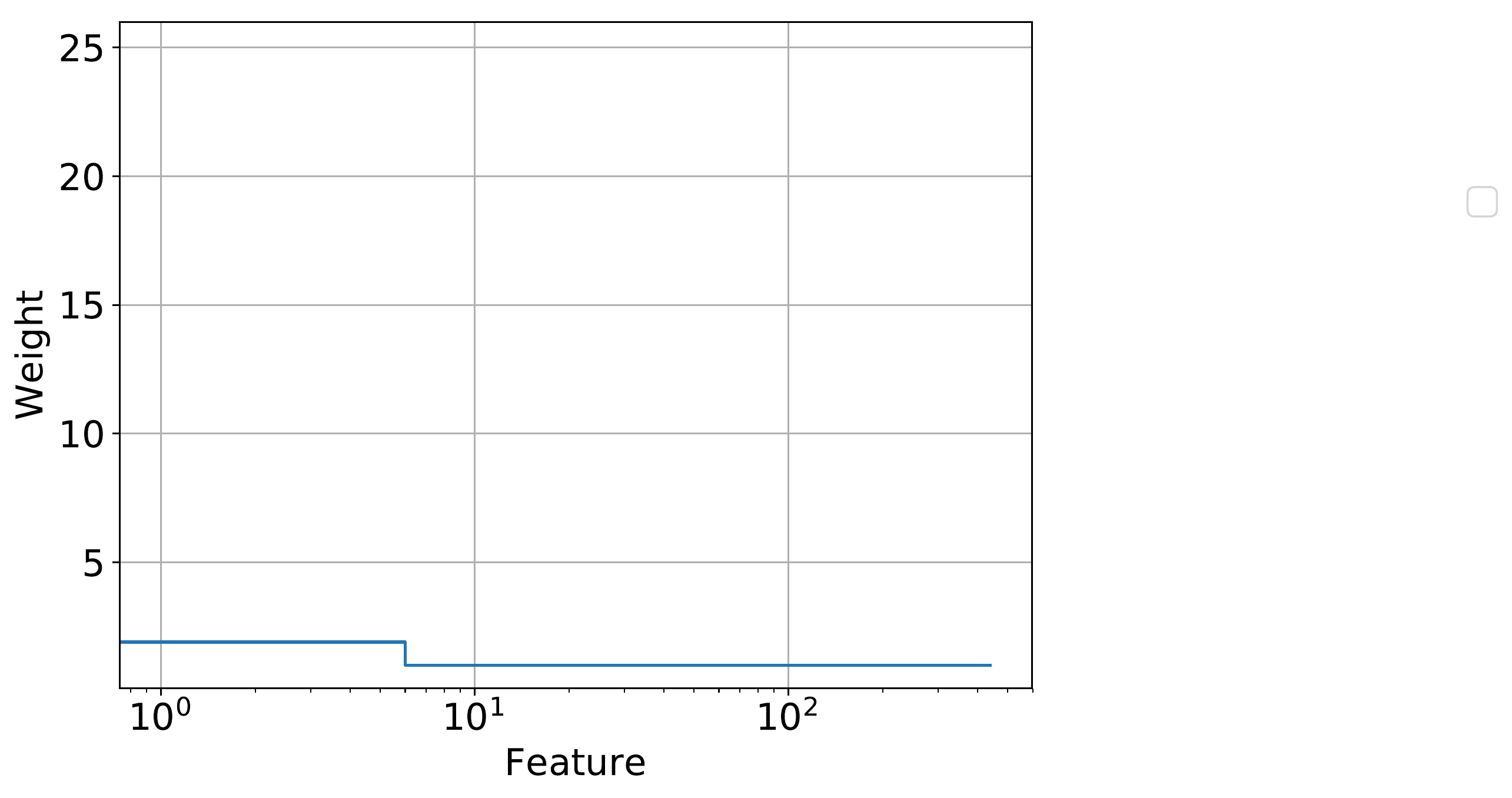}
    &
    \centering
    \includegraphics[width=0.48\textwidth]{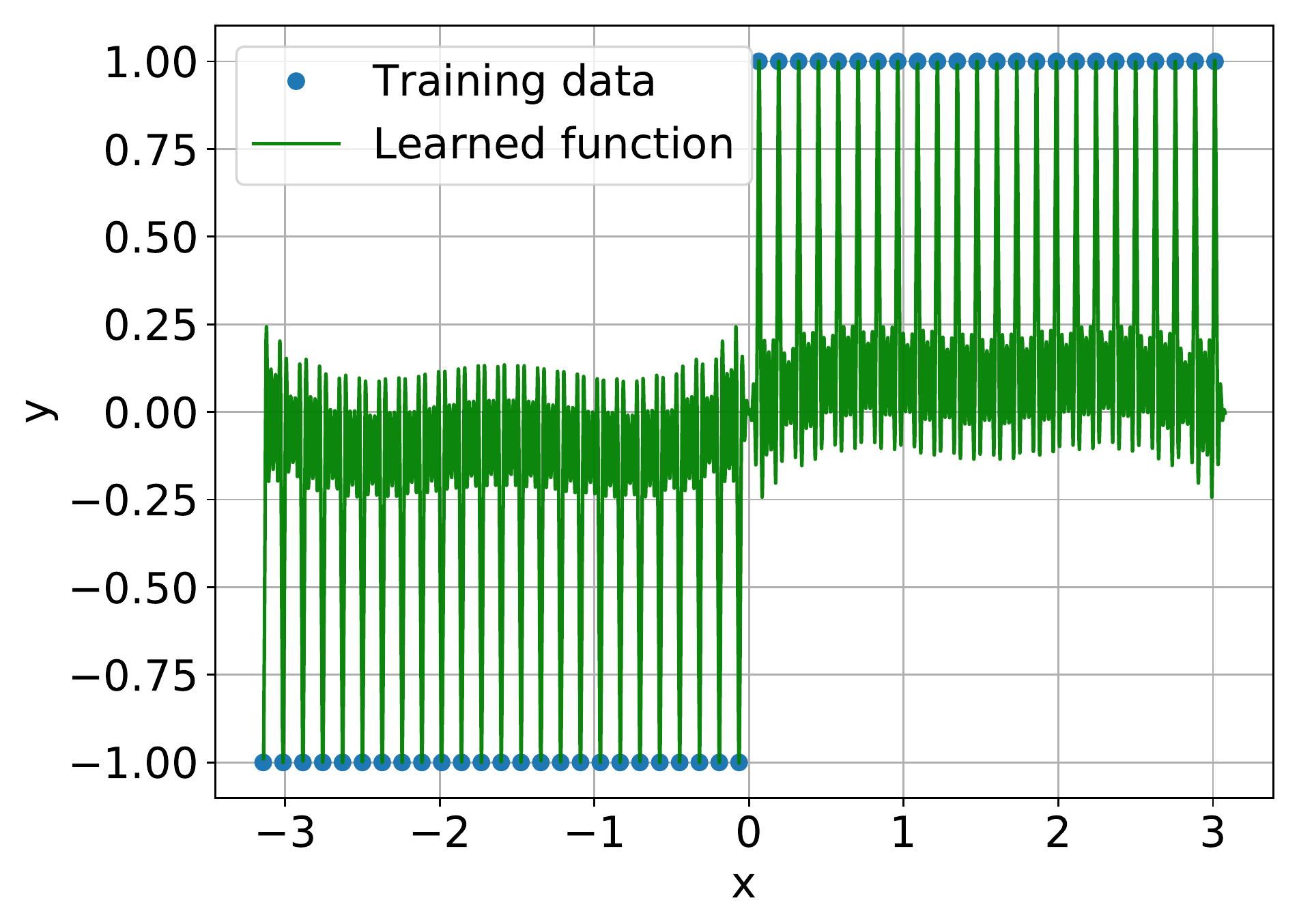}
    \tabularnewline
    \centering
    (e) $\lambda_H = 1.89$
    &
    \centering
    (f) $\lambda_H = 1.89$
  \end{tabular}
  \caption{The three qualitative regimes illustrated using Fourier features and regularly spaced training points. The top corresponds to both regression and classification succeeding, the middle one is the intermediate regime where only classification works, and the bottom one is where neither works. Here $n = 49, s = 7, d = 441$. }
  \label{fig:spikedl2_weights_reg}
\end{figure}

To see this more formally, we can upper bound the expression of classification error and show that it goes to zero as $n \to \infty$ under these conditions\footnote{The \textit{exact} Gaussian-feature expression for classification error in Proposition~\ref{theorem:classloss} depends solely on the ratio $a/\sigma_{CN}.$ Characterizing the exact expression for Fourier features is more challenging because the contamination does not have a clean distribution, but we can upper bound the probability of classification error using the standard deviation alone.}.
We use a union bound together with Chebyshev's inequality in a manner reminiscent of typicality proofs in information theory \citep{cover2012elements}.
Let $\epsilon = \frac{(p-1)}{2} - (q - (1-r))$ be the difference between the relevant two exponents of $n$.
Then, we define the events $\mathcal{A} := \{x |\  |\cos(x)| < 2n^{-\frac{\epsilon}{2}}\}$ and $\mathcal{B} := \{x |\  |B(x)| > n^{-\frac{\epsilon}{2}} n^{-(q-(1-r))} \}$.
The event $\mathcal{A}$ corresponds to having an atypically weak signal in the true feature, and the event $\mathcal{B}$ corresponds to having an atypically large contamination term.
Observe that if neither event $\mathcal{A}$ nor event $\mathcal{B}$ holds, we can substitute Equation~\eqref{eq:survivalapproximation} to get $|a \cos(x)| \geq 2|B(x)|$,
and this implies that a classification error will not occur. Therefore, the probability of classification error is \textit{upper bounded} by $\Pr\left[\mathcal{A} \cup \mathcal{B}\right]$, and by the union bound it suffices to upper bound the probability of each of these events individually.
We start with the ``weak signal'' event $\mathcal{A}$.
Because $\cos(x)$ is a function that is always differentiable in the neighborhood where $\cos(x) = 0$, this means that $\cos(\mathbf{X})$ as a random variable has a density\footnote{This is known as a shifted arc-sine distribution.} in the neighborhood of $0$.
Consequently, we have $\Pr\left[\mathcal{A} \right] = \int_{-n^{-\frac{\epsilon}{2}}}^{+n^{-\frac{\epsilon}{2}}} \frac{1}{\pi\sqrt{1-y^2}} dy = \frac{2}{\pi}\sin^{-1}(n^{-\frac{\epsilon}{2}})$ which goes to zero as $n \rightarrow \infty$.
We now turn to the ``unusually large contamination" event $\mathcal{B}$.
Because $q < (1-r) + \frac{(p-1)}{2}$, we have $\Pr\left[\mathcal{B}\right] = \Pr\left[|B(\mathbf{X})| > n^{-\frac{\epsilon}{2}} n^{-(q-(1-r))}\right] \leq \Pr\left[|B(\mathbf{X})| > n^{\frac{\epsilon}{2}} \sigma_{CN}\right]$.
By Chebyshev's inequality, we have $\Pr\left[\mathcal{B}\right] \leq n^{-\frac{\epsilon}{2}}$, which goes to zero as $n \rightarrow \infty$.

Since the probabilities of both events $\mathcal{A}$ and $\mathcal{B}$ have been shown to go to $0$ as $n \to \infty$, the limiting classification error will also be zero when $q <(1-r) + \frac{(p-1)}{2}$.
In fact, this argument can be extended to the case of interpolation of \textit{binary labels} by using the Fourier series representation of the underlying true label function.
Since there is now misspecification induced by the sign operator, this requires understanding the approximation-theoretic properties of the Fourier series by its first $s$ terms as $s \rightarrow \infty$.
Such a treatment is beyond the scope of this paper, and we defer it to a separate forthcoming note.

Finally, it is worth noting that the above calculation only \emph{upper bounds} the classification error; whether this upper bound is matched by a lower bound remains an open question.
This would shed light on whether, in fact, classification generalizes poorly when $q > (1 - r) + \frac{(p-1)}{2}$ for the case of Fourier featurization.
Figure~\ref{fig:spikedl2_weights_reg} illustrates the three regimes for Fourier featurization.
While the first two regimes display behavior that parallels the Gaussian-feature results in Theorem~\ref{theorem:threeregimes}, the third regime is inconclusive with respect to classification performance.
This is an important question for future work.

\section{Additional notation for proofs}\label{appendix:additionalnotation}

We consider zero-mean Gaussian featurization, i.e. $\phi(X_i) = \mathcal{N}(\mathbf{0}, \Sigmabold)$. For ease of exposition, we consider $\Sigmabold$ to be diagonal\footnote{This is without loss of generality: if $\Sigmabold$ were not diagonal, we could first do a coordinate transformation to the basis of the eigenvectors of $\Sigmabold$.}.
Corresponding to a given index $t \in \{1,\ldots,d\}$, we define the ``leave-one-out" matrix $\Sigmaboldtd$ whose eigenvalues are given by: $\eig_j(\Sigmaboldtd) = \lambdatd_j$ for $j  \in \{1, \ldots, d-1\}$.
The relation between the spectrum $\{\lambdatd_j\}_{j=1}^{d-1}$ and $\{\lambda_j\}_{j=1}^d$ is given by
\begin{align}
    \lambdatd_j = \begin{cases} \lambda_j, & j < \truek \\ \lambda_{j+1}, & j \geq \truek \end{cases}. \label{eq:lambdatd}
\end{align}
Consider $\{\z_i\}_{i=1}^d$ i.i.d. with $\z_i \sim \normal{\mathbf{0}}{\mathbf{I_n}}$.
Observe that we can write effective Gram matrices corresponding to the full as well as the ``leave-one-out" spectrum of the covariance matrix:
\begin{align*}
        \A &=  \sumjd \lambda_j \z_j \z_j\tran = \Atrain \Atrain\tran,
        \quad \quad \quad
     \Amk = \sumjdk \lambda_j \z_j \z_j\tran.
\end{align*}
Using Equation \eqref{eq:lambdatd}, we can also express the ``leave-one-out" Gram matrix $\Amk$ as follows:
\begin{align}
    \Amk = \sum_{j = 1}^{d-1} \lambdatd_j \z_j \z_j\tran. \label{eq:Amkformtd}
\end{align}
We will use both of the above expressions for the leave-one-out matrix $\Amk$ in our analysis.

\section{Support vector proofs and calculations}\label{appendix:svmproofs}

In this section, we collect proofs and calculations that accompany Section~\ref{sec:svm}, which links the SVM to the minimum-$\ell_2$-norm interpolation of binary labels.
We first prove Theorem~\ref{thm:everythingsupport}, and then collect auxiliary calculations for the bi-level ensemble.

\subsection{Proof of Theorem~\ref{thm:everythingsupport}}

Recall that we defined the random Gram matrix $\A$ as
\begin{align*}
    \A := \sum_{j=1}^d \lambda_j \z_j \z_j^\top ,
\end{align*}
where $\z_j \text{ i.i.d.} \sim \mathcal{N}(\mathbf{0},\mathbf{I}_n)$ reflects all the randomness in the matrix $\A$.
Note that the spectrum $\{\lambda_j\}_{j=1}^d$, and all functionals of it, are deterministic.

The dual to the optimization problem~\eqref{eq:l2SVM} can be expressed as below~\citep{boser1992training}:
    \begin{align} \max_{\betabold} \; &\Ytrain^\top \betabold - \frac{1}{2} \betabold^\top \A \betabold \nonumber \\
    &\text{ subject to } \nonumber \\
    Y_i \beta_i &\geq 0 \text { for all } i \in [n] \nonumber.
\end{align}
Note that the \textit{unconstrained} solution of the above is: $\betabold^* := \A^{-1} \Ytrain$.
By complementary slackness, all of the constraints in the optimization problem~\eqref{eq:l2SVM} will be satisfied with equality at optimum i.e.~all training points are support vectors, \textit{if} we have
\begin{align}\label{eq:slacknesscondition}
    Y_i \beta^*_i > 0 \quad \text{for all $i \in [n]$}.
\end{align}
Thus, it suffices to establish conditions under which Equation~\eqref{eq:slacknesscondition} holds with high probability.

We start by showing that this is the case, provided that the condition in  Equation~\eqref{eq:everythingsupportcondition} holds.
To do this, we use the following lemma.
\begin{lemma}\label{lem:opnormconcentration}
Let $\E := \A - ||\lambdabold||_1 \mathbf{I}_n$.
Then, for any choice of positive constant $0 < \epsilon < 1$ and $\tau > 0$, we have (for large enough $n$),
\begin{align*}
    ||\E||_{\mathsf{op}} &\leq \max\{f_1(\lambdabold;\epsilon,\tau), f_2(\lambdabold;\epsilon,\tau)\} \text{ where } \\
    f_1(\lambdabold;\epsilon,\tau) &:= \left(\frac{1}{(1 - \epsilon)^2} - 1\right) ||\lambdabold||_1 \\
     &\quad + \frac{1}{(1-\epsilon)^2}\left(\sqrt{2||\lambdabold||_2^2\left(\tau + n \ln\left(1 + \frac{2}{\epsilon}\right)\right)} + 2||\lambdabold||_{\infty} \cdot \left(\tau + n \ln(1 + \frac{2}{\epsilon}\right)\right) \\
    f_2(\lambdabold;\epsilon,\tau) &:=  \left(\frac{2\epsilon}{1 - \epsilon}\right) ||\lambdabold||_1 \\
     & \quad  + \frac{1 + \epsilon}{1-\epsilon}\left(\sqrt{2||\lambdabold||_2^2\left(\tau + n \ln\left(1 + \frac{2}{\epsilon}\right)\right)}\right) + \frac{2\epsilon}{1 - \epsilon} \left(2||\lambdabold||_{\infty} \cdot \left(\tau + n \ln\left(1 + \frac{2}{\epsilon}\right)\right)\right)
\end{align*}
with probability at least $(1 - 2e^{-\tau})$ over the randomness in the matrix $\A$.
\end{lemma}
Lemma~\ref{lem:opnormconcentration}, which essentially controls the operator norm of the error matrix $\E$ using a union bound  with discretization (also known as the ``epsilon-net" argument), is proved in Appendix~\ref{appendix:lemmaproof:opnormconcentration}.
Now, substituting $\tau := \ln n$ and $\epsilon := \frac{1}{36\sqrt{n}}$, all of the following inequalities can be verified to hold for large enough $n$:
\begin{align*}
    \frac{1}{(1 - \epsilon)^2} - 1 &\leq \frac{1}{12\sqrt{n}} \\
    \frac{2\epsilon}{1 - \epsilon} &\leq  \frac{2}{35\sqrt{n}} \\
    \frac{1}{(1 - \epsilon)^2} \cdot \sqrt{2\left(\tau + n \ln\left(1 + \frac{2}{\epsilon}\right)\right)} &\leq 4\sqrt{n \ln n} \\
    \frac{1 + \epsilon}{1 - \epsilon}\cdot \sqrt{2\left(\tau + n \ln\left(1 + \frac{2}{\epsilon}\right)\right)} &\leq 4\sqrt{n \ln n} \\
    \frac{2}{(1 - \epsilon)^2} \cdot \left(\tau + n \ln\left(1 + \frac{2}{\epsilon}\right)\right) &\leq 8n \ln n \\
    \frac{4\epsilon}{1 - \epsilon} \cdot \left(\tau + n \ln\left(1 + \frac{2}{\epsilon}\right)\right) &\leq \frac{2n \ln n}{3\sqrt{n}} .
\end{align*}
Together, this gives us both of
\begin{align*}
    ||\E||_{\mathsf{op}} &\leq \frac{1}{12\sqrt{n}}\cdot ||\lambdabold||_1 + 4 \sqrt{n \ln n} \cdot ||\lambdabold||_2 + 8n\ln n \cdot ||\lambdabold||_{\infty} \text{ and } \\
    ||\E||_{\mathsf{op}} &\leq \frac{2}{35\sqrt{n}} \cdot ||\lambdabold||_1 + 4\sqrt{n \ln n} \cdot ||\lambdabold||_2 + \frac{2n\ln n}{3\sqrt{n}} \cdot ||\lambdabold||_{\infty}
\end{align*}
with probability at least $(1 - 2e^{-\ln n}) = \left(1 - \frac{2}{n}\right)$.
Now, observe that as a consequence of Equation~\eqref{eq:everythingsupportcondition}, we have
\begin{align*}
    ||\lambdabold||_2 &\leq \frac{||\lambdabold||_1}{72n\sqrt{\ln n}} \text{ and } \\
    ||\lambdabold||_{\infty} &\leq \frac{||\lambdabold||_1}{72 n\sqrt{n} \ln n} .
\end{align*}
Substituting these inequalities above finally gives us
\begin{align*}
    ||\E||_{\mathsf{op}} &\leq \left(\frac{1}{12\sqrt{n}} + \frac{4}{72\sqrt{n}} + \frac{8}{72\sqrt{n}}\right) ||\lambdabold||_1 = \frac{1}{4\sqrt{n}}||\lambdabold||_1 \text{ and } \\
    ||\E||_{\mathsf{op}} &\leq \left(\frac{2}{35\sqrt{n}} + \frac{4}{72\sqrt{n}} + \frac{2}{3\cdot 72n}\right) ||\lambdabold||_1 < \frac{1}{4\sqrt{n}}||\lambdabold||_1 .
\end{align*}
Thus, we have shown that for large enough $n$, we have
\begin{align}\label{eq:opnormupperboundE}
    ||\E||_{\mathsf{op}} \leq \frac{||\lambdabold||_1}{4\sqrt{n}}
\end{align}
with probability at least $\left(1 - \frac{2}{n}\right)$.

\noindent Now, we denote $\E' := \frac{1}{||\lambdabold||_1}\mathbf{I}_n - \A^{-1}$.
Observe that when Equation~\eqref{eq:opnormupperboundE} holds, we have
\begin{align*}
    \eig_{\mathsf{min}}(\A) &= \eig_{\mathsf{min}}(\E) + ||\lambdabold||_1 \\
    &\geq - ||\E||_{\mathsf{op}} + ||\lambdabold||_1 \\
    &\geq - \frac{||\lambdabold||_1}{4 \sqrt{n}} + ||\lambdabold||_1 \\
    &\geq 0.9||\lambdabold||_1 ,
\end{align*}
where the last inequality again holds for large enough $n$.
Thus, for large enough $n$, we have
\begin{align}\label{eq:eigmininequality}
    \eig_{\mathsf{min}}(\A) \geq 0.9||\lambdabold||_1 .
\end{align}
Furthermore, since we can write $\E' = \frac{1}{||\lambdabold||_1} \cdot \Ainv \cdot \E$, we get
\begin{align*}
    ||\E'||_{\mathsf{op}} &\stackrel{\1}{\leq} \frac{1}{||\lambdabold||_1} \cdot ||\Ainv||_{\mathsf{op}} \cdot ||\E||_{\mathsf{op}} \\
    &= \frac{1}{||\lambdabold||_1 \cdot \mu_{\mathsf{min}}(\A)} \cdot ||\E||_{\mathsf{op}} \\
    &\stackrel{\2}{\leq} \frac{1}{0.9 \cdot ||\lambdabold||_1^2} \cdot ||\E||_{\mathsf{op}} \\
    &\stackrel{\3}{\leq} \frac{1}{0.9 \cdot ||\lambdabold||_1} \cdot \frac{1}{4 \sqrt{n}} \\
    &\leq \frac{1}{||\lambdabold||_1} \cdot \frac{1}{2\sqrt{n}} ,
\end{align*}
where inequality $\1$ uses the standard inequality on product of operator norms, inequality $\2$ substitutes Equation~\eqref{eq:eigmininequality}, and inequality $\3$ substitutes Equation~\eqref{eq:opnormupperboundE}.
Thus, we get, for every $i \in [n]$,
\begin{align*}
    Y_i \beta^*_i &= Y_i \evec_i^\top \A^{-1} \Ytrain \\
    &= Y_i \evec_i^\top \left(\frac{1}{||\lambdabold||_1} \mathbf{I}_n + \E'\right) \Ytrain \\
    &= \frac{1}{||\lambdabold||_1} + Y_i \evec_i^\top \E' \Ytrain \\
    &\stackrel{\1}{\geq} \frac{1}{||\lambdabold||_1} - \sqrt{n} ||\E'||_{\mathsf{op}} \\
    &\stackrel{\2}{\geq} \frac{1}{||\lambdabold||_1} - \sqrt{n} \cdot \frac{1}{||\lambdabold||_1} \cdot \frac{1}{2\sqrt{n}} \\
    &= \frac{1}{2||\lambdabold||_1} > 0
\end{align*}
for large enough $n$ and with probability at least $\left(1 - \frac{2}{n}\right)$.
Here, inequality $\1$ follows from the inequality $\mathbf{a}^\top \mathbf{M} \mathbf{b} \geq - ||\mathbf{a}||_2 ||\mathbf{b}||_2 ||\mathbf{M}||_{\mathsf{op}} $, and inequality $\2$ follows by substituting the upper bound we just derived on $||\E'||_{\mathsf{op}}$.
This completes our proof for Part 1 of Theorem~\ref{thm:everythingsupport}.

Next, we show that Equation~\eqref{eq:slacknesscondition} holds with high probability under the condition provided in Equation~\eqref{eq:everythingsupportconditiongaussian}, which is the strictly sharper condition for the isotropic Gaussian case.
For every $i \in [n]$, we denote $\vvec_i := \sqrt{n} \cdot Y_i \evec_i$.
We add and subtract terms to get
\begin{align*}
    Y_i \beta^*_i &= \frac{1}{\sqrt{n}} \vvec_i^\top \A^{-1} \Ytrain \\
    &= \frac{1}{4\sqrt{n}} \left((\vvec_i + \Ytrain)^\top \A^{-1}(\vvec_i + \Ytrain) - (\vvec_i - \Ytrain)^\top \A^{-1} (\vvec_i - \Ytrain)\right) .
\end{align*}
We use the following technical lemma that shows concentration on quadratic forms of the inverse Wishart matrix $\A^{-1}$.
From here on, we denote $d'(n) := (d - n + 1)$ for shorthand.
\begin{lemma}\label{lem:invwishartconcentration}
Let $\A \sim \text{Wishart}(d,\mathbf{I}_n)$.
For any vector $\uvec \in S^{n-1}$ and any $t > 0$, we have
\begin{align*}
    \Pr\left[\frac{1}{\uvec^\top \A^{-1} \uvec} > d'(n) + \sqrt{2t \cdot d'(n)} + 2t \right] &\leq e^{-t} \\
        \Pr\left[\frac{1}{\uvec^\top \A^{-1} \uvec} < d'(n) - \sqrt{2t \cdot d'(n)}\right] &\leq e^{-t} .
\end{align*}
provided that $d'(n) > 2 \max\{t,1\}$.
\end{lemma}
Lemma~\ref{lem:invwishartconcentration} is proved in Appendix~\ref{appendix:lemmaproof:invwishartconcentration}.
Substituting the lower tail bound of Lemma~\ref{lem:invwishartconcentration} with $t := 2 \ln n$ gives us
\begin{align*}
    (\vvec_i + \Ytrain)^\top \A^{-1} (\vvec_i + \Ytrain) \geq \frac{||\vvec_i + \Ytrain||_2^2}{d'(n) + \sqrt{4\ln n \cdot d'(n)} + 4 \ln n }
\end{align*}
with probability at least $\left(1 - \frac{1}{n^2}\right)$.
Similarly, substituting the upper tail bound with $t := 2 \ln n$ gives us
\begin{align*}
    (\vvec_i - \Ytrain)^\top \A^{-1} (\vvec_i - \Ytrain) &\leq \frac{||\vvec_i - \Ytrain||_2^2}{d'(n) - \sqrt{4 \ln n \cdot d'(n)}}
\end{align*}
with probability at least $\left(1 - \frac{1}{n^2}\right)$.
Noting that $||\vvec_i + \Ytrain||_2^2 = 2(n + \sqrt{n})$ and $||\vvec_i - \Ytrain||_2^2 = 2(n - \sqrt{n})$, we then get
\begin{align*}
    Y_i \beta^*_i &\propto \frac{n + \sqrt{n}}{d'(n) + \sqrt{4\ln n \cdot d'(n)} + 4 \ln n } - \frac{n - \sqrt{n}}{d'(n) - \sqrt{4 \ln n \cdot d'(n)}} \\
    &= \frac{2\sqrt{n}d'(n) - 2n\sqrt{4\ln n \cdot d'(n)} - (4\ln n)(n - \sqrt{n}) }{(d'(n) + \sqrt{4\ln n \cdot d'(n)})(d'(n) - \sqrt{4 \ln n \cdot d'(n)})} \\
    &> \frac{2\sqrt{n}d'(n) - 2n\sqrt{4\ln n \cdot d'(n)} - 4 n \cdot \ln n }{(d'(n) + \sqrt{4\ln n \cdot d'(n)})(d'(n) - \sqrt{4 \ln n \cdot d'(n)})} \\
    &> 0
\end{align*}
\textit{if} we have
\begin{align*}
    d'(n) > 9n \ln n \iff d > 9n \ln n + n - 1 ,
\end{align*}
which is precisely the condition in Equation~\eqref{eq:everythingsupportconditiongaussian}.
Under this condition, we have proved that for any training data point corresponding to $i \in \{1,\ldots, n\}$, we have $Y_i \beta^*_i > 0$ with probability at least $\left(1 - \frac{2}{n^2}\right)$.
Finally, applying the union bound on all $n$ training data points gives us
\begin{align*}
    Y_i \beta^*_i > 0 \text{ for all } i \in \{1,\ldots,n\}
\end{align*}
with probability at least $\left(1 - \frac{2}{n}\right)$.
This completes the proof of Theorem~\ref{thm:everythingsupport}.
\qed

\subsection{Implications of Theorem~\ref{thm:everythingsupport} for the bi-level ensemble}\label{appendix:svm_bilevel}

In this section, we provide the calculations that help us understand the ramifications of Theorem~\ref{thm:everythingsupport} --- in particular, the condition in Equation~\eqref{eq:everythingsupportcondition} --- for the bi-level ensemble (Definition~\ref{def:bilevel_covariance}).
We reproduce Equation~\eqref{eq:everythingsupportcondition} below:
\begin{align*}
    ||\lambdabold||_1 \geq 72 \left(||\lambdabold||_2 \cdot n\sqrt{\ln n} + ||\lambdabold||_{\infty} \cdot n\sqrt{n} \ln n + 1\right) .
\end{align*}
We substitute the parameters of the bi-level ensemble into the left hand and right hand sides of the inequality.
Recall that, by definition, we have $||\lambdabold||_1 = d = n^p$ for the bi-level ensemble and so the left hand side is equal to $n^p$.
On the other hand, for the right hand side, a simple calculation shows that
\begin{align*}
    ||\lambdabold||_2 &= \sqrt{s \cdot \frac{a^2 d^2}{s^2} + (d - s) \cdot \frac{(1 - a)^2 d^2}{(d-s)^2}} \\
    &\stackrel{\1}{\asymp} \sqrt{\frac{a^2 d^2}{s} + d} \\
    &= \sqrt{n^{2p - 2q - r} + n^p} \\
    &\asymp n^{\max\{p - q - \frac{r}{2}, \frac{p}{2}\}} ,
\end{align*}
where the scaling in $\1$ follows because the bi-level ensemble defines $r < 1 < p$ and $q > 0$ (so $(1 - a) \asymp 1$ and $(d - s) \asymp d$).
Moreover, we have
\begin{align*}
    ||\lambdabold||_{\infty} = \frac{ad}{s} = n^{p - q - r} .
\end{align*}
Putting these together, the right hand side of Equation~\eqref{eq:everythingsupportcondition} scales as
\begin{align*}
    &72 \left(||\lambdabold||_2 \cdot n\sqrt{\ln n} + ||\lambdabold||_{\infty} \cdot n\sqrt{n} \ln n + 1\right) \\
    &\asymp n^{\max\{p - q - \frac{r}{2}, \frac{p}{2}\} + 1} \cdot \sqrt{\ln n} + n^{p + \frac{3}{2} - q - r} \cdot (\ln n) + 1 ,
\end{align*}
and so, for Equation~\eqref{eq:everythingsupportcondition} to hold, we get the following \textit{sufficient} conditions on the parameters $(p,q,r)$ of the bi-level ensemble for sufficiently large\footnote{The reason for requiring sufficiently large $n$ in these statements is the application of the $\asymp$ relation in multiple places.
(Also note that Theorem~\ref{thm:everythingsupport} also required sufficiently large $n$.)
Accordingly, we can also omit constants from consideration.} $n$:
\begin{align*}
    p &> \frac{p}{2} + 1 \implies p > 2 \\
    p &> p - q - \frac{r}{2} + 1 \implies q > \left(1 - \frac{r}{2}\right) \\
    p &> p + \frac{3}{2} - q - r \implies q > \left(\frac{3}{2} - r\right)
\end{align*}
Now, observe that $1 - \frac{r}{2} \leq \frac{3}{2} - r$ for all $0 \leq r \leq 1$, and so we get \textit{sufficient conditions} as follows:
\begin{align*}
p &> 2 \text{ and } q > \left(\frac{3}{2} -r \right) ,
\end{align*}
These are precisely the conditions in Equation~\eqref{eq:everythingsupportcondition_bilevel}.

\section{Proof of Theorem~\ref{theorem:threeregimes}: Bounds on survival and contamination}\label{appendix:gaussianproofs}

In this section, we obtain a general, non-asymptotic characterization of classification (and regression) error by bounding survival and contamination terms.
As described in Section~\ref{sec:pathtoanalysis}, this is then plugged into the expressions in Proposition~\ref{theorem:classloss} to prove Theorem~\ref{theorem:threeregimes}.

First, we define shorthand notation that is useful for this section, in addition to the notation already defined in Appendix~\ref{appendix:additionalnotation}.
For ease of notation, we denote the survival and contamination factors under the $1$-sparse model for the case where we interpolate binary labels as
\begin{align*}
    \survival(t) = \survivalb(\alphahatbold_{2, \mathsf{binary}}, t), \quad  \contamination(t) = \contaminationb(\alphahatbold_{2, \mathsf{binary}}, t),
\end{align*}
and for the case where we interpolate real output as
\begin{align*}
    \survivalr(t) = \survivalb(\alphahatbold_{2, \mathsf{real}}, t), \quad  \contaminationr(t) = \contaminationb(\alphahatbold_{2, \mathsf{real}}, t).
\end{align*}
Finally, for a given index $t \in \{1,\ldots, d\}$, we denote as shorthand $\ztk := \Ztrain$.
It is easy to verify that $\ztk \sim \mathcal{N}(\mathbf{0}, \mathbf{I}_n)$ under the $1$-sparse Assumption~\ref{assumptionsimple}.
We also denote $\ytk := \Ytrain$.
Recall that we consider the possibility of label noise probability equal to $\betastar$: from the generative model defined in Equation~\eqref{eq:generativemodel}, we have
\begin{align}\label{eq:ytkequation}
    y_{t,i} = \begin{cases}
    \sgn(z_{t,i}) \text{ with probability } (1 - \betastar) \\
    - \sgn(z_{t,i}) \text{ with probability } \betastar .
    \end{cases}
\end{align}
for every $i \in \{1,\ldots, n\}$.
Finally, for a given positive semi-definite matrix $\Mbold \in \mathbb{R}^{d \times d}$ and a given index\\ $k \in \{0,\ldots, (d-1)\}$, we define the \textit{effective rank}
\begin{align*}
    r_k(\Mbold) := \frac{\sum_{\ell > k} \mu_{\ell}(\Mbold)}{\mu_{k+1}(\Mbold)} .
\end{align*}
Recall that this is the precisely the definition of the first effective rank in~\cite{bartlett2020benign}, which dictates the contribution of pure signal to regression test error incurred by the minimum-$\ell_2$-norm interpolation.

\subsection{Bounds on survival and contamination}\label{appendix:survivalcontaminationtheoremstatements}

The notions of survival and contamination were first introduced in~\citet{muthukumar2020harmless}, and characterized there with equality for Fourier featurization on regularly spaced training data.
Here, we characterize these quantities for Gaussian features.
We state our upper and lower bounds on survival and contamination respectively for two cases --- when the output being interpolated is binary, and when the output being interpolated is real.
We start with upper and lower bounds on the survival factor.

\begin{theorem}[Upper and lower bounds on survival] \label{theorem:survivalbounds}

There exist universal positive \newline constants $(b, b_2, c, c_3, c_4)$ (that do not depend on parameters $(n,d,k,\Sigmabold)$) such that if $r_k(\Sigmabold) \geq bn$ and  $r_k(\Sigmaboldtd) \geq b_2n$,
 we have the following characterizations of the survival factor for any $k \geq t$:
\begin{enumerate}
    \item \textbf{Interpolation of binary labels:} The minimum-$\ell_2$-norm interpolation of binary labels, i.e.~$\alphahatbold_{2,\mathsf{binary}}$, satisfies each of
\begin{subequations}
\begin{align}
    \survival(\truek) &\geq  \sqrt{\frac{2}{\pi}} \cdot (1 - 2\betastar) \cdot \frac{\lambda_t\left(\frac{(n-k)}{c\widetilde{\lambda}_{k+1} r_k(\Sigmaboldtd)} - \frac{c_3 n^{3/4}}{\lambda_{k+1} r_k(\Sigmabold)} \right)}{1 + \lambda_t\left(\frac{cn}{\widetilde{\lambda}_{k+1} r_k(\Sigmaboldtd)} + \frac{c_4 n^{3/4}}{\lambda_{k+1} r_k(\Sigmabold)} \right)}, \text{ and } \label{eq:survivallowerbound} \\
    \survival(\truek) &\leq  \sqrt{\frac{2}{\pi}} \cdot (1 - 2\betastar) \cdot \frac{\lambda_t\left(\frac{cn}{\widetilde{\lambda}_{k+1} r_k(\Sigmaboldtd)} + \frac{c_3 n^{3/4}}{\lambda_{k+1} r_k(\Sigmabold)} \right)}{1 + \lambda_t\left(\frac{(n-k)}{c\widetilde{\lambda}_{k+1} r_k(\Sigmaboldtd)} - \frac{c_4 n^{3/4}}{\lambda_{k+1} r_k(\Sigmabold)} \right)} \label{eq:survivalupperbound}
\end{align}
\end{subequations}
with probability at least $(1 - 3e^{-\sqrt{n}} - 2e^{-\frac{n}{c}})$ over the randomness in the training data $\{X_i,Y_i\}_{i=1}^n$.
\item \textbf{Interpolation of real output:} The minimum-$\ell_2$-norm interpolation of real output, i.e.~$\alphahatbold_{2,\mathsf{real}}$, satisfies each of
\begin{subequations}
\begin{align}
    \survivalr(\truek) &\geq \frac{1}{1 + \frac{1}{\lambdatk\left(\frac{(n-k)}{c\lambdatd_{k+1}r_k(\Sigmaboldtd)} - \frac{c_4 n^{\frac{3}{4}}}{\lambda_{k+1}r_k((\Sigmabold)} \right)}}, \text{ and }
    \label{eq:survivallowerboundreal} \\
    \survivalr(\truek) &\leq \frac{1}{1 + \frac{1}{\lambdatk\left(\frac{cn}{\lambdatd_{k+1}r_k(\Sigmaboldtd)} + \frac{c_4 n^{\frac{3}{4}}}{\lambda_{k+1}r_k((\Sigmabold)} \right)}} \label{eq:survivalupperboundreal}
\end{align}
\end{subequations}
with probability at least $(1 - 2e^{-\sqrt{n}} - 2e^{-\frac{n}{c}})$ over the randomness in the training data $\{X_i,Y_i\}_{i=1}^n$.
\end{enumerate}
\end{theorem}
We will see subsequently (in Appendix~\ref{appendix:bilevel}) that the survival bounds, whether binary labels or real output are interpolated, are matching in their dependence on $n$ up to constants.
We now state our characterization of the contamination factor.

\begin{theorem}[Upper and lower bounds on contamination]\label{theorem:contaminationbounds}
There exist universal \newline positive constants $b_2, c_5, c_6, c_7, c_8, c_9$ (that do not depend on parameters $(n,d,k,\Sigmabold)$) such that if $0 \leq k \leq n/c_5$ and $r_k(\Sigmaboldtd) \geq b_2$,  the following characterizations of  the contamination factor hold for any choice of $\ell \leq k$:
\begin{enumerate}
    \item \textbf{Interpolation of binary labels:} Provided that $n \geq c_6$ %
    , the minimum-$\ell_2$-norm interpolation of binary labels, i.e.~$\alphahatbold_{2,\mathsf{binary}}$, satisfies each of
    \begin{subequations}
    \begin{align}
      \contamination(\truek) &\leq c_7 \cdot \sqrt{\left(\frac{\ell}{n} + n \cdot  \frac{\sum_{j> \ell} \lambdatd_j^2}{\left(\sum_{j>k}\lambdatd_j \right)^2} \right) \cdot \ln n \cdot (1 + \survival(\truek)^2)} \text{, and }  \label{eq:contaminatonupperboundsu} \\
      \contamination(\truek) &\geq \sqrt{n}\cdot \frac{\sqrt{r_k\left(\Sigmaboldtd^2\right) \cdot \lambdatd_{k+1}^2 }}{c_9 \left(\sum\limits_{j=1}^d \lambda_j + \lambda_1 n\right)}\label{eq:contaminatorlowerbound}
    \end{align}
    \end{subequations}
    \textit{almost surely} for any realization of the random quantity $\survival(\truek)$, and with probability at least $\left(1 - \frac{3}{n}\right)$ and $(1 - 2e^{-\frac{n}{c_8}})$ respectively over the randomness in the training data $\{X_i,Y_i\}_{i=1}^n$.

    \item \textbf{Interpolation of real output:} Provided that $n \geq c_6$, the minimum-$\ell_2$-norm interpolation of real output, i.e.~$\alphahatbold_{2,\mathsf{real}}$, satisfies each of
    \begin{subequations}
    \begin{align}
        \contaminationr(\truek) &\leq c_7 \abs{1 - \survivalr(\truek)}\cdot \sqrt{\left(\frac{l}{n} + n \cdot \frac{\sum_{j> l} \lambdatd_j^2}{\left(\sum_{j>k}\lambdatd_j \right)^2} \right)\cdot  \ln n} \text{, and } \label{eq:contaminatonupperboundsureal} \\
        \contaminationr(\truek) &\geq \sqrt{n(1-\delta)}\cdot \frac{\sqrt{r_k\left(\Sigmaboldtd^2\right) \cdot \lambdatd_{k+1}^2 }}{c_9\left(\sum\limits_{j=1}^d \lambda_j + \lambda_1 n\right)}\label{eq:contaminatorlowerboundreal}
    \end{align}
    \end{subequations}
    \textit{almost surely} for any realization of the random quantity $\survival(\truek)$, and with probability at least $\left(1 - \frac{2}{n}\right)$ and $(1 - 2e^{-\frac{n}{c_8}} - e^{-n\delta^2})$ respectively over the randomness in the training data $\{X_i,Y_i\}_{i=1}^n$.
\end{enumerate}
\end{theorem}
Observe that the high-probability characterizations of contamination in Theorem~\ref{theorem:contaminationbounds} themselves hold almost surely for every realization of the respective survival factors for binary and real interpolation, which are random variables.
In Appendix~\ref{appendix:bilevel}, these expressions will be used together (with a simple union bound) with the matching high-probability characterization of survival factor in Theorem~\ref{theorem:survivalbounds}.
Unlike for the case of survival, the upper and lower bounds for contamination are not necessarily matching --- however, as we will see in Appendix~\ref{appendix:bilevel}, they turn out to match for all parameterizations of the bi-level ensemble.

As a final remark, in both theorem statements, the only randomness over which all probabilities are taken is solely in the training data $\{X_i,Y_i\}_{i=1}^n$.
Further, all universal positive constants are taken to be independent of the parameters $(n,d,k,\Sigmabold)$, which entirely describe the problem.
In the proofs of Theorems~\ref{theorem:survivalbounds} and~\ref{theorem:contaminationbounds}, we will follow these conventions unless specified otherwise.

\subsection{Background lemmas}\label{appendix:usefullemmas}

We begin our proofs of Theorems~\ref{theorem:survivalbounds} and~\ref{theorem:contaminationbounds} by stating lemmas that serve as background for our analysis.
The first lemma is from~\citet{bartlett2020benign}.

\begin{lemma}~\textbf{\citep[Concentration of eigenvalues, Lemmas 9 and 10 in][]{bartlett2020benign}}
There exist universal positive constants $(b, c)$ such that:
\begin{enumerate}
\item For any $k \geq 0$ such that $r_k(\Sigmabold) \geq bn$, we have
\begin{align}
    \frac{1}{c}\lambda_{k+1} r_k(\Sigmabold) &\leq \eig_n\left(\A\right) \leq \eig_1\left(\A\right) \leq c\left( \sum\limits_{j=1}^d\lambda_j + \lambda_1 n \right)   \text{ and } \label{eq:A_eigbounds}\\
    \eig_{k+1}(\A) &\leq c \lambda_{k+1} r_k(\Sigmabold) \label{eq:suproofeigvaluebound}
\end{align}
with probability at least $(1 - 2e^{-\frac{n}{c}})$ over the random matrix $\A$.
\item For any $k \geq t$ such that $r_k(\Sigmabold) \geq bn$, we have
\begin{align}\label{eq:Ami_eigbounds}
    \frac{1}{c}\lambda_{k+1} r_k(\Sigmabold) &\leq \eig_n\left(\Amk\right) \leq \eig_1\left(\Amk\right) \leq c\left( \sum\limits_{j=1}^d\lambda_j + \lambda_1 n \right)
\end{align}
with probability at least $(1 - 2e^{-\frac{n}{c}})$ over the random matrix $\Amk$.
\end{enumerate}{}
Further, as corollaries to the above, we have the following statements:{}
\begin{enumerate}
\item For any $k \geq 0$ such that $r_k(\Sigmabold) \geq b n$, we have
\begin{align}
    \frac{1}{c\left( \sum\limits_{j=1}^d\lambda_j + \lambda_1 n \right)} \leq \eig_n\left(\A^{-1}\right) \leq \eig_1\left(\A^{-1}\right) \leq \frac{c}{\lambda_{k+1} r_k(\Sigmabold)} \label{eq:Ainv_eigbounds}
\end{align}
with probability at least $(1 - 2e^{-\frac{n}{c}})$ over the random matrix $\A$.
\item For any $k \geq t$ such that $r_k(\Sigmabold) \geq b n$,  we have
\begin{align}
    \frac{1}{c\left( \sum\limits_{j=1}^d\lambda_j + \lambda_1 n \right)} \leq \eig_n\left(\Amkinv\right) \leq \eig_1\left(\Amkinv\right) \leq \frac{c}{\lambda_{k+1} r_k(\Sigmabold)} \label{eq:Amiinv_eigbounds}
\end{align}
with probability at least $(1 - 2e^{-\frac{n}{c}})$ over the random matrix $\Amk$.
\end{enumerate}
\label{lemma:eigboundsAmiinv}
\end{lemma}
Note that using Equation \eqref{eq:Amkformtd} to express $\Amk$, we can rewrite the bounds in the above lemma in terms of the quantities $\Sigmaboldtd$ and $\lambdatd_j$.
In particular, it follows that each of
\begin{subequations}
\begin{align}
    \frac{1}{c}\lambdatd_{k+1} r_k(\Sigmaboldtd) \leq \eig_n\left(\Amk\right) &\leq \eig_1\left(\Amk\right) \leq c\left( \sum\limits_{j=1}^{d-1}\lambdatd_j + \lambdatd_1 n \right) \text{ and } \label{eq:Amk_eigbounds_td} \\
    \frac{1}{c\left( \sum\limits_{j=1}^{d-1}\lambdatd_j + \lambdatd_1 n \right)} \leq \eig_n\left(\Amk^{-1}\right) &\leq \eig_1\left(\Amk^{-1}\right) \leq \frac{c}{\lambdatd_{k+1} r_k(\Sigmaboldtd)}.\label{eq:Amkinv_eigbounds_td}
\end{align}
\end{subequations}
holds with probability at least $(1 - 2e^{-\frac{n}{c}})$.
We will also apply Equation~\eqref{eq:suproofeigvaluebound} with $\Amk$ instead of $\A$, and use the corresponding condition $r_k(\Sigmabold) \geq b_2 n$.

\noindent The next lemma is the Hanson-Wright inequality, which shows that the quadratic form of a (sub)-Gaussian random vector concentrates around its expectation.

\begin{lemma}\textbf{~\citep[Hanson-Wright inequality,][]{rudelson2013hanson}}\label{lem:hansonwright}
Let $\z$ be a random vector composed of i.i.d.~random variables that are zero mean and  sub-Gaussian with parameter at most $1$.
Then, there exists universal constant $c > 0$ such that for any positive semi-definite matrix $\Mbold$ and for every $t \geq 0$, we have
\begin{align*}
    \Pr\left[|\z^\top \Mbold \z - \EE[\z^\top \Mbold \z ]| > t \right] \leq 2 \exp\left\{-c \min\left\{\frac{t^2}{||\Mbold||_{\mathsf{F}}^2}, \frac{t}{||\Mbold||_{\mathsf{op}}}\right\}\right\}.
\end{align*}
\end{lemma}
We will apply this inequality in two ways.
First, we will note that $||\Mbold||_{\mathsf{F}}^2 \leq n ||\Mbold||_{\mathsf{op}}^2$ and substitute $t := c_1 ||\Mbold||_{\mathsf{op}} \cdot n^{3/4}$ (where $c_1^2 = \frac{1}{c}$) to get
\begin{align}\label{eq:hansonwright}
    |\z^\top \Mbold \z - \EE[\z^\top \Mbold \z ]| \leq c_1 ||\Mbold||_{\mathsf{op}} \cdot n^{3/4}
\end{align}
with probability at least $(1 - 2e^{-\sqrt{n}})$.
Second, we will note that $||\Mbold||_{\mathsf{op}} \leq \tr(\Mbold)$ and moreover, $||\Mbold||_{\mathsf{F}}^2 = \tr(\Mbold^2) \leq (\tr(\Mbold))^2$.
Then, substituting $t := \frac{1}{c} \cdot \tr(\Mbold) \cdot (\ln n)$, we get
\begin{align}\label{eq:hansonwright2}
    \z^\top \Mbold \z \leq \EE[\z^\top \Mbold \z ] + \frac{1}{c} \cdot \tr(\Mbold) \cdot (\ln n) \leq \left(1 + \frac{1}{c}\right) \cdot \tr(\Mbold) \cdot (\ln n)
\end{align}
with probability at least $(1 - \frac{1}{n})$.
Finally, note that all probabilities are only over the random vector $\z$.
We will frequently apply Lemma~\ref{lem:hansonwright} as a high-probability statement conditioned on the realization of a \textit{random}, almost surely positive semi-definite matrix $\Mbold$ which is independent of $\z$.

Finally, the following lemma bounds the squared norm of a Gaussian random vector by a standard tail bound on chi-squared random variables ~\citep[for e.g. see][Chapter~2]{wainwright2019high}, stated for completeness.
\begin{lemma}
Let $\z \sim \mathcal{N}(\mathbf{0},\mathbf{I}_n)$.
Then, for any $\delta \in (0,1)$, we have
\begin{align}
    n(1 - \delta) \leq \|\z \|_2^2 \leq n( 1+ \delta)
\end{align}
with probability at least $(1-2e^{-n\delta^2})$.
\label{lemma:chisquaredtailbound}
\end{lemma}

\subsection{Proof of Theorem \ref{theorem:survivalbounds}}\label{appendix:proofsurvivalbounds}

We first prove Theorem~\ref{theorem:survivalbounds}, i.e.~upper and lower bounds on survival when binary labels \textit{or} real output are interpolated.
We start with the slightly more difficult case of interpolation of binary labels (Equations~\eqref{eq:survivallowerbound} and~\eqref{eq:survivalupperbound}).

\subsubsection{Interpolation of binary labels}\label{appendix:binarylabelsurvival}

Recall that, by Assumption~\ref{assumptionsimple}, we have $\alpha^*_t = \frac{1}{\sqrt{\lambda_t}}$.
A standard argument based on Moore-Penrose pseudoinverse calculations shows that $\alphahatbold_{2,\mathsf{binary}} = \Atrain\tran (\Atrain \Atrain\tran)^{-1} \Ytrain$.
We get
\begin{align*}
    \survival(\truek) &= \frac{\alphahat_{\truek,2,\mathsf{binary}}}{\alpha^*_{\truek}}\\
    &= \sqrt{\lambda_{\truek}} \alphahat_{\truek,2,\mathsf{binary}}\\
    &= \sqrt{\lambdatk} \e_{\truek} \tran \Atrain\tran (\Atrain \Atrain\tran)^{-1} \Ytrain \\
    &= \lambda_t \ztk^\top \Ainv \ytk ,
\end{align*}
where $\ztk, \ytk$ are as defined at the beginning of Appendix~\ref{appendix:gaussianproofs}, and $\A$ is the Gram matrix defined in Appendix~\ref{appendix:additionalnotation}.
Next, we use the Sherman-Morrison-Woodbury identity to get
\begin{align}
    \Ainv &= (\lambdatk \ztk \ztk \tran + \Amk)^{-1} \nonumber \\
    &= \Amkinv - \frac{\lambdatk \Amkinv \ztk \ztk \tran \Amkinv}{1 + \lambdatk \ztk \tran \Amkinv \ztk} . \label{eq:ShermanMorisson}
\end{align}
Using this, we obtain
\begin{align}
    \survival(\truek) &=  \frac{\lambdatk \ztk \tran \Amkinv \ytk}{1 + \lambdatk \ztk \tran \Amkinv \ztk} \label{eq:suform} .
\end{align}
Adding and subtracting terms to the numerator, we get
\begin{align*}
    \ztk^\top \Amkinv \ytk &= \frac{1}{4} \left((\ztk + \ytk)^\top \Amkinv (\ztk + \ytk) - (\ztk - \ytk)^\top \Amkinv (\ztk - \ytk) \right) .
\end{align*}
Because of the ``leave-one-out" property, note that $\Amkinv \perp \{\ztk, \ytk\}$.
Also note that $\Amkinv$ is almost surely positive semidefinite.
Thus, we can upper \textit{and} lower bound the numerator of Equation~\eqref{eq:suform} around its expectation using the Hanson-Wright inequality.
First, we calculate the conditional expectation:
\begin{align*}
    \EE\left[\ztk^\top \Amkinv \ytk \Big{|} \Amkinv \right] &= \EE\left[\tr(\Amkinv \ytk \ztk^\top) \Big{|} \Amkinv \right] \\
    &= \tr\left(\Amkinv \cdot \EE\left[\ytk \ztk^\top \right]\right) .
\end{align*}
Recalling the expression for $\ytk$ from Equation~\eqref{eq:ytkequation}, a simple calculation yields that
\begin{align*}
    \EE\left[\ytk \ztk^\top\right] &= \EE\left[y_{t,1} z_{t,1}^\top \right] \cdot \mathbf{I}_n \\
    &= \left((1 - \betastar) \EE\left[\sgn(z_{t,1}) z_{t,1}^\top\right] + \betastar \EE\left[- \sgn(z_{t,1}) z_{t,1}^\top\right]\right) \cdot \mathbf{I}_n \\
    &= (1 - 2\betastar) \EE\left[\sgn(z_{t,1}) z_{t,1}^\top\right] \cdot \mathbf{I}_n \\
    &= (1 - 2\betastar) \cdot \sqrt{\frac{2}{\pi}} \cdot \mathbf{I}_n ,
\end{align*}
where the last step follows because $z_{t,1} \sim \mathcal{N}(0,1)$.

Now, we apply Equation~\eqref{eq:hansonwright} (the Hanson-Wright inequality) \textit{almost surely} for every realization of the random matrix $\Amkinv$, and simultaneously to the quadratic forms $(\ztk + \ytk)^\top \Amkinv (\ztk + \ytk)$ and $(\ztk - \ytk)^\top \Amkinv (\ztk - \ytk)$.
Thus, we have each of
\begin{align*}
    \ztk^\top \Amkinv \ytk &\geq \left((1 - 2\betastar) \sqrt{\frac{2}{\pi}} \tr(\Amkinv) - 2c_1 ||\Amkinv||_{\mathsf{op}} \cdot n^{3/4}\right) \text{ and } \\
    \ztk^\top \Amkinv \ytk &\leq \left((1 - 2\betastar) \sqrt{\frac{2}{\pi}} \tr(\Amkinv) + 2c_1 ||\Amkinv||_{\mathsf{op}} \cdot n^{3/4}\right)
\end{align*}
with probability at least $(1 - 2e^{-\sqrt{n}})$ over the randomness in $\{\ztk,\ytk\}$.
Similarly, to bound the the denominator, we have each of
\begin{subequations}
\begin{align}
    \ztk^\top \Amkinv \ztk &\geq \tr(\Amkinv) - c_1 ||\Amkinv||_{\mathsf{op}} \cdot n^{3/4} \text{ and } \label{eq:zAzlb} \\
    \ztk^\top \Amkinv \ztk &\leq \tr(\Amkinv) + c_1 ||\Amkinv||_{\mathsf{op}} \cdot n^{3/4}   \label{eq:zAzub}
\end{align}
\end{subequations}
with probability at least $(1 - e^{-\sqrt{n}})$ over the randomness in $\{\ztk,\ytk\}$.
Substituting these bounds into Equation~\eqref{eq:suform}, we get each of
\begin{align*}
    \survival(\truek) &\geq \frac{\lambda_{\truek} \cdot \left(\sqrt{\frac{2}{\pi}}(1 - 2\betastar) \tr(\Amkinv) - 2c_1 ||\Amkinv||_{\mathsf{op}} \cdot n^{3/4}\right)}{1 + \lambda_{\truek}\left(\tr(\Amkinv) + c_1 ||\Amkinv||_{\mathsf{op}} \cdot n^{3/4}\right)} \text{ and } \\
    \survival(\truek) &\leq \frac{\lambda_{\truek} \cdot \left(\sqrt{\frac{2}{\pi}}(1 - 2\betastar) \tr(\Amkinv) + 2c_1 ||\Amkinv||_{\mathsf{op}} \cdot n^{3/4}\right)}{1 + \lambda_{\truek}\left(\tr(\Amkinv) - c_1 ||\Amkinv||_{\mathsf{op}} \cdot n^{3/4}\right)} ,
\end{align*}
with probability at least $(1 - 3e^{-\sqrt{n}})$ over the randomness in $\{\ztk,\ytk\}$.
\noindent It remains to obtain high-probability bounds on the random quantities $\tr(\Amkinv)$ and $||\Amkinv||_{\mathsf{op}}$.
Note that we need both lower bounds and upper bounds on the quantity $\tr(\Amkinv)$, but we only need an upper bound on the quantity $||\Amkinv||_{\mathsf{op}}$.

\noindent We assume that we can choose $k \geq t$ such that $r_k(\Sigmabold) \geq bn$ and $r_k(\Sigmaboldtd) \geq b_2 n $ for universal positive constants $(b,b_2)$.
Consider any such choice of $k$ (which in general could depend on $(n,d)$).
First, we use Equation~\eqref{eq:Amiinv_eigbounds} from Lemma~\ref{lemma:eigboundsAmiinv} to upper bound the quantity $||\Amkinv||_{\mathsf{op}}$ as
\begin{align}
    ||\Amkinv||_{\mathsf{op}} = \mu_1(\Amkinv) &\leq \frac{c}{\lambda_{k+1} r_k(\Sigmabold)} \label{eq:2normAmkinvub}
\end{align}
with probability at least $(1 - e^{-\frac{n}{c}})$ over the random matrix $\A$.
Next, we turn to the quantity $\tr(\Amkinv)$.
To lower bound this quantity, we notice that
\begin{align*}
    \tr(\Amkinv) &= \sum_{j=1}^n \frac{1}{\mu_j(\Amk)} \\
    &\geq \sum_{j=k}^n \frac{1}{\mu_j(\Amk)} \\
    &\geq \frac{(n-k)}{\mu_{k+1}(\Amk)} .
\end{align*}
Now, from Equation~\eqref{eq:suproofeigvaluebound} in Lemma~\ref{lemma:eigboundsAmiinv} applied with $\Amk$, we have
\begin{align*}
    \mu_{k+1}(\Amk) \leq c\widetilde{\lambda}_{k+1} r_k(\Sigmaboldtd)
\end{align*}
with probability at least $(1 - e^{-\frac{n}{c}})$ provided that $r_k(\Sigmaboldtd) \geq b_2n$.
This gives us:
\begin{align}
    \tr(\Amkinv) &\geq \frac{(n-k)}{c\widetilde{\lambda}_{k+1} r_k(\Sigmaboldtd)} .\label{eq:traceAmkinvlb}
\end{align}
with probability at least $(1 - e^{-\frac{n}{c}})$.
On the other hand, the upper bound on the trace follows simply by
\begin{align}
    \tr(\Amkinv) &\leq \frac{n}{\mu_n(\Amk)}\nonumber \\
    &\leq \frac{cn}{\widetilde{\lambda}_{k+1} r_k(\Sigmaboldtd)} ,\label{eq:traceAmkinvub}
\end{align}
where the last inequality substitutes Equation~\eqref{eq:Amk_eigbounds_td}, which again holds with probability at least $(1 - e^{-\frac{n}{c}})$.
Noting that the upper bound on $\survival(\truek)$ is monotonically increasing in both $\tr(\Amkinv)$ and $||\Amkinv||_{\mathsf{op}}$, and the lower bound on $\survival(\truek)$ is monotonically increasing in $\tr(\Amkinv)$ but decreasing in $||\Amkinv||_{\mathsf{op}}$, we can substitute the above bounds on these quantities.
This completes our characterization of survival when binary labels are interpolated, with the probability of this characterization lower bounded by taking a union bound over the complement of all the above events.
After taking this union bound, the probability of each of the lower bound (Equation~\eqref{eq:survivallowerbound}) and upper bound (Equation~\eqref{eq:survivalupperbound}) holding is \textit{at least} $(1 - 3e^{-\sqrt{n}} - 2e^{-\frac{n}{c}})$.

\subsubsection{Interpolation of real output}

For completeness, we also include the proof of Theorem~\ref{theorem:survivalbounds} for the simpler case of interpolation of real-valued output (Equations~\eqref{eq:survivallowerboundreal} and~\eqref{eq:survivalupperboundreal}).
By the same standard argument, we can characterize the minimum-$\ell_2$-norm interpolator of real output as $\alphahatbold_{2,\mathsf{real}} = \Atrain\tran (\Atrain \Atrain\tran)^{-1} \Ztrain$.
By a similar argument to the case of binary labels, we have
\begin{align*}
    \survivalr(\truek) &= \sqrt{\lambda_{\truek}} \alphahat_{\truek}\\
    &= \sqrt{\lambdatk} \e_{\truek} \tran \Atrain\tran (\Atrain \Atrain\tran)^{-1} \Ztrain \\
    &= \lambda_{\truek}\ztk\tran \Ainv \ztk.
\end{align*}
Again, using the Sherman-Morrison-Woodbury identity, we have
\begin{align*}
    \Ainv  &= \Amkinv - \frac{\lambdatk \Amkinv \ztk \ztk \tran \Amkinv}{1 + \lambdatk \ztk \tran \Amkinv \ztk} ,
\end{align*}

which gives us
\begin{align}
    \survivalr(\truek) &=  \frac{\lambdatk \ztk \tran \Amkinv \ztk}{1 + \lambdatk \ztk \tran \Amkinv \ztk} \nonumber \\
    &= \frac{1}{1 + \frac{1}{\lambdatk \ztk \tran \Amkinv \ztk}} .  \label{eq:suformreal}
\end{align}

From Equations~\eqref{eq:zAzlb} and \eqref{eq:zAzub} above, the following statements each hold with probability at least $(1 - e^{-\sqrt{n}})$ over the randomness in $\ztk$ and for every realization of the random matrix $\Amkinv$:
\begin{align*}
    \ztk^\top \Amkinv \ztk &\geq \tr(\Amkinv) - c_2 ||\Amkinv||_{\mathsf{op}} \cdot n^{3/4} \text{ and } \\
    \ztk^\top \Amkinv \ztk &\leq \tr(\Amkinv) + c_2 ||\Amkinv||_{\mathsf{op}} \cdot n^{3/4} .
\end{align*}
Here, $c_2$ is a universal positive constant.

Observe that the right hand side of Equation~\eqref{eq:suformreal} is increasing in the quantity $\ztk^\top \Amkinv \ztk$.
Thus, substituting the lower bound for $\tr(\Amkinv)$ from Equation~\eqref{eq:traceAmkinvlb} and  the upper bound for $||\Amkinv||_{\mathsf{op}}$ from Equation~\eqref{eq:2normAmkinvub} lower bounds the quantity $\ztk^\top \Amkinv \ztk$, yielding the lower bound for $\survivalr(t)$.
Similarly, substituting the upper bound for $\tr(\Amkinv)$ from Equation~\eqref{eq:traceAmkinvub} and  the upper bound for $||\Amkinv||_{\mathsf{op}}$ from Equation~\eqref{eq:2normAmkinvub} upper bounds the quantity $\ztk^\top \Amkinv \ztk$, yielding the upper bound for $\survivalr(t)$.
This completes the proof of Theorem~\ref{theorem:survivalbounds}.
Again, a simple application of the union bound shows that each of the lower bound (Equation~\eqref{eq:survivallowerboundreal}) and the upper bound (Equation~\eqref{eq:survivalupperboundreal}) hold with probability at least $(1 - 2e^{-\sqrt{n}} - 2e^{-\frac{n}{c}})$.

\subsection{Proof of Theorem \ref{theorem:contaminationbounds}}\label{appendix:proofcontaminationbounds}

We next prove Theorem~\ref{theorem:contaminationbounds}, i.e.~upper and lower bounds on contamination, for the cases of interpolating binary labels and real output.
Since the contamination factor is intricately related to the contribution of additive noise to regression test error, the proof primarily consists of refinements of the arguments in~\citet{bartlett2020benign}.

\subsubsection{Interpolation of binary labels}
We start with a useful set of expressions for the contamination factor in the following lemma.
The proof of this lemma is contained in Appendix \ref{appendix:lemmaproof:contaminationformula}.

\begin{lemma}
The contamination of the minimum-$\ell_2$-norm interpolation of binary labels, denoted by $\alphahatbold_{2,\mathsf{binary}}$, can be written in the following two forms:
\begin{subequations}
\begin{align}
    \contamination(\truek) &= \sqrt{\ytk^\top \C \ytk}, \label{eq:contaminonform1}\\
    &=   \sqrt{\widetilde{\ytk} \tran \Ctd \widetilde{\ytk}} \label{eq:contaminonform2} ,
\end{align}
\end{subequations}
where we denote
\begin{align*}
\widetilde{\ytk} &:= \ytk - \survival(t) \ztk \text{ , } \\
    \C &:= \Ainv \left( \sumjdk \lambda_j^2 \z_j \z_j\tran \right)  \Ainv \text{ , and}\\
        \Ctd &:= \Amkinv \left( \sumjdk \lambda_j^2 \z_j \z_j\tran \right)  \Amkinv.
\end{align*}

\label{lemma:contaminationformula}
\end{lemma}
We will use the expression in Equation \eqref{eq:contaminonform2} to prove an upper bound on contamination, and the expression in Equation \eqref{eq:contaminonform1} for the lower bound.

\subsubsection{Upper bound on $\contamination(\truek)$}
We start with the proof for the upper bound on contamination for interpolation of binary labels (Equation~\eqref{eq:contaminatonupperboundsu}).
From Equation \eqref{eq:contaminonform2} in Lemma \ref{lemma:contaminationformula}, we have $\contamination^2(\truek) = \widetilde{\ytk} \tran \Ctd \widetilde{\ytk}
$.
Note that by construction, $\Ctd$ has no dependence on $\{\ztk, \ytk\}$ and thus $\Ctd \perp \widetilde{\ytk}$.
The next lemma upper bounds the term $ \widetilde{\ytk} \tran \Ctd \widetilde{\ytk}$ in terms of $\tr(\Ctd)$ and is proved in Appendix \ref{appendix:lemmaproof:epsilontrace}.

\begin{lemma}
There exists universal positive constant $c_6$ such that when $n\geq c_6$, we have
\begin{align*}
    \widetilde{\ytk} \tran  \Ctd \widetilde{\ytk} \leq  2 \left(1 + \frac{1}{c}\right) \cdot (1 + \survival(\truek)^2) \cdot \tr(\Ctd) \cdot \ln n
\end{align*}
almost surely for every realization of the random matrix $\Ctd$, and with probability at least $\left(1 - \frac{2}{n}\right)$ over the randomness in $\widetilde{\ytk}$.
\label{lemma:epsilontrace}
\end{lemma}
Applying Lemma~\ref{lemma:epsilontrace}, we get
\begin{align}
 \contamination^2(\truek) &\leq 2 \left(1 + \frac{1}{c}\right) \cdot \tr(\Ctd) \cdot \ln n \label{eq:cnproofpart2}
\end{align}
almost surely for every realization of the random matrix $\Ctd$, and with probability at least $\left(1 - \frac{2}{n}\right)$ over the randomness in $\widetilde{\ytk}$.
The next lemma, which is taken from~\citet{bartlett2020benign}, provides a high-probability upper bound on the quantity $\tr(\Ctd)$.

\begin{lemma} (From Lemma 11 in \citet{bartlett2020benign})
There exist universal constants $(b_2,c_5, c_{10} \geq 1)$ such that whenever $0 \leq k \leq n/c_5$ and $r_k(\Sigmaboldtd) \geq b_2n$, we have
\begin{align*}
\tr(\Ctd) \leq c_{10} \cdot \left( \frac{l}{n} + n  \cdot \frac{\sum_{j>l } \lambdatd_j^2}{\left(\sum_{j>k}\lambdatd_j\right)^2} \right)
\end{align*}
for any choice of $l \leq k$, with probability at least $(1 - 6 e^{-\frac{n}{c_5}})$ over the randomness in $\Ctd$.
\label{lemma:traceupperbound}
\end{lemma}
Substituting the upper bound from Lemmas \ref{lemma:traceupperbound} and  into Equation \eqref{eq:cnproofpart2}, and taking the square root on both sides, we have
\begin{align*}
       \contamination(\truek) \leq  \sqrt{2 \left(1 + \frac{1}{c}\right) \cdot  c_{10} \cdot \left(\frac{l}{n} + n\cdot \frac{\sum_{j>l} \lambdatd_j^2}{\left(\sum_{j>k}\lambdatd_j\right)^2}  \right)\cdot \left(1 + \survival(\truek)^2 \right) \cdot \ln n}.
\end{align*}
with probability at least $\left(1 - \frac{2}{n} - 6e^{-\frac{n}{c_2}}\right)$ over the training data.
Taking $c_7 = \sqrt{2 \left(1 + \frac{1}{c}\right) c_{10}}$, the upper bound on $\contamination(\truek)$ in Equation~\eqref{eq:contaminatonupperboundsu} follows.
Noting that~$\left(1 - \frac{2}{n} - 6e^{-\frac{n}{c_2}}\right) \geq \left(1 - \frac{3}{n}\right)$ for large enough $n$, this completes the proof of the upper bound.

\subsubsection{Lower bound on $\contamination(\truek)$}
Now we move on to the proof for the lower bound on contamination for interpolation of binary labels (Equation~\eqref{eq:contaminatorlowerbound}). Using Equation \eqref{eq:contaminonform1} from  Lemma \ref{lemma:contaminationformula}, we get
\begin{align*}
    \contamination^2(\truek) &= \ytk \tran \C \ytk \\
    & \geq \eig_n(\C) \norm{\ytk}_2^2 = n \eig_n(\C) .
\end{align*}
The next lemma lower bounds the minimum eigenvalue of $\C$ and is proved in Appendix \ref{appendix:lemmaproofs:eigClowerbound}.

\begin{lemma}
Let $k \geq 0$ and $r_k\left(\Sigmaboldtd^2\right) \geq b_4n$.
Then, we have
 \begin{align*}
     \eig_n(\C) \geq \frac{r_k\left(\Sigmaboldtd^2\right) \cdot\lambdatd_{k+1}^2}{c_{11} \cdot c^2\cdot\left(\sumjd \lambda_j + \lambda_1 n \right)^2} \end{align*}
with probability at least $(1 - e^{-\frac{n}{c}} - e^{-\frac{n}{c_{11}}})$.= over the randomness in $\C$.
Here, $(b_4,c,c_{11})$ are universal positive constants.
\label{lemma:eigClowerbound}
\end{lemma}
A direct substitution of the above gives us
\begin{align*}
    \contamination(\truek) \geq \sqrt{n} \cdot \frac{\sqrt{r_k\left(\Sigmaboldtd^2\right) \cdot \lambdatd_{k+1}^2}}{c \cdot \sqrt{c_{11}} \cdot \left(\sum\limits_{j=1}^d \lambda_j + \lambda_1 n\right)}
\end{align*}
with probability at least $(1 - e^{-\frac{n}{c}} -e^{-\frac{n}{c_{11}}})$ over the training data.
Taking $c_9 = c\sqrt{c_{11}}$ and $c_8$ such that $\frac{1}{c_8} = \min( \frac{1}{c},\frac{1}{c_{11}})$ holds, the lower bound in Equation  \eqref{eq:contaminatorlowerbound} follows. This completes the characterization of the contamination factor  when we interpolate binary labels.

\subsubsection{Interpolation of real output}\label{appendix:reallabelsurvival}

For completeness, we also provide the proof of Theorem~\ref{theorem:contaminationbounds} for the simpler case of interpolation of real output.
We start with a useful set of expressions for the contamination factor in the following lemma.
The proof of this lemma is contained in Appendix \ref{appendix:lemmaproof:contaminationformula}.

\begin{lemma}
The contamination of the minimum-$\ell_2$-norm interpolator of binary labels, denoted by $\alphahatbold_{2,\mathsf{real}}$, can be written in the following two forms:
\begin{subequations}
\begin{align}
    \contaminationr(\truek) &= \sqrt{\ztk \tran \C \ztk}, \label{eq:contaminonform1real}\\
    &=  \abs{1 - \survivalr(\truek)} \sqrt{\ztk \tran \Ctd \ztk} \label{eq:contaminonform2real},
\end{align}
\end{subequations}
where we denote
\begin{align*}
    \C &= \Ainv \left( \sumjdk \lambda_j^2 \z_j \z_j\tran \right)  \Ainv \text{ , and } \\
        \Ctd &= \Amkinv \left( \sumjdk \lambda_j^2 \z_j \z_j\tran \right)  \Amkinv.
\end{align*}

\label{lemma:contaminationformulareal}
\end{lemma}
We will use the form in Equation \eqref{eq:contaminonform2real} to prove an upper bound on contamination and the form in Equation \eqref{eq:contaminonform1real} for the lower bound.

\subsubsection{Upper bound on $\contaminationr(\truek)$}
We start with the proof for the upper bound on contamination for interpolation of real output (Equation~\eqref{eq:contaminatonupperboundsureal}.
From Equation \eqref{eq:contaminonform1real} in Lemma \ref{lemma:contaminationformulareal}, we get
\begin{align}
    \contaminationr^2(\truek) &=(1 - \survivalr(\truek))^2 \ztk \tran \Ctd \ztk. \label{eq:contreal1}
\end{align}
From Equation~\eqref{eq:zCtdzub} in Appendix \ref{appendix:lemmaproof:epsilontrace} (proof of Lemma~\ref{lemma:epsilontrace}), we can upper bound the quadratic form $\ztk \tran \Ctd \ztk $ as
\begin{align*}
    \ztk \tran \Ctd \ztk \leq  7\tr(\Ctd) \ln n
\end{align*}
with probability at least $\left(1 - \frac{1}{n}\right)$ over the randomness in $\ztk$.
Then, substituting the upper bound on $\tr(\Ctd)$ from Lemma \ref{lemma:traceupperbound} directly gives us the expression for the upper bound on $\contaminationr(t)$.
Noting again that $\left(1 - \frac{1}{n} - 6e^{-\frac{n}{c_2}}\right) \geq \left(1 - \frac{2}{n}\right)$ for large enough $n$, this completes the proof for the upper bound.

\subsubsection{Lower bound}
We conclude this section by proving the lower bound on contamination for interpolation of real output (Equation~\eqref{eq:contaminatorlowerboundreal}).
We directly apply Equation \eqref{eq:contaminonform1real} (from Lemma \ref{lemma:contaminationformulareal}) to get
\begin{align*}
    \contaminationr^2(\truek) &= \ztk \tran \C \ztk \\
    & \geq \eig_n(\C) \norm{\ztk}_2^2 \\
    & \stackrel{\1}{\geq} n(1 - \delta) \eig_n(\C)
\end{align*}
with probability at least $(1 - e^{-n \delta^2})$ over the randomness in $\ztk$ for any $\delta \in (0,1)$. Here, inequality $\1$ follows from the lower bound in Lemma~\ref{lemma:chisquaredtailbound}.
Finally, substituting the lower bound for $\eig_n(\C)$ from Lemma~\ref{lemma:eigClowerbound} gives us the desired expression for the lower bound on $\contaminationr(\truek)$.
Note that by the union bound, this expression will hold with probability at least $(1 - e^{-n \delta^2} - e^{-\frac{n}{c}} - e^{-\frac{n}{c_{11}}}) = (1 - 2e^{-\frac{n}{c_8}} - e^{-n \delta^2})$ over the randomness in the training data.
This completes the proof of~Theorem \ref{theorem:contaminationbounds}.
\qed

\section{Implications for bi-level covariance: Proof of Theorem \ref{theorem:threeregimes}}
\label{appendix:bilevel}

In this section, we follow the \textit{path to analysis} described in Section~\ref{sec:pathtoanalysis} and prove Theorem~\ref{theorem:threeregimes} for the bi-level ensemble (Definition~\ref{def:bilevel_covariance}) in the following series of steps:
\begin{enumerate}
    \item We substitute the spectrum of the bi-level ensemble into Theorems~\ref{theorem:survivalbounds} and~\ref{theorem:contaminationbounds} to get asymptotic expressions for survival and contamination.
    \item We substitute these expressions into the expressions for regression and classification test loss (Proposition~\ref{theorem:classloss}) to characterize the regimes for good generalization of classification and regression.
\end{enumerate}
For convenience of notation, we consider $\truek = 1$.
(Note, however, that the analysis holds for any $1 \leq \truek \leq s$ since the first $s$ eigenvalues of $\Sigmabold$ are equal.)
Further, to emphasize that the survival and contamination quantities depend on $n$, in this section we refer to them as $\survivaln, \contaminationn, \survivalnreal, \text{ and } \contaminationnreal$ for interpolators of binary and real output respectively.

First, we characterize some useful quantities for the bi-level ensemble.
Recall that the bi-level ensemble is parameterized by $p > 1$, $0 < q \leq (p-r)$ and $0 < r \leq 1$.
We first compute the effective ranks $r_k(\Sigmabold)$ and $r_k(\Sigmaboldtd)$ for two choices of $k$.
First, we have
\begin{align*}
     r_s(\Sigmabold) = \frac{1}{\frac{(1-a) d}{d-s}} .\frac{(1-a)d }{d-s}.(d-s) = d - s .
\end{align*}
Substituting $d = n^p$ and $s = n^r$, we have, for sufficiently large $n$,
\begin{align}\label{eq:rk}
    r_s(\Sigmabold) \asymp n^p \gg n .
\end{align}
Similarly because $1 \leq \truek \leq s$, we have, for sufficiently large $n$,
\begin{align}
r_{s}(\Sigmaboldtd) = d-s-1 \asymp n^{p} \gg n. \label{eq:rktd}
\end{align}
Moreover, we get
\begin{align}\label{eq:rk2}
    r_0(\Sigmabold) &= \frac{1}{\frac{ad}{s}} \cdot d = \frac{s}{a} = n^{q + r}
    \gg n \text{ iff } (q + r ) > 1 .
\end{align}
and by a similar argument, provided that $r > 0$, we can show that (for large enough $n$),
\begin{align}\label{eq:rktd2}
    r_0(\Sigmaboldtd) = \frac{1}{\frac{ad}{s}} \cdot \left(d - \frac{ad}{s}\right) = \frac{s}{a} - 1 = n^{q + r} - 1 \gg n \text{ iff } (q + r) > 1.
\end{align}
We will apply Equations~\eqref{eq:rk} and~\eqref{eq:rktd} for bounding survival in general, as well as contamination when we have $q \leq (1 - r)$, and Equations~\eqref{eq:rk2} and~\eqref{eq:rktd2} for bounding contamination when we have $q > (1 - r)$.
Now, we state and prove our matching upper and lower bounds for survival for the bi-level ensemble.
\begin{lemma}[Survival for interpolation of binary labels]
There exist universal positive constants $(L_1, U_1, L_2, U_2)$ such that for sufficiently large $n$, we have
\begin{align*}
    \survivallowerbound \leq \survivaln \leq \survivalupperbound,
\end{align*}
with probability at least $(1 - 10e^{-\sqrt{n}})$ over the training data $\{X_i,Y_i\}_{i=1}^n$, where we denote
\begin{subequations}
\begin{align}
        \survivallowerbound &:= \begin{cases} \sqrt{\frac{2}{\pi}}(1 - 2\betastar) \left(1 + L_1 n^{q - (1 - r)}\right)^{-1}, & q < (1 - r)\\
    \sqrt{\frac{2}{\pi}}(1 - 2\betastar) \cdot L_2n^{(1 - r) - q}, & q > (1 - r)
    \end{cases}, \label{eq:slsl2}\\
    \survivalupperbound &:= \begin{cases}\sqrt{\frac{2}{\pi}}(1 - 2\betastar) \left(1 + U_1 n^{q - (1 - r)}\right)^{-1}, & q < (1 - r)\\
    \sqrt{\frac{2}{\pi}}(1 - 2\betastar) \cdot U_2n^{(1 - r) - q}, & q > (1 - r)
    \end{cases} \label{eq:susl2}.
\end{align}
\end{subequations}
\label{lemma:survivalspikedl2}
\end{lemma}

\begin{proof}
Note that Equations~\eqref{eq:rk} and~\eqref{eq:rktd} imply that the conditions $r_s(\Sigmabold) \geq bn$ and $r_s(\Sigmaboldtd) \geq b_2 n$ are clearly satisfied for large enough $n$.
Thus, we can apply Equation~\eqref{eq:survivallowerbound} of Theorem \ref{theorem:survivalbounds} setting $k = s$ to get
\begin{align*}
\survivaln \geq  \sqrt{\frac{2}{\pi}}(1 - 2\betastar)\frac{\lambda_1\left(\frac{(n-s)}{\widetilde{c\lambda}_{s+1} r_s(\Sigmabold_{-1})} - \frac{c_3 n^{3/4}}{\lambda_{s+1} r_s(\Sigmabold)} \right)}{1 + \lambda_1\left(\frac{cn}{\widetilde{\lambda}_{s+1} r_s(\Sigmabold_{-1})} + \frac{c_4 n^{3/4}}{\lambda_{s+1} r_s(\Sigmabold)} \right)}
\end{align*}
with probability at least $(1 - 5e^{-\sqrt{n}})$ over the training data.
Substituting $s = n^r$ and $a = n^{-q}$, note that
\begin{align*}
   \frac{\lambda_{s+1}r_s(\Sigmabold)}{\lambda_1} = \frac{\widetilde{\lambda}_{s+1}r_s(\Sigmabold_{-1})}{\lambda_1}\asymp \frac{\lambdasmall n^p}{\lambdabig} \asymp \frac{n^{p + r}}{n^{p-q}}\asymp  n^{q+r}.
\end{align*}
Substituting this above yields
\begin{align*}
    \survivaln &\geq \sqrt{\frac{2}{\pi}}(1 - 2\betastar)\left( \frac{\frac{(n-n^r)}{cn^{q + r}} - \frac{c_3 n^{3/4}}{n^{q + r}}}{1 + \frac{cn}{n^{q + r}} +\frac{c_4 n^{3/4}}{n^{q + r}} }\right) \\
    &= \sqrt{\frac{2}{\pi}} (1 - 2\betastar) \left(\frac{\frac{1}{c} \cdot (n^{(1-r) - q} - n^{-q}) - c_3 \cdot n^{(3/4 - r) - q}}{1 + cn^{(1 - r) - q} + c_4 \cdot n^{(3/4 - r) - q}}\right) .
\end{align*}
Thus, there are two cases:
\begin{enumerate}
    \item $0 < q \leq (1-r)$, in which case the terms corresponding to $n^{q - (1-r)}$ dominate, and there exists universal constant $L_1$ such that
    \begin{align*}
        \survivaln \geq \sqrt{\frac{2}{\pi}}(1 - 2\betastar) \left(1 + L_1 n^{q - (1-r)}\right)^{-1} .
    \end{align*}
    \item $q > (1-r)$, in which case the numerator goes to $0$ but the denominator goes to $1$ as $n \to \infty$, and so there exists universal constant $L_2$ such that
    \begin{align*}
        \survivaln \geq \sqrt{\frac{2}{\pi}}(1 - 2\betastar) \cdot L_2 n^{(1-r) - q} .
    \end{align*}
\end{enumerate}
This completes the proof of the lower bound.
An almost identical argument gives the proof of the upper bound, so we omit it here.
\end{proof}
Observe that for $q > (1-r)$, the true signal does not survive at all, i.e.~$\survivalnreal \to 0$ as $n \to \infty$.
Interestingly, for $q \leq (1-r)$, there is also non-trivial attenuation of signal when binary labels are interpolated, i.e.~$\survivalnreal \to \sqrt{\frac{2}{\pi}} \cdot (1 - 2\betastar) < 1$ as $n \to \infty$.
At a high level, this is a consequence of effective misspecification induced by the sign operator on real output.
As mentioned in the discussion in Section~\ref{sec:pathtoanalysis}, this is also spiritually related to the attenuation factor of signal that has been traditionally been observed as a result of $1$-bit quantization applied to a matched filter~\citep{1454543,chang1982presampling}.

As we will see in the following lemma, the corresponding case leads to zero attenuation of signal when real output is interpolated., i.e.~$\survivalnreal \to 1$.

\begin{lemma}[Survival for interpolation of real output]
There exist universal positive constants $(L_1, U_1, L_2, U_2, \widebar{L_1}, \widebar{U_1}, \widebar{L_2}, \widebar{U_2})$ such that for sufficiently large $n$, we have
\begin{align*}
    \survivallowerboundreal \leq \survivalnreal \leq \survivalupperboundreal,
\end{align*}
with probability at least $(1 - 8e^{-\sqrt{n}})$ over the randomness in the training data $\{X_i,Y_i\}_{i=1}^n$, where we denote
\begin{subequations}
\begin{align}
        \survivallowerboundreal &:= \begin{cases} \left(1 + L_1 n^{q - (1-r)}\right)^{-1}, & q < (1-r)\\
    L_2n^{(1-r) - q}, & q > (1-r)
    \end{cases}, \label{eq:slsl2-real}\\
    \survivalupperboundreal &:= \begin{cases}\left(1 + U_1 n^{q - (1-r)}\right)^{-1}, & q < (1-r)\\
    U_2n^{(1-r) - q}, & q > (1-r)
    \end{cases} \label{eq:susl2-real}.
\end{align}
\end{subequations}
Equivalently, we can write
\begin{align*}
   \oneminussurvivallowerboundreal \leq 1 - \survivalnreal \leq  \oneminussurvivalupperboundreal,
\end{align*}
where we denote
\begin{subequations}
\begin{align}
    \oneminussurvivallowerboundreal &:=\begin{cases} \widebar{L_1}n^{q-(1-r)}, & q < (1-r) \\
    \left(1 + \widebar{L_2} n^{ (1-r)-q}\right)^{-1}, & q > (1-r)
    \end{cases} \label{eq:slbarsl2real}, \\
        \oneminussurvivalupperboundreal &:= \begin{cases} \widebar{U_1}n^{q-(1-r)}, & q < (1-r) \\
     \left(1 + \widebar{U_2} n^{ (1-r)-q}\right)^{-1}, & q > (1-r)
    \end{cases} \label{eq:subarsl2real}
\end{align}
\end{subequations}
\label{lemma:survivalspikedl2real}
\end{lemma}

\begin{proof}
The proof follows by substituting the spectrum of the bi-level covariance model into the upper and lower bounds of survival from Equations~\eqref{eq:survivalupperboundreal} and \eqref{eq:survivallowerboundreal}.
This is essentially an identical argument to the proof of Lemma~\ref{lemma:survivalspikedl2}, and so we omit it here.
\end{proof}
Observe that for the case of interpolation of real output, we have additionally computed bounds on the quantity $(1- \survivalnreal)$, which will subsequently be useful for the computation of bounds on contamination. We have not stated this here to avoid complicating the proof, but it is interesting to note that if the real-valued output had a non-zero level of independent additive zero-mean Gaussian noise, then this would not matter for the scaling of the survival results asymptotically --- this is a consequence of the range of parameter choices that we have chosen for our bi-level ensemble. Such label noise would effectively be completely absorbed by the excess features.

We now state an upper bound on contamination for the bi-level ensemble.
\begin{lemma}[Contamination for interpolation of binary labels]
There are universal positive constants $(U_3,U_4$ and $U_5)$ such that for large enough $n$, we have $\contaminationn \leq \contaminationupperbound$  with probability at least $\left(1 - \frac{4}{n}\right)$ over the randomness in the training data $\{X_i,Y_i\}_{i=1}^n$, where we denote
\begin{align}
    \contaminationupperbound =  \begin{cases}U_3 n^{\frac{- \min\{(p - 1), (1 - r)\}}{2}} \cdot \sqrt{\ln n} \text{ if } q < (1 - r) \\
    U_4 n^{\frac{- \min\{(p - 1), (2q + r - 1)\}}{2}} \cdot \sqrt{\ln n} \text{ if } q > (1 - r)
    \end{cases}
    \label{eq:cusl2}
\end{align}
\label{lemma:contaminationupperboundspikedl2}
\end{lemma}
\begin{proof}
We start by proving the statement for the case $q \leq (1 - r)$.
From Equations \eqref{eq:rk} and \eqref{eq:rktd}, we showed that for large enough $n$, we have $r_{s}(\Sigmabold_{-1}) \asymp n^p \gg n$. Substituting $k = l = s$ in Equation~\eqref{eq:contaminatonupperboundsu} from Theorem \ref{theorem:contaminationbounds}, we have
\begin{align}
        \contaminationn \leq c_7 \cdot \sqrt{\left(\frac{s}{n} + n \cdot \frac{\sum_{j> s} \lambdatd_j^2}{\left(\sum_{j>s}\lambdatd_j \right)^2} \right) \cdot \ln n \cdot (1 + \survivaln^2)}
        \label{eq:contaminationupperboundspikedl2}
\end{align}
almost surely for every realization of $\survivalb$ with probability at least $\left(1 - \frac{3}{n}\right)$ over the training data.
We first evaluate the term
\begin{align*}
    T_1  &:= \frac{s}{n} + n \cdot \frac{\sum_{j>s} \lambdatd_j^2}{\left(\sum_{j>s}\lambdatd_j\right)^2} .
\end{align*}
First, note that
\begin{align*}
    \sum_{j>s} \lambdatd_j^2 &= (d-s-1)\left(\lambdasmall\right)^2 \asymp d = n^p \text{ and } \\
    \left(\sum_{j>s}\lambdatd_j\right)^2 &= \left((d-s-1)\lambdasmall\right)^2 \asymp n^{2p}.
\end{align*}
Using this, we obtain
\begin{align}
    T_1 \asymp n^{(r - 1)} + n^{(1 - p)} \asymp n^{-\min\{(p - 1), (1 - r)\}} . \label{eq:T1}
\end{align}
Now, from Equation~\eqref{eq:susl2}, we get (for large enough $n$)
\begin{align}
    \survivaln \leq \mathds{1}_{q\leq (1-r)} \sqrt{\frac{2}{\pi}} \left(1 + U_1 n^{q - (1-r)}\right)^{-1} + \mathds{1}_{q > (1-r)}
    U_2n^{(1-r) - q} \leq \max\left\{U_2,\sqrt{\frac{2}{\pi}}\right\} \label{eq:subl}
\end{align}
with probability at least $(1 - 4e^{-p_1n})$ over the training data.
Substituting Equations~\eqref{eq:T1} and \eqref{eq:subl} in Equation~\eqref{eq:contaminationupperboundspikedl2}, we have
\begin{align*}
    \contaminationn \leq U_3 n^{-\frac{\min\{(p - 1), (1 - r)\}}{2}} \cdot \sqrt{\ln n}
\end{align*}
with probability at least $\left(1 - \frac{4}{n}\right)$ for appropriately defined positive constant $U_3$. This completes the proof for the first case.

\noindent Now, we move on to the second case, i.e.~$q > ( 1- r)$.
From Equations~\eqref{eq:rk2} and~\eqref{eq:rktd2}, we saw that in this case, we have $r_0(\Sigmabold_{-1}) \asymp n^{q + r} \gg n$.
Substituting $k = l = 0$ in Equation~\eqref{eq:contaminatonupperboundsu} from Theorem \ref{theorem:contaminationbounds}, we have
\begin{align*}
        \contaminationn \leq c_7 \cdot \sqrt{\left(n \cdot \frac{\sum_{j> 0} \lambdatd_j^2}{\left(\sum_{j>0}\lambdatd_j \right)^2} \right) \cdot \ln n \cdot (1 + \survivaln^2)}
\end{align*}
with probability at least $\left(1 - \frac{3}{n}\right)$ over the training data. As before, we evaluate the term
\begin{align*}
    T_1 := n \cdot \frac{\sum_{j > 0} \lambdatd_j^2}{(\sum_{j > 0} \lambdatd_j)^2}
\end{align*}
By a calculation very similar to the one in Appendix~\ref{appendix:svm_bilevel}, we get
\begin{align*}
    \sum_{j > 0} \lambdatd_j^2 = (s-1) \cdot \frac{a^2 d^2}{s^2} + (d - s - 1) \cdot \frac{(1 - a)^2 d^2}{(d-s)^2} \asymp n^{2p + 2q - r} + n^p .
\end{align*}
Moreover, we get $(\sum_{j > 0} \lambdatd_j)^2 = (d - \frac{ad}{s})^2 = (n^p - n^{p - (r + q)})^2 \asymp n^{2p}$ since $(q + r) > 0$.
Therefore, we get
\begin{align*}
    T_1 \asymp n^{(1-p)} + n^{(1 + 2q - r)} \asymp n^{-\min\{(p - 1), (2q + r - 1)\}} .
\end{align*}
The other steps proceed as for the first case, and substituting this expression for the term $T_1$ completes the proof for the second case.
\end{proof}
For some parameterizations of the bi-level ensemble, we can get a slightly more sophisticated upper bound on contamination when the labels interpolated are real, as detailed in the following lemma.
\begin{lemma}[Contamination for interpolation of real output] For universal positive constants $(U_3,U_4,U_5)$ and large enough $n$, we have $\contaminationnreal \leq \contaminationupperboundreal$ with probability at least $\left(1 - \frac{3}{n}\right)$ over the randomness in the training data $\{X_i,Y_i\}_{i=1}^n$, where we denote
\begin{align} \contaminationupperboundreal =   \begin{cases}
    U_3 n^{q - (1 - r) - \frac{\min\{(p - 1), (1 - r)\}}{2}} \cdot \sqrt{\ln n}, & q < (1-r),\\
    U_4 n^{-\frac{\min\{(p - 1), (2q + r - 1)\}}{2}} \cdot \sqrt{\ln n}, & q > (1-r)\end{cases}.
    \label{eq:cusl2real}
\end{align}
\label{lemma:contaminationupperboundspikedl2real}
\end{lemma}
\begin{proof}
We follow an identical approach as in the proof of
 Lemma~\ref{lemma:contaminationupperboundspikedl2} to bound the term $T_1$.
Substituting this along with the upper bound on the quantity $(1 - \survivalnreal)$ from Equation~\eqref{eq:subarsl2real} (Lemma~\ref{lemma:survivalspikedl2real})  in Equation~\eqref{eq:contaminatonupperboundsureal}, and using the fact that $\survivalnreal \leq 1$, Equation~\eqref{eq:cusl2real} follows for appropriately defined positive constants $(U_3,U_4)$.  This completes the proof.
\end{proof}
Finally, we state and prove our lower bounds on contamination together for interpolation of binary labels as well as real output.
\begin{lemma}[Lower bounds on contamination]
There are universal positive constants $(L_3,L_4, p_2)$ such that for large enough $n$, we have   $\contaminationn, \contaminationnreal \geq \mathsf{CN}^L(n)$ with probability at least $(1 - 2e^{-p_2n})$ over the randomness in the training data $\{X_i,Y_i\}_{i=1}^n$, where we define
\begin{align} \mathsf{CN}^L(n) :=  \begin{cases}
       L_3 n^{q - (1 - r) -\frac{p-1}{2}}, & q < (1-r)\\
       L_4 \npm{(p-1)}{2}, & q > (1-r)
       \end{cases}. \label{eq:clsl2}
\end{align}
\label{lemma:contaminationlowerboundspikedl2}
\end{lemma}

\begin{proof}
Using Equation \eqref{eq:rktd} we have, for large enough $n$, $r_{s}\left(\Sigmabold_{-1}^2\right)  \asymp n^p \gg n$. Taking $k = s$  in Equation \eqref{eq:contaminatorlowerbound} from  Theorem \ref{theorem:contaminationbounds}, for universal constants $c_8, c_9$, with probability at least $(1 - 2e^{-\frac{n}{c_8}})$, we have
\begin{align*}
       \contaminationn & \geq  \sqrt{n}.\frac{\sqrt{r_{s}\left(\Sigmabold_{-1}^2\right) \lambdatd_{s+1}^2 }}{c_9\left(\sum\limits_{j=1}^d \lambda_j + \lambda_1 n\right)} \\
       &\asymp \np{1}{2} \cdot \frac{\sqrt{n^p \left(\lambdasmall \right)^2}}{d + n\lambdabig} \\
       &\asymp  \frac{n^{-\frac{(p-1)}{2}}}{1 + n^{(1-r) - q}},\\
       &\asymp \begin{cases}
       n^{q - (1-r) - \frac{(p-1)}{2}}, & q < (1-r)\\
        \npm{(p-1)}{2}, & q > (1-r)
       \end{cases}.
\end{align*}
Thus Equation~\eqref{eq:cusl2} follows by choosing appropriate constants $p_2, L_3$ and $L_4$, completing the proof.

\end{proof}
Comparing the upper bound (Equation~\eqref{eq:cusl2real}) and lower bound (Equation~\eqref{eq:clsl2}) for the case of interpolating real output, we observe that these bounds would be matching up to constant factors \textit{iff} $(p-1) \leq (1 - r)$.
In addition to the above condition, the upper bound for interpolation of binary labels (Equation~\eqref{eq:contaminationupperboundspikedl2}) will match the lower bound \textit{iff} $q > (1-r)$.

Finally, we compute bounds on the ratio of survival to contamination, \newline $\survivaln/\contaminationn$, for the interpolation of binary labels.
A directly substitution of the upper and lower bounds for $\survivaln$ and $\contaminationn$ from  Equations \eqref{eq:slsl2}, \eqref{eq:susl2} in Lemma \ref{lemma:survivalspikedl2}, Equations \eqref{eq:cusl2} in  Lemma \ref{lemma:contaminationupperboundspikedl2} and Equation  \eqref{eq:clsl2} in Lemma \ref{lemma:contaminationlowerboundspikedl2}, gives us (for large enough $n$)
\begin{align}
    \SUClower \leq \frac{\survivaln}{\contaminationn} \leq \SUCupper, \label{eq:snrbounds}
\end{align}
with probability at least $\left(1 - \frac{16}{n}\right)$ over the training data, where we denote
\begin{subequations}
\begin{align}
    \SUClower &:= \begin{cases}
    L_5 \cdot n^{\frac{\min\{(p-1),(1-r)\}}{2}} \cdot (\ln n)^{-\frac{1}{2}}, & 0 < q < (1-r)\\ 
    L_6 \cdot n^{\frac{\min\{(p-1),(2q + r - 1)\}}{2} + (1-r) - q}\cdot (\ln n)^{-\frac{1}{2}}, & q > (1-r)  \end{cases}.\label{eq:suclower}\\
    \SUCupper &= U_5 \cdot n^{\frac{p-1}{2} +(1 - r) - q}.\label{eq:sucupper}
\end{align}
\end{subequations}

\subsection{Proof of Theorem \ref{theorem:threeregimes} }

We are now ready to complete the proof of Theorem \ref{theorem:threeregimes}.
First we compute a lower bound on regression test loss. From Equations \eqref{eq:reglosssuc}, \eqref{eq:slbarsl2real} and \eqref{eq:clsl2}, we have (for large enough $n$)
\begin{align*}
     \reglossn &= (1 - \survivalnreal)^2 + (\contaminationnreal)^2\\
     &\geq (\oneminussurvivallowerboundreal)^2 + (\contaminationlowerboundreal)^2 \\
     &=\begin{cases}\widebar{L_1}^2 n^{ 2(q-(1-r))} + L_3^2 n^{-2(1-r) - (p-1) + 2q}, & q < (1-r)\\   \left(1 + \widebar{L_2} n^{ (1-r)-q}\right)^{-2} + L_4^2 n^{-(p-1)}, &  q > (1-r) \end{cases}\\
\end{align*}
with probability at least $(1 - 2e^{-\sqrt{n}} - 2e^{-p_2n})$.
Thus, we have
\begin{align*}
    \liminf\limits_{n \rightarrow \infty} \reglossn \geq \begin{cases} 0, & q < (1-r)\\ 1, & q > (1-r) \end{cases}
\end{align*}
with probability equal to $1$.
Next, we compute an upper bound on regression test loss. From Equations \eqref{eq:reglosssuc}, \eqref{eq:subarsl2real} and \eqref{eq:cusl2real}, we have (for large enough $n$)
\begin{align*}
     \reglossn
     &\leq (\oneminussurvivalupperboundreal)^2 + (\contaminationupperboundreal)^2 \\
     &= \begin{cases}\widebar{U_1}^2 n^{2(q-(1-r))} + U_3^2 n^{-2(1-r) - \min\{(p-1),(1-r)\} + 2q}\ln n, & q < (1-r)\\ 
     \left(1 + \widebar{U_2} n^{ (1-r)-q}\right)^{-2} + U_4^2 n^{-(p-1)}\ln n, & q > (1-r) \end{cases}
\end{align*}
with probability at least $\left(1 -  2e^{-\sqrt{n}} - \frac{3}{n}\right)$.
Thus, we have
\begin{align*}
   \operatorname{\lim \sup}_n \reglossn \leq \begin{cases} 0, & q < (1-r)\\ 1, & q > (1-r) \end{cases}
\end{align*}
with probability equal to $1$.
By the sandwich theorem, we get
\begin{align*}
    \lim_{n \rightarrow \infty} \reglossn = \begin{cases} 0, & q < (1-r) \\ 1, & q > (1-r)  \end{cases}
\end{align*}
with probability $1$, completing our characterization of regression.

\noindent We now move on to our final characterization of classification test loss, starting with the upper bound.
By Proposition~\ref{theorem:classloss}, we have
\begin{align*}
   \classlossn &=  \frac{1}{2} - \frac{1}{\pi}\taninv{\frac{\survivaln}{\contaminationn}}.
\end{align*}
From Equation \eqref{eq:snrbounds}, we get
\begin{align*}
      \frac{1}{2} - \frac{1}{\pi}\taninv{\SUCupper} \leq \classlossn \leq \frac{1}{2} - \frac{1}{\pi}\taninv{\SUClower}.
\end{align*}
Taking the limit as $n \to \infty$ in Equation~\eqref{eq:suclower}, we have
\begin{align*}
    {\lim \inf}_{n\rightarrow\infty} \SUClower  = \begin{cases}
    \infty, & q < \frac{\min\{(p-1),(2q + r - 1)\}}{2} + (1-r) \\
    0, & q > \frac{\min\{(p-1),(2q + r - 1)\}}{2} + (1-r)
    \end{cases}.
\end{align*}
with probability $1$.
Thus, we have
\begin{align*}
    \operatorname{\lim \sup}_n \classlossn \leq \begin{cases}
    0, &  q < \frac{\min\{(p-1),(2q + r - 1)\}}{2} + (1-r)\\ 
    \frac{1}{2}, & q >  \frac{\min\{(p-1),(2q + r - 1)\}}{2} + (1-r)
    \end{cases}.
\end{align*}
with probability $1$.
To simplify further, consider the case for which $(2q + r - 1) < (p - 1)$.
Then, the condition becomes $q < q + \frac{(r-1)}{2} + (1-r) = \frac{(1-r)}{2} \implies \frac{(1-r)}{2} > 0$, which is always true under the bi-level ensemble (as $r < 1$).
Thus, we can effectively ignore this argument, and simply write
\begin{align*}
    \operatorname{\lim \sup}_n \classlossn \leq \begin{cases}
    0, &  q < \frac{(p-1)}{2} + (1-r)\\ 
    \frac{1}{2}, & q > \frac{(p-1)}{2} + (1-r)
    \end{cases}.
\end{align*}
On the other hand, we can also compute the limiting upper bound on SNR:
\begin{align*}
    \operatorname{\lim\sup}_n \SUCupper = \begin{cases}
    \infty, & 0 < q < \frac{(p-1)}{2} + (1-r) \\
    0, & q >  \frac{(p-1)}{2} + (1-r) .
    \end{cases}
\end{align*}
and so the classification test loss is \textit{lower bounded} by:
\begin{align*}
    \operatorname{\lim\inf}_n \classlossn \geq \begin{cases}
    0,& 0 < q < \frac{(p-1)}{2} + (1-r) \\
    \frac{1}{2},&q > \frac{(p-1)}{2} + (1-r) .
    \end{cases}
\end{align*}
Putting these together, we get
\begin{align*}
    \lim_{n \to \infty} \classlossn = \begin{cases}
    0,& 0 < q < \frac{(p-1)}{2} + (1-r) \\
    \frac{1}{2},&q > \frac{(p-1)}{2} + (1-r) .
    \end{cases}
\end{align*}
This completes the proof.
\qed

\section{Technical lemmas}\label{appendix:technicallemmas}

\subsection{Proof of Lemma~\ref{lem:opnormconcentration}}\label{appendix:lemmaproof:opnormconcentration}

In this subsection, we prove Lemma~\ref{lem:opnormconcentration}, i.e.~concentration on the operator norm of the random matrix $\E := \A - ||\lambdabold||_1 \mathbf{I}_n$.
Recall that $\A$ is the random Gram matrix as defined in Appendix~\ref{appendix:additionalnotation}.
It is easy to verify that $\mathbb{E}[\A] = ||\lambdabold||_1 \mathbf{I}_n$.
We start by recalling the following lemma by Laurent and Massart~\citep{laurent2000adaptive}.
\begin{lemma} \textbf{~\citep{laurent2000adaptive}}\label{lem:laurentmassart}
For any $t > 0$, and any $\uvec \in \mathbb{R}^n$, we have
\begin{align*}
    \Pr\left[\uvec^\top \E \uvec > \sqrt{2 ||\lambdabold||_2 \cdot t} + 2 ||\lambdabold||_{\infty} \cdot t \right] &\leq e^{-t} \\
    \Pr\left[\uvec^\top \E \uvec < - \sqrt{2 ||\lambdabold||_2 \cdot t}\right] &\leq e^{-t} ,
\end{align*}
where the probability is taken over the randomness in the matrix $\E$.
\end{lemma}
We use this lemma together with a discretization and covering argument.
Let $\mathcal{U} := \{\uvec_1, \ldots, \uvec_N\}$ be an $\epsilon$-net for the unit sphere $\mathcal{S}^{n-1}$ in $\mathbb{R}^n$, i.e., we have $\min_{i \in \{1,\ldots,N\}} ||\uvec - \uvec_i||_2 \leq \epsilon \text{ for all } \uvec \in \mathcal{S}^{n-1}$.
It is easy to show (for e.g. according to the covering arguments provided in~\citet[Chapter 4]{wainwright2019high}) that we can pick $\mathcal{U}$ to be an $\epsilon$-net such that $N \leq \left(1 + \frac{2}{\epsilon}\right)^n$.

We set $t := \tau + n \ln \left(1 + \frac{2}{\epsilon}\right)$ for some $\tau > 0$.
Then, by a union bound over the set $\mathcal{U}$, we have
\begin{subequations}
\begin{align}
    \max_{i \in [N]} \uvec_i^\top \E \uvec_i &\leq \sqrt{2 ||\lambdabold||_2 \cdot t} + 2 ||\lambdabold||_{\infty} \cdot t \text{ and } \label{eq:eigmaxepsilonnet}\\
     \min_{i \in [N]} \uvec_i^\top \E \uvec_i &\geq - \sqrt{2 ||\lambdabold||_2 \cdot t} \label{eq:eigminepsilonnet}
\end{align}
\end{subequations}
with probability at least $(1 - 2e^{-\tau})$ over the randomness in the matrix $\E$.
It now remains to remove the discretization in both directions.
Let $\widehat{\uvec} := {\arg \max}_{\uvec \in \mathcal{S}^{n-1}} \uvec^\top \A \uvec = {\arg \max}_{\uvec \in \mathcal{S}^{n-1}} ||\A^{1/2} \uvec||_2$, and let $i_0 := {\arg \min}_{i \in \{1,\ldots,N\}} ||\widehat{\uvec} - \uvec_i||_2$ denote the nearest neighbor of $\widehat{\uvec}$.
Then, we have
\begin{align*}
    ||\A^{1/2} \widehat{\uvec}||_2 &= ||\A^{1/2}\uvec_{i_0} + \A^{1/2}(\widehat{\uvec} - \uvec_{i_0})||_2^2 \\
    &\stackrel{\1}{\leq} ||\A^{1/2}\uvec_{i_0}||_2 + ||\widehat{\uvec} - \uvec_{i_0}||_2 ||\A^{1/2} \widehat{\uvec}||_2 \\
    &\stackrel{\2}{\leq} ||\A^{1/2}\uvec_{i_0}||_2 + \epsilon ||\A^{1/2} \widehat{\uvec}||_2 ,
\end{align*}
where inequality $\1$ is the triangle inequality on the $\ell_2$-norm, and inequality $\2$ follows from the definition of the $\epsilon$-net.
Thus, we get
\begin{align*}
    \max_{\uvec \in \mathcal{S}^{n-1}} \uvec^\top \A \uvec &\leq \frac{1}{(1 - \epsilon)^2} \uvec_{i_0}^\top \A \uvec_{i_0} \\
    &\leq \frac{1}{(1 - \epsilon)^2}\left(||\lambdabold||_1 +\sqrt{2 ||\lambdabold||_2 \cdot t} + 2 ||\lambdabold||_{\infty} \cdot t \right) .
\end{align*}
Noting that $\mu_{\mathsf{max}}(\E) = \mu_{\mathsf{max}}(\A) - ||\lambdabold||_1$ gives us
\begin{align}\label{eq:maxevalupperbound}
    \mu_{\mathsf{max}}(\E) \leq f_1(\lambdabold;\epsilon,\tau) .
\end{align}
On the other side, for any $\uvec \in \mathcal{S}^{n-1}$, let $i^*$ be the index of its nearest neighbor in $\mathcal{U}$.
Then, we have
\begin{align*}
    \uvec^\top \A \uvec &= \uvec_{i^*}^\top \A \uvec_{i^*} + 2\uvec_{i^*}^\top \A (\uvec - \uvec_{i^*}) + (\uvec - \uvec_{i^*})^\top \A (\uvec - \uvec_{i^*}) \\
    &\stackrel{\1}{\geq} \uvec_{i^*}^\top \A \uvec_{i^*} + 2\uvec_{i^*}^\top \A (\uvec - \uvec_{i^*}) \\
    &\stackrel{\2}{\geq} \uvec_{i^*}^\top \A \uvec_{i^*} - 2||\A^{1/2} \uvec_{i^*}||_2 \cdot ||\A^{1/2}(\uvec - \uvec_{i^*})||_2 \\
    &\stackrel{\3}{\geq} \uvec_{i^*}^\top \A \uvec_{i^*} - 2 ||\uvec - \uvec_{i^*}||_2 \cdot ||\A^{1/2} \uvec_{i^*}||_2 \cdot ||\A^{1/2} \widehat{\uvec}||_2 \\
    &\stackrel{\4}{\geq} \uvec_{i^*}^\top \A \uvec_{i^*} - 2 \epsilon \cdot ||\A^{1/2} \uvec_{i^*}||_2 \cdot ||\A^{1/2} \widehat{\uvec}||_2 \\
    &\geq  \uvec_{i^*}^\top \A \uvec_{i^*} - \left(\frac{2\epsilon}{1 - \epsilon}\right) \cdot ||\A^{1/2} \uvec_{i_0}||_2 \cdot ||\A^{1/2} \uvec_{i^*}||_2 \\
    &\geq ||\lambdabold||_1 - \sqrt{2 ||\lambdabold||_2 \cdot t} - \left(\frac{2\epsilon}{1 - \epsilon}\right) \cdot \left(||\lambdabold||_1 + \sqrt{2 ||\lambdabold||_2 \cdot t} + 2||\lambdabold||_{\infty} \cdot t\right)
\end{align*}
where inequalities $\1$ and $\2$ follow from the positive semidefiniteness of $\A$ and the Cauchy-Schwarz inequality respectively, and inequalities $\3$ and $\4$ follow from the definition of the operator norm and the $\epsilon$-net respectively.
The last two inequalities follow since we recall that $||\A^{1/2} \widehat{\uvec}||_2 \leq \frac{1}{(1 - \epsilon)} ||\A^{1/2} \uvec_{i_0}||_2$.
Then, we substitute Equation~\eqref{eq:eigmaxepsilonnet} twice for indices $i^*$ and $i_0$ respectively.
Again, noting that $\mu_{\min}(\E) = \mu_{\min}(\A) - ||\lambdabold||_1$, and substituting $t := \tau + n \ln(1 + \frac{2}{\epsilon})$, gives us
\begin{align}\label{eq:minevallowerbound}
    \mu_{\min}(\E) \geq -f_2(\lambdabold;\epsilon,\tau) .
\end{align}
Finally, using Equations~\eqref{eq:maxevalupperbound} and~\eqref{eq:minevallowerbound}, we have
\begin{align*}
    ||\E||_{\mathsf{op}} = \max\{\mu_{\max}(\E), -\mu_{\min}(\E)\} \leq \max\{f_1(\lambdabold;\epsilon,\tau), f_2(\lambdabold;\epsilon,\tau)\} ,
\end{align*}
completing the proof.
\qed

\subsection{Proof of Lemma~\ref{lem:invwishartconcentration}}\label{appendix:lemmaproof:invwishartconcentration}

In this subsection, we prove Lemma~\ref{lem:invwishartconcentration}, i.e.~concentration on the quantity $\frac{1}{\uvec^\top \Ainv \uvec}$ for the inverse Wishart matrix $\Ainv$.
Because $\A$ is a Wishart matrix, we can use rotational invariance of the distribution of the random variable $\uvec^\top \Ainv \uvec$ for any $\uvec \in \mathcal{S}^{n-1}$.
Thus, it suffices to prove the concentration bound for $\uvec := \evec_n$, i.e.~study the random variable $\Ainv_{n,n} = \evec_n^\top \Ainv \evec_n$.

From elementary properties of the inverse Wishart distribution, we know that the quantity $\frac{1}{\Ainv_{n,n}} \sim \chi^2(d - n + 1)$.
Recall that we denoted $d'(n) := (d - n + 1)$ for shorthand.
Therefore, substituting Lemma~\ref{lem:laurentmassart} (with $\lambdabold := \mathbf{1}$), we get
\begin{align*}
    \Pr\left[\frac{1}{\Ainv_{n,n}} > \sqrt{2d'(n)t} + 2t\right] &\leq e^{-t} \\
    \Pr\left[\frac{1}{\Ainv_{n,n}} < - \sqrt{2d'(n)t}\right] \leq e^{-t} .
\end{align*}
Since $\Ainv_{n,n}$ is identically distributed to $\uvec^\top \Ainv \uvec$ for any $\uvec \in \mathcal{S}^{n-1}$, the above concentration inequalities hold for the random variable $\frac{1}{\uvec^\top \Ainv \uvec}$.
This completes the proof.
\qed

\subsection{Proof of Lemma \ref{lemma:contaminationformula}}
\label{appendix:lemmaproof:contaminationformula}

In this subsection, we prove Lemma~\ref{lemma:contaminationformula}, i.e.~equivalent quadratic form expressions for the contamination factor when binary labels are interpolated.
As argued in Appendix~\ref{appendix:binarylabelsurvival}, for any $j \in \{1,\ldots, d\}$, the coefficient $\alphahat_j$ is given by
\begin{align*}
    \alphahat_j = \e_j \tran \Atrain \Ainv \Ytrain &= \sqrt{\lambda_j} \z_j \tran \Ainv \ytk.
\end{align*}
From the Sherman-Morrison-Woodbury identity, we have
\begin{align*}
    \Ainv &=\Amkinv - \frac{\lambdatk \Amkinv \ztk \ztk \tran \Amkinv}{1 + \lambdatk \ztk \tran \Amkinv \ztk}.
\end{align*}
Using this, we can rewrite $\alphahat_j$ as
\begin{align*}
     \alphahat_j &= \sqrt{\lambda_j} \z_j^\top \left(\Amkinv - \frac{\lambdatk \Amkinv \ztk \ztk \tran \Amkinv}{1 + \lambdatk \ztk \tran \Amkinv \ztk} \right) \ytk \nonumber \\
    &= \sqrt{\lambda_j} \cdot \left(1 - \frac{1}{1 + \lambdatk \ztk \tran \Amkinv \ztk}\right) \cdot \z_j \tran \Amkinv \ytk \nonumber \\
    &= {\sqrt{\lambda_j}} \cdot \z_j \tran \Amkinv \left(\ytk - \survival(t) \ztk\right)
\end{align*}
where the last equality follows from Equation~\eqref{eq:suform}.\\
Using the definition of contamination (Equation~\eqref{eq:contaminationformula}) and the above expressions, we get
\begin{align*}
    \contamination^2(\truek) = \sumjdk \lambda_j \alphahat^2_j
    &= \sumjdk \lambda_j^2 \ytk \tran \Ainv \z_j \z_j\tran \Ainv \ytk \\
    &= \ytk \tran \Ainv \left(\sumjdk \lambda_j^2 \z_j \z_j\tran \right) \Ainv \ytk \\
    &= \ytk \tran \C \ytk.
\end{align*}
Now, we denote $\widetilde{\ytk} := \ytk - \survival(t) \ztk$.
To prove the second form of contamination, we use the following sequence of equalities:
\begin{align*}
    \contamination^2(\truek) = \sumjdk \lambda_j \alphahat^2_j   &= \sumjdk \lambda_j  \left({\sqrt{\lambda_j}} \z_j \tran \Amkinv \widetilde{\ytk}\right)^2\\
    &= \sumjdk \lambda_j^2 \widetilde{\ytk} \tran \Amkinv \z_j \z_j\tran \Amkinv \widetilde{\ytk} \\
    &= \widetilde{\ytk} \tran \Amkinv \left(\sumjdk \lambda_j^2 \z_j \z_j\tran \right) \Amkinv \widetilde{\ytk} \\
    &=  \widetilde{\ytk} \tran \Ctd \widetilde{\ytk}.
\end{align*}
This completes the proof of Lemma~\ref{lemma:contaminationformula}.
\qed

\subsection{Proof of Lemma \ref{lemma:epsilontrace}}
\label{appendix:lemmaproof:epsilontrace}

In this subsection, we prove Lemma~\ref{lemma:epsilontrace}, i.e.~a high-probability upper bound on the quadratic forms $\widetilde{\ytk}^\top \Ctd \widetilde{\ytk}$ and $\ztk^\top \Ctd \ztk$ over only the randomness in $\{\ztk,\ytk\}$.
Recall that we defined the random variables $\{\ztk,\ytk\}$ in Appendix~\ref{appendix:gaussianproofs}.
Note that $\Ctd$ is almost surely positive definite and $\{\ztk,\widetilde{\ytk}\}$ are both pairwise independent of $\Ctd$.
Further, note that
\begin{align*}
    \widetilde{\ytk}^\top \Ctd \widetilde{\ytk} &= (\ytk - \survival(t) \ztk)^\top \Ctd (\ytk - \survival(t) \ztk) \\
    &\leq (\ytk - \survival(t) \ztk)^\top \Ctd (\ytk - \survival(t) \ztk) + (\ytk + \survival(t) \ztk)^\top \Ctd (\ytk + \survival(t) \ztk) \\
    &= 2 \ytk^\top \Ctd \ytk + 2 \survival(t)^2 \ztk^\top \Ctd \ztk .
\end{align*}
From Equation~\eqref{eq:hansonwright2}, we have
\begin{align*}
    \ztk \tran  \Ctd \ztk &\leq \tr(\Ctd) \left(1 + \frac{1}{c}\right) \cdot (\ln n)
\end{align*}
almost surely for every realization of the random matrix $\Ctd$, and with probability at least $\left(1 - \frac{1}{n}\right)$ over the randomness in $\ztk$.
By an identical argument (noting that $y_{t,i}^2 = 1$ almost surely, and that $\EE\left[y_{t,i} y_{t,j}\right] = 0$ for any $i \neq j$), we can show that
\begin{align}\label{eq:zCtdzub}
    \ztk \tran  \Ctd \ztk &\leq \tr(\Ctd) \left(1 + \frac{1}{c}\right) \cdot (\ln n)
\end{align}
Substituting these inequalities in the expression for $\widetilde{\ytk}^\top \Ctd \widetilde{\ytk}$ completes the proof.
\qed

\subsection{Proof of Lemma
\ref{lemma:eigClowerbound}}
\label{appendix:lemmaproofs:eigClowerbound}

In this subsection, we prove Lemma~\ref{lemma:eigClowerbound}, i.e.~a high-probability lower bound on the minimum eigenvalue of the random (almost surely positive semidefinite) matrix $\C$.
Recall that we defined
 \begin{align*}
     \C   &:= \Ainv \left(\sumjdk \lambda_j^2 \z_j \z_j\tran \right) \Ainv,\\
     &= \Ainv \left(\sum_{j=1}^{d-1} \lambdatd_j^2 \z_j \z_j\tran \right) \Ainv .
 \end{align*}
Using the mathematical fact from Appendix~\ref{appendix:mf:mineigproduct}, we have
 \begin{align*}
      \eig_n(\C) \geq (\eig_n(\Ainv))^2 \eig_n\left(\sum_{j=1}^{d-1} \lambdatd_j^2 \z_j \z_j\tran \right) .
 \end{align*}
Now, Equations  \eqref{eq:Amiinv_eigbounds} and  \eqref{eq:Amk_eigbounds_td} from Lemma \ref{lemma:eigboundsAmiinv} can be used to lower bound the terms $(\eig_n(\Ainv))^2$ and $\eig_n\left(\sum_{j=1}^{d-1} \lambdatd_j^2 \z_j \z_j\tran \right)$ respectively.
Substituting these lower bounds into the above bound completes the proof.
\qed

\subsection{Proof of Lemma
\ref{lemma:contaminationformulareal}}
\label{appendix:lemmaproofs:contaminationformulareal}

In this subsection, we prove Lemma~\ref{lemma:contaminationformulareal}, i.e.~equivalent quadratic form expressions for the contamination factor when real output is interpolated.
This proof closely mirrors the proof of Lemma~\ref{lemma:contaminationformula}.

\noindent Let $\alphahat_j$ denote the $j^{\text{th}}$ component of $\alphahatbold_{2, \mathsf{real}}$.
As argued in Appendix~\ref{appendix:reallabelsurvival}, for any $j \in \{1,\ldots, d\}$, the coefficient $\alphahat_j$ is given by
\begin{align}
    \alphahat_j = \e_j \tran \Atrain \Ainv \Ztrain &= \sqrt{\lambda_j} \z_j \tran \Ainv \ztk. \label{eq:alphahat1real}
\end{align}
By the Sherman-Morrison-Woodbury Identity, we have
\begin{align*}
    \Ainv &=\Amkinv - \frac{\lambdatk \Amkinv \ztk \ztk \tran \Amkinv}{1 + \lambdatk \ztk \tran \Amkinv \ztk}.
\end{align*}
Using this, we can rewrite $\alphahat_j$ as
\begin{align}
     \alphahat_j &= \sqrt{\lambda_j} \left(1 - \frac{\lambdatk \ztk \tran \Amkinv \ztk}{1 + \lambdatk \ztk \tran \Amkinv \ztk} \right) \z_j \tran \Amkinv \ztk \nonumber \\ 
    &= {\sqrt{\lambda_j}} (1 - \survivalr(\truek)) \z_j \tran \Amkinv \ztk, \label{eq:alphahat2real}
\end{align}
where the last equality follows from Equation~\eqref{eq:suformreal}.\\
Finally, using the definition of contamination (Equation \eqref{eq:contaminationformula}) together with Equation~\eqref{eq:alphahat1real} gives us
\begin{align*}
    \contaminationr^2(\truek) = \sumjdk \lambda_j \alphahat^2_j
    &= \sumjdk \lambda_j^2 \ztk \tran \Ainv \z_j \z_j\tran \Ainv \ztk \\
    &= \ztk \tran \Ainv \left(\sumjdk \lambda_j^2 \z_j \z_j\tran \right) \Ainv \ztk \\
    &= \ztk \tran \C \ztk.
\end{align*}
Similarly, applying Equation~\eqref{eq:alphahat2real} gives us
\begin{align*}
    \contaminationr^2(\truek) = \sumjdk \lambda_j \alphahat^2_j   &= \sumjdk \lambda_j  \left({\sqrt{\lambda_j}}(1 - \survivalr(\truek)) \z_j \tran \Amkinv \ztk\right)^2\\
    &= (1 - \survivalr(\truek))^2 \sumjdk \lambda_j^2 \ztk \tran \Amkinv \z_j \z_j\tran \Amkinv \ztk \\
    &=  (1 - \survivalr(\truek))^2 \ztk \tran \Amkinv \left(\sumjdk \lambda_j^2 \z_j \z_j\tran \right) \Amkinv \ztk \\
    &=  (1 - \survivalr(\truek))^2 \ztk \tran \Ctd \ztk.
\end{align*}
This completes the proof.
\qed

\section{Mathematical Facts}
\label{appendix:mathfacts}
\subsection{Upper bound on maximum eigenvalue of product of positive definite matrices}
\label{appendix:mf:maxeigproduct}

Let $\A, \B \in \real{n \times n}$ be symmetric positive definite matrices and let $\C = \A\B$. It is a well known fact that for positive definite matrix $\Mbold$,  $\eig_1(\Mbold) = \norm{\M}_2$, i.e the largest eigenvalue is the operator norm.
Using this,
\begin{align*}
    \eig_1(\C) = \norm{\C}_2 = \norm{\A \B}_2 \leq \norm{\A}_2 \norm{\B}_2 = \eig_1(\A) \eig_1(\B),
\end{align*}
where the inequality follows from the sub-multiplicativity of operator norm.

\subsection{Lower bound on minimum eigenvalue of product of positive definite matrices}

Let $\A, \B \in \real{n \times n}$ be symmetric positive definite matrices and let $\C = \A\B$. Note that since inverses exist for positive definite matrices we can write,
\begin{align*}
    \eig_n(\C) = \frac{1}{\eig_1(\C^{-1})} \geq \frac{1}{\eig_1(\A^{-1}) \eig_1(\B^{-1})} = \eig_n(\A) \eig_n(\B),
\end{align*}
where the inequality follows by applying the upper bound for eigenvalue of product of two positive definite matrices from Appendix \ref{appendix:mf:maxeigproduct}.

\label{appendix:mf:mineigproduct}

\section{Normalized margin calculations}
\label{appendix:spectral_margin}

In this section, we verify that the statement of Equation~\eqref{eq:spectral_gen_bound} exactly matches with the statement in~\citet[Theorem 21]{bartlett2002rademacher}, using the notation from that paper.
Observe that the first and third terms in the generalization bound exactly match.
We only need to verify the second term.
Note that the linear kernel is precisely
\begin{align*}
    k(X,X') := \avec(X)^\top \avec(X') .
\end{align*}
Therefore, we get
\begin{align*}
    \sqrt{\sum_{i=1}^n k(X_i,X_i)} &= \sqrt{\sum_{i=1}^n ||\avec(X_i)||_2^2} \\
    &= ||\Atrain||_{\mathsf{F}} .
\end{align*}
Similarly, using the kernel trick (see the discussion just below Theorem $21$ in the paper), we can verify that the term $B$ is an upper bound on the quantity $||\alphahatbold||_2$.
Substituting these equivalences into the original statement completes the verification.

\vskip 0.2in
\bibliography{reference}

\end{document}